\documentclass[nonblindrev, hidelinks]{informs3}

\usepackage{enumerate}
\usepackage{graphicx,fancyhdr,amsmath,amssymb,url}
\usepackage{algorithm}
\usepackage{todonotes}
\usepackage{algpseudocode}
\usepackage{booktabs}
\usepackage{float}
\restylefloat{table}
\usepackage[normalem]{ulem}
\usepackage{setspace,rotating,longtable,color}
\usepackage{xcolor}
\usepackage{verbatim,epsfig,lscape}
\usepackage{xspace}
\usepackage{bm}
\usepackage{comment}
\usepackage{array}
\usepackage{amsfonts}
\usepackage{mathtools}
\usepackage{xifthen}
\usepackage{natbib}
\usepackage{hyperref}
\usepackage{multirow}
\usepackage{subfigure}
\usepackage{bbm}
\usepackage{stackengine}

\usepackage{epstopdf}

\OneAndAHalfSpacedXI

\usepackage{url}
\makeatletter
\g@addto@macro{\UrlBreaks}{\UrlOrds}
\makeatother

\bibpunct[, ]{(}{)}{,}{a}{}{,}%
\def\bibfont{\small}%

\TheoremsNumberedThrough

\begin{document}

	\TITLE{	Online Learning and Decision-Making under Generalized Linear Model with High-Dimensional Data}
	\ARTICLEAUTHORS{
		\AUTHOR{Xue Wang$^\ast$ \quad \quad Mike Mingcheng Wei$^\star$ \quad \quad Tao Yao$^\dagger$}
		\medskip
		\AFF{$^\ast$Penn State University, Industrial and Manufacturing Engineering, xzw118@psu.edu \\
			$^\star$University at Buffalo, School of Management, mcwei@buffalo.edu\\
			$^\dagger$Penn State University, Industrial and Manufacturing Engineering, tyy1@engr.psu.edu}

	}
	\ABSTRACT{

		 We propose a minimax concave penalized multi-armed bandit algorithm under generalized linear model (G-MCP-Bandit) for a decision-maker facing high-dimensional data in an online learning and decision-making process. We demonstrate that the G-MCP-Bandit algorithm asymptotically achieves the optimal cumulative regret in the sample size dimension $T$, $O(\log T)$, and further attains a tight bound in the covariate dimension $d$, $O(\log d)$. In addition, we develop a linear approximation method, the 2-step weighted Lasso procedure, to identify the MCP estimator for the G-MCP-Bandit algorithm under non-iid samples. Under this procedure, the MCP estimator  matches the oracle estimator with high probability and converges to the true parameters with the optimal convergence rate. Finally, through experiments based on synthetic data and two real datasets (warfarin dosing dataset and Tencent search advertising dataset), we show that the G-MCP-Bandit algorithm outperforms other benchmark algorithms, especially when there is a high level of data sparsity or the decision set is large.
	}

	\maketitle

\section{Introduction}

Individual-level data have become increasingly accessible in the Internet era, and decision-makers have accelerated  data accumulation with extraordinary speed in a variety of industries, including health care, retail, and advertising. The growing availability of user-specific data, such as demographics, geographics, medical records, and searching/browsing history, provides decision-makers with unprecedented opportunities to tailor decisions to individual users. For example, doctors can personalize treatments for patients based on their medical history, clinical tests, and biomarkers; search engines can offer personalized advertisements for users based on their queries, demographics, and geographics. These user-specific data are often collected sequentially over time, during which decision-makers adaptively learn to predict the expected rewards based on users' responses to each available decision as a function of the user-specific data (i.e., the user's covariates) and optimally adjust decisions to maximize their rewards -- an \emph{online} learning and decision-making process.

This online learning and decision-making process requires a thoughtful balance between exploration and exploitation. Consider a decision-maker who selects decisions for incoming users and obtains rewards based on users' responses to these decisions. To maximize his expected rewards, the decision-maker first needs an accurate predictive model for users' responses, which is typically uncertain at the beginning but can be partially learned through collecting samples of users' responses. On the one hand, the decision-maker could select a decision that yields the ``highest'', based on his best knowledge so far, expected reward (i.e., exploitation). Yet, this decision can be suboptimal, as the selection is based on the rough prediction of users' responses due to limited samples. Even worse, the decision-maker could incorrectly estimate the expected reward of the true optimal decision to be low and never have a chance to correct such a mistake (as the decision-maker will not select the true optimal decision due to the current low reward prediction, he will not generate additional samples to be able to learn and correct his incorrect estimation). On the other hand, the decision-maker can improve his predictive ability and learn users' responses by collecting more response samples, which often are obtained through random clinical trials and/or user experiments and are typically costly (i.e., exploration). The exploration and exploitation dilemma has been extensively studied in the multi-armed bandit model (\citealt{robbins1952some}), but the growing dimensionality and availability of data have added another layer of complexity to the bandit model.

In practice, individual-level data are typically presented in a high-dimensional fashion, which poses significant computational and statistical challenges in the online learning and decision-making process. Traditional statistical methods, such as Ordinary Least Squares (OLS), require a large number of samples (e.g., the sample size must be larger than the covariate dimension) to be deemed computationally feasible. Under high-dimensional settings, learning the accurate predictive models requires a substantial amount of samples, which are obtained, if possible, through costly trials or experiments. Take the search advertising industry for example. Search advertising occurs when an Internet user searches certain keyword(s) (i.e., a query) in an online search engine and then the search engine displays both search results, in response to the user's query, and some sponsored ads, in response to the query and user-specific information. In order to select the ad that maximizes its revenue, the search engine must have accurate estimations on users' clicking probabilities in response to the displayed ads -- Click-Through Rate (CTR).

However, the search engine's ability to accurately predict CTR is often crippled by the high-dimensional search advertising data coupled with limited samples. Counting more than three quarters of a million distinct words and their combinations (\citealt{DICT2018}), there are nearly infinite possible queries the user can submit to the search engine. For example, from 2003 to 2012, Google answered 450 billion unique queries, and it has estimated that $16\%$ to $20\%$ of queries submitted every day have never been used before (\citealt{Ben2012}). Hence, to accurately estimate a single ad's CTR to these queries, the search engine requires billions, if not trillions, of samples. The craving for samples will be further intensified if the search engine practices personalized advertising by taking users' individual information (such as demographics and geographics) into consideration. However, the available samples for the search engine to learn and predict CTR are greatly limited. It was suggested that an ideal length of time to run a new marketing campaign promoting a sales event or merchandise is around 45 days (\citealt{Shaheen2018}), during which time an average ad is expected to reach approximately one third of a million users (\citealt{WordStream2017,WordStream2018}). Among these users, a very small portion can be selected to perform costly experiments to learn CTR, but that number is much smaller comparing to the size of queries and individual data.

In this paper, we propose a new algorithm, the G-MCP-Bandit algorithm, for online learning and decision-making processes in high-dimensional settings. Our algorithm follows the ideas of the bandit model and develops a $\epsilon$-decay random sampling method to balance the exploration-and-exploitation trade-off.
We allow the decision-maker's reward function to follow the generalized linear model (\citealt{mcCullagh1989generalized}), which is a large class of models including the linear model, the logistic model, the Poisson regression model, etc., and we adopt the Minimax Concave Penalized (MCP) method (\citealt{zhang2010nearly}) to improve the parameter estimations and predict the expected rewards in high-dimensional settings.


In the high-dimensional statistics literature, MCP is developed to explore and recover the latent sparse data structure for high-dimensional data. Compared to traditional statistical methods (e.g., OLS), MCP uses significantly fewer data samples and delivers better performance in high-dimensional settings (\citealt{zhang2010nearly}). Although it is statistically favorable to adopt MCP, solving the MCP estimator (an NP-complete problem) could be computationally challenging. We propose a linear approximation method, the 2-step weighted Lasso procedure (2sWL), under the bandit setting as an efficient approach to tackle this challenge. We show that the MCP estimator solved by the 2sWL procedure matches the oracle estimator with high probability and converges to the true parameter with the optimal convergence rate (Proposition \ref{MW_proposition:1}). Since the bandit model mixes the exploitation and exploration phases, samples generated under the exploitation phase may be non-iid. Therefore, we adopt a matrix perturbation technique to derive new oracle inequalities for MCP under non-iid samples. To the best of our knowledge, this work is the first one that applies MCP to handle non-iid samples.

We theoretically demonstrate that the G-MCP-Bandit algorithm can significantly improve the cumulative regret bound in high-dimensional settings comparing to existing bandit algorithms. In particular, we benchmark the G-MCP-Bandit algorithm to an oracle policy, in which all parameter vectors are common knowledge, and adopt the expected cumulative regret (i.e., the difference in rewards achieved by the oracle policy and the G-MCP-Bandit algorithm) as the performance measure. We show that the cumulative regret of the G-MCP-Bandit algorithm over $T$ users (i.e., a sample size of $T$) is at most $O(\log T)$, which is the optimal/lowest theoretical bound for all possible algorithms (\citealt{goldenshluger2013linear}). 
Further, we show that the G-MCP-Bandit algorithm also attains a tight bound in the covariate dimension $d$, $O(\log d)$.
We believe that our work is the first one in high-dimensional settings that attains the logarithmic dependence on both the sample size dimension and the covariate dimension, which are of particular importance in high-dimensional data with limited samples and suggest that the G-MCP-Bandit algorithm can bring substantial regret reduction comparing to existing bandit algorithms.

Through two synthetic-data-based experiments, we benchmark the G-MCP-Bandit algorithm's performance to other state-of-the-art bandit algorithms designed both in low-dimensional settings, OLS-Bandit by \citealt{goldenshluger2013linear} and OFUL by \citealt{abbasi2011improved}, and in high-dimensional settings, Lasso-Bandit by \citealt{bastani2015online}. We find that the G-MCP-Bandit algorithm performs favorably in both experiments.  In particular, when the sample size is not extremely small\footnote{When the sample size is extremely small,  the decision-maker has little information to learn. Therefore, all algorithms perform equally poorly.}, the G-MCP-Bandit algorithm appears to be able to accurately learn the parameter estimations with limited samples and therefore have the lowest cumulative regret. Furthermore, we observe that the benefits of the G-MCP-Bandit algorithm over other benchmark algorithms seems to increase with the data's sparsity level and the size of the decision-maker's decision set.

Finally, we evaluate the G-MCP-Bandit algorithm's performance through two real-data-based experiments, warfarin dosing data and Tencent search advertising data, where the technical assumptions specified for the theoretical analysis of the G-MCP-Bandit algorithm's expected cumulative regret may not hold. We observe that the G-MCP-Bandit algorithm continues to perform favorably in both experiments. In particular, in the warfarin dosing experiment (formulated as a 3-armed bandit problem with $93$ covariates), the G-MCP-Bandit algorithm needs the fewest patient samples (i.e., merely $50$ patients) to provide better dosing decisions than actual physicians. Similarly, in the Tencent search advertising experiment (formulated as a 3-armed bandit problem with hundreds of thousands of covariates), the G-MCP-Bandit algorithm, after observing 140 users, can consistently generate better average revenue than other benchmark algorithms under the linear model. Further, we observe that the choice of the underlying reward model can significantly influence the G-MCP-Bandit algorithm's performance. In particular, under the logistic model, which is a special case of the generalized linear model, the G-MCP-Bandit algorithm merely needs $20$ users to outperform other benchmark algorithms. This observation suggests that understanding the context of the underlying managerial problem and identifying the appropriate model for the G-MCP-Bandit algorithm can be critical and bring the decision-maker substantial revenue improvement.

\section{Literature Review} \label{sec:literature}

This research is closely related to the exploration-exploitation trade off in the multi-armed bandit literature.  \cite{rigollet2010nonparametric,slivkins2014contextual} follow the non-parametric approach and consider that the arm reward can be any smooth non-parametric function. Under this approach, the expected cumulative regret has an exponential dependence on the covariate dimension $d$. Such dependence can be improved by the parametric approach. \cite{auer2002using} proposes the UCB algorithm for a linear bandit model, where the arm reward can be approximated by a linear combinations of covariates. Since \cite{auer2002using}, other UCB-type algorithms(e.g., \citealt{dani2008stochastic,rusmevichientong2010linearly,abbasi2012online,deshpande2012linear}) and Bayesian-type algorithms (e.g., \citealt{agrawal2013thompson,russo2014learning}) have been proposed and shown to improve on the expected cumulative regret. Yet, allowing the adversary and without regulating the sample generating process, the statistical performance of the parameter vector estimation in the learning process may suffer. As a result, the expected cumulative regret bound typically has a sublinear dependence on the sample size dimension $T$ (e.g., $O(\sqrt{T})$) and a polynomial dependence on the covariate dimension $d$. However, in high-dimensional settings, where the covariate dimension and the sample size dimension can be exceedingly large, these algorithms can perform poorly.

By introducing a forced sampling approach to the linear bandit model, \cite{goldenshluger2013linear} ensure that enough samples generated in their algorithm possess desired iid property and show that their proposed OLS-Bandit algorithm can achieve $O(\log T)$ dependence on the sample size dimension $T$ in low-dimensional settings. Following a similar approach, \cite{bastani2015online} propose the Lasso-Bandit algorithm, which attains a poly-logarithmic square dependence on the sample size dimension $O(\log^2 T)$ and the covariate dimension $O(\log^2 d)$ in high-dimensional settings. In this paper, we allow the reward function to follow the generalized linear model, which contains a wide family of models that includes the linear bandit model. We propose a $\epsilon$-decay random sampling method and show that our proposed G-MCP-Bandit algorithm continues to achieve the optimal cumulative regret bound on the sample size dimension $O(\log T)$ and attain a tight bound in the covariate dimension $O(\log d)$ in high-dimensional settings. We believe that our work is the first one that attains the logarithmic dependence on both the sample size dimension and the covariate dimension in high-dimensional settings.


Our research is also connected to the statistical learning literature. In high-dimensional statistics, Lasso type methods (\citealt{tibshirani1996regression}) have become the golden standard for high-dimensional learning (\citealt{meinshausen2006high,meinshausen2009lasso,zhang2008sparsity,van2008high}). Yet, Lasso-type regularizations may lead to estimation bias, and strong conditions are needed for analyzing its theoretical performance guarantee (\citealt{fan2014challenges}). Recently, \cite{zhang2010nearly} proposes MCP, a non-convex penalty method, which entails better statistical properties, such as the unbiasedness and a strong oracle property for high-dimensional sparse estimation, and requires weaker conditions than Lasso (\citealt{zou2006adaptive,fan2014strong,meinshausen2006high}). Although it is statistically favorable to adopt MCP, solving the MCP estimator (an NP-complete problem) could be computationally challenging (\citealt{liufolded,liu2016global}). Various approximation methods have been developed in the literature. For example, \cite{fan2001variable} use the local quadratic approximation, \cite{fan2014strong,fan2018lamm,zou2006adaptive,zhao2014pathwise} adopt the local linear approximation, \cite{zhang2010nearly} choose the path following algorithm, and \cite{liufolded} propose the second-order approximation. Our proposed solution procedure (the 2sWL procedure) is analogous to the local linear approximation and guarantees that the solution has desirable statistical properties for theoretical analysis and can be efficiently solved. In the literature, the theoretical analysis of MCP's statistical properties relies on the assumption that all samples are iid, which is hardly the case under bandit models. This paper also contribute to the statistical learning literature by deriving new oracle inequalities for MCP under non-iid samples.

\section{Model Settings}\label{sec:Model}

Consider a sequential arrival process $t\in \{1,2,...,T\}$. At each time step $t$, a single user (e.g., consumer or patient), described by a high-dimensional feature covariate vector $\bm{x}_{t}\in\mathbb{R}^{1\times d} $, arrives. The covariate vector combines all available (but not necessarily valuable for the decision-maker to base his decision on) user-specific data, such as demographics, geographics, browsing/shopping history, and medical records. Upon arrival, users' covariate vectors $\{\bm{x}_t\}_{t\ge 0}$ become observable to the decision-maker and are iid distributed according to an unknown distribution $\mathcal{P}_{x}$.

Based on the user's covariate vector $\bm{x}$, the decision-maker will select a decision from a decision set $\mathcal{K}=\{1,2,...,K\}$ to maximize his expected reward. The user will respond to the chosen decision  $k\in\mathcal{K}$, and such response will generate a reward for the decision-maker. Take the search advertising for example. The search engine can recommend one of $K$ different ads to the user; the user can respond to the recommended ad by clicking, which generates revenue for the search engine. We denote this reward under the chosen decision $k$ as $R_k$, which follows a distribution $\mathbbm{P}(R_k|\bm{x}^T\bm{\beta}_k^{true})$, where $\bm{x}$ is the user's covariate vector and $\bm{\beta}_k^{true}$ is the unknown parameter vector corresponding to decision $k$.

We present the reward function in terms of the generalized linear model (\citealt{mcCullagh1989generalized}), which  is a large class of models including the linear model,  the logistic model, the Poisson regression model, etc. For example, if we assume that $R_k$ is a $\sigma$-gaussian random variable with mean $\bm{x}^T\bm{\beta}_k^{true}$, then we can define the density function of the  distribution $\mathbbm{P}(R_k|\bm{x}^T\bm{\beta}_k^{true})$  as $g(R_k=r|\bm{x}^T\bm{\beta}^{true}_k)=(1/{\sqrt{2\pi\sigma^2}})\exp(-\frac{(r-\bm{x}^T\bm{\beta}^{true}_k)^2}{2\sigma^2})$, which is the standard setting for the classic linear multi-armed bandit model where the reward takes a linear form: $R_{k}(\bm{x})=\bm{x}^T\bm{\beta}_k^{true}+\epsilon$  (\citealt{auer2002using,agrawal2013thompson}). The cumulative regret performance of the linear bandit algorithms has been extensively studied by \citet{dani2008stochastic} and \citet{goldenshluger2013linear}, among others, under low-dimensional settings and by \citet{bastani2015online} under high-dimensional settings. The generalized linear model adopted in this paper facilitates us to go beyond the classic linear bandit model, as the reward may take a nonlinear form in practice. For instance, the search engine collects revenue only when a user has clicked the recommended ad; otherwise, the search engine earns nothing -- a logistic model by nature. By specifying $R_k$ as a binary random variable (e.g., $R_k\in\{0,1\}$), we can define the mass function of the  distribution $\mathbbm{P}(R_k|\bm{x}^T\bm{\beta}_k^{true})$  as $g(R_k=1|\bm{x}^T\bm{\beta}_k^{true})=1/(1+\exp(-\bm{x}^T\bm{\beta}_k^{true}))$ and $g(R_k=0|\bm{x}^T\bm{\beta}_k^{true})=\exp(-\bm{x}^T\bm{\beta}_k^{true})/(1+\exp(-\bm{x}^T\bm{\beta}_k^{true}))$, which is a logistic bandit model with the binary reward (\citealt{elmachtoub2017practical, scott2015multi, scott2010modern}).

The parameter vector $\bm{\beta}^{true}_k$ is high-dimensional with latent sparse structure, and we denote $\mathcal{S}_k = \{j: \beta^{true}_{k,j}\ne 0\}$ as the index set for significant covariates, which have non-zero coefficient parameters and therefore are important for the  decision-maker to predict the user's response. This index set is also unknown to the decision-maker. We define the number of significant covariates as $|\mathcal{S}_k|$ which is typically much smaller than the dimension of the covariate vector (i.e., $|\mathcal{S}_k|\ll d$).

The decision-maker's objective is to maximize his expected cumulative reward. Denote the decision-maker's current policy as $\pi=\{\pi_t\}_{t\ge0}$, where $\pi_t\in\mathcal{K}$ is the decision prescribed by policy $\pi$ at time $t$. To benchmark the performance of policy $\pi$, we first introduce an \emph{oracle policy} $\pi^*=\{\pi^*_t\}_{t\ge0}$ under which the decision-maker knows the true parameter vector values  $\bm{\beta}^{true}_k$ for all $k\in\mathcal{K}$ and chooses the best decision to maximize his expected reward:
\begin{align}
\pi^*_t=\arg\max_{k\in\mathcal{K}} \left\{\mathbbm{E}[R_k|\bm{x}_t,\bm{\beta}_k^{true}]\right\}=\arg\max_{k\in\mathcal{K}}\left\{\int_{-\infty}^{+\infty}r_kdG(r_k|\bm{x}_t^T\bm{\beta}_k^{true})\right\},\notag
\end{align}
where $G(r_k|\bm{x}_t^T\bm{\beta}_k^{true})$ is the cumulative distribution function for $R_k$. Note that in practice, the parameter vector $\bm{\beta}^{true}_k$ is unknown to the decision-maker, and therefore the construction and definition of the oracle policy directly imply that the decision-maker's reward under policy $\pi$  is upper-bounded by that of the oracle policy. We therefore define the decision-maker's expected cumulative regret up to time $T$ under the policy $\pi$ as follows:
\begin{align}
R^{C}(T) = \sum_{t=1 }^T \mathbbm{E}[R_{t}^{\pi_{t}^*}-R_{t}^{\pi_t}],\notag
\end{align}
which is the expected reward difference between the optimal policy $\pi^*$ and the decision-maker's alternative policy $\pi$.  To maximize his expected cumulative reward, the decision-maker is equivalent to explore for the policy $\pi$ that minimizes the cumulative regret up to time $T$.

Before presenting the proposed G-MCP-Bandit algorithm, we will first state five technical assumptions necessary for the theoretical analysis of the decision-maker's expected cumulative regret. The first three assumptions are adopted directly from the multi-armed bandit literature, and the last two assumptions from the high-dimensional statistics literature.

\noindent\textbf{A. 1} (Parameter set) There exist positive constants $x_{\max}$, $s$, $R_{\max}$, $\beta_{\min}$  and $b$ such that for any $t$ and $k \in \mathcal{K}$, we have $\|\bm{x}_t\|_{\infty}\le x_{\max}$, $|\mathcal{S}_k|\le s$, $|R_k|\le R_{\max}$, $\beta_{\min}\le \min_{j\in\mathcal{S}_k,k\in\mathcal{K}}|\beta^{true}_{k,j}|$, $\|\bm{\beta}_{k}^{true}\|_1\le b$ and all feasible $\beta$ satisfies $\|\bm{\beta}\|_1\le b$.

The first assumption is a standard assumption in the bandit literature (\citealt{rusmevichientong2010linearly}) and ensures that both the covariate vector $\bm{x}$ and the coefficient vector $\bm{\beta}_k$ are upper bounded so that the maximum regret at every time step will also be upper bounded to avoid trivial decisions. Most real world applications, including two real data experiments in \S \ref{sec:empirical_warfarin} and \S \ref{sec:empirical_tencent}, satisfy this assumption.

\noindent\textbf{A. 2}  (Margin condition) There exists a $C>0$ such that $\mathbb{P}(0<|\mathbbm{E}[R_i|\bm{x},\bm{\beta}^{true}_i]-\mathbbm{E}[R_j|\bm{x},\bm{\beta}^{true}_j]|\le\gamma)\le CR_{\max}\gamma$  for $i \neq j$ and $i,j\in\mathcal{K}$.

The second assumption is first introduced in the classification literature by \citet{tsybakov2004optimal}. \citet{goldenshluger2013linear} and \citet{bastani2015online} adopt this assumption to the linear bandit model, under which the Margin Condition ensures only a fraction of covariates can be drawn near the boundary hyperplane $\bm{x}^T(\bm{\beta}_i^{true}-\bm{\beta}_j^{true})=0$ in which rewards for both arms are nearly equal. Clearly, if a large proportion of covariates are drawn from the vicinity of the boundary hyperplane, then for any bandit algorithm, a small estimation error in the decision parameter vectors may lead the decision-maker to choose the suboptimal decision and perform poorly (\citealt{bastani2015online}).  Therefore, this margin condition ensures that given a user's covariate vector, decisions can be properly separated from each other and ordered based on their rewards.

\noindent\textbf{A. 3} (Arm optimality)
There exists a partition $\mathcal{K}_o$ and $\mathcal{K}_s$ for $\mathcal{K}$. For $k_1\in
\mathcal{K}_s$, we will have $\mathbbm{E}[R_{k_1}|\bm{x},\bm{\beta}^{true}_{k_1}]+h< \max_{k\ne k_1}\mathbbm{E}[R_{k}|\bm{x},\bm{\beta}_{k}^{true}]$ for a positive constant $h$ for every $\bm{x}$. For $k_2\in\mathcal{K}_o$, these exists another positive constant $p^*$ such that $\min \mathbb{P}(\bm{x}\in U_{k_2})\ge p^*$, where $U_{k_2} \dot= \left\{\bm{x}|\mathbbm{E}[R_{k_2}|\bm{x},\bm{\beta}^{true}_{k_2}]>\max_{k\ne {k_2}}\mathbbm{E}[R_k|\bm{x},\bm{\beta}_k^{true}]+h, k\in\mathcal{K}\right\}$.

The arm optimality condition (\citealt{goldenshluger2013linear, bastani2015online}) ensures that as the sample size increases, the parameter vectors for optimal decisions can eventually be learned. In particular, this condition separates decisions to an optimal decision subset $\mathcal{K}_o$ and a suboptimal decision subset $\mathcal{K}_s$. Decision $i$ in $\mathcal{K}_o$ is strictly optimal for some users' covariate vectors (denoted by set $U_i$); otherwise, decision $j$ in $\mathcal{K}_s$ must be strictly suboptimal for all users' covariate vectors. Therefore, even if there is a small estimation error for decision $i$ in $\mathcal{K}_o$, the decision-maker will be more likely to choose decision $i$ for a user with a covariate vector draw from the set $U_i$. Accordingly, as sample size $T$ increases, decision-makers can improve their estimations for optimal arms' parameter vectors.

These first three assumptions are directly adopted from the multi-armed bandit literature and have been shown to be satisfied for all discrete distributions with finite support and a very large class of continuous distributions (see \citealt{bastani2015online} for detailed examples and discussions).

\noindent\textbf{A. 4} (Restricted eigenvalue condition) There exists $\kappa>0$ such that for all feasible $\bm{\xi}$ satisfying $\|\xi\|_1\le b$ and $\bm{u}$ such that $\|\bm{u}_{\mathcal{S}_k}^c\|_1\le 3\|\bm{u}_{\mathcal{S}_k}\|_1$, we have $\frac{\kappa}{s}\|\bm{u}_{\mathcal{S}_k}\|_1^2 \le \bm{u}^T\mathbbm{E}[\nabla^2\mathcal{L}(\bm{\xi})]\bm{u}$,
where $\mathcal{L}$ is the log likelihood function, $\mathcal{L}(\bm{\beta}) = \frac{1}{n}\sum_{j=1}^n-\log g(r_j|\bm{x}_j^T\bm{\beta})$, and $\{\bm{x}_j, j =1,2,...,n\}$ are  iid random samples with $\bm{x}_{j}\in U_k, k\in\mathcal{K}$.

The restricted eigenvalue condition assumption is a standard assumption in high-dimensional statistics that is necessary for the identifiability and consistency of high-dimensional estimators (\citealt{fan2018lamm,fan2014strong}). This assumption considers the local geometry of the log likelihood function $\mathcal{L}$ with iid samples in $U_k$. To intuit, note that under low-dimensional settings, the literature (\citealt{montgomery2012introduction}) requires that $\mathcal{L}$ is strongly convex around the true parameter vector ${\bm{\beta}}^{true}$ (e.g., the Hessian matrix in OLS estimator is positive-definite and invertible) in order to achieve identifiability of the parameter vector. However, the strong convexity assumption is typically violated in high-dimensional settings, as the sample size can be much smaller than the covariate dimension. Therefore, a weaker condition is adopted: The $\mathcal{L}$ exhibits local strongly convex behavior only in some restricted subspace of $u$. In high-dimensional linear models, the restricted eigenvalue condition assumption is analogous to the compatibility condition (\citealt{bastani2015online,buhlmann2011statistics}), restrict strongly convexity condition (\citealt{negahban2009unified,loh2013regularized}), and sparse eigenvalue condition ( \citealt{zhang2012general,fan2018lamm}).

\noindent\textbf{A. 5} (Density function)  The negative logarithm of the reward density function $f(\cdot|\cdot)\dot{=}-\log g(\cdot|\cdot)$ is (i) smooth and convex, and (ii) there exists positive constants $\sigma$, $\sigma_2$ and $\sigma_3$ such that $|f^{'}(\cdot|\cdot)|\le \sigma$, $f^{''}(\cdot|\cdot)<\sigma_2$ and $|f^{'''}(\cdot|\cdot)|\le \sigma_3$.

The density function assumption enables us to use the estimated expected reward to statistically infer the true expected reward. Specifically, under this assumption, when the parameter estimator $\bm{\beta}$ is close enough to the underlying true parameter vector ${\bm{\beta}}^{true}$, the negative logarithm of the reward density function under the estimator $\bm{\beta}$, $g(\bm{x}^{T}\bm{\beta})$, will converge to that under the true parameter vector ${\bm{\beta}}^{true}$, $g(\bm{x}^{T}{\bm{\beta}}^{true})$. The density function assumption is a fairly weak technical assumption. Many common distributions, such as sub-Gaussian distribution and Bernoulli distribution, satisfy this density function assumption.

\section{G-MCP-Bandit Algorithm} \label{sec:algorithm}

One of the major challenges for online learning and decision-making problems is discovering the underlying sparse data structure and estimating the parameter vector for high-dimensional data with limited samples. Lasso (\citealt{tibshirani1996regression}) has been proposed as an efficient statistical learning method and adopted in the multi-armed bandit literature (\citealt{bastani2015online}) to hurdle this challenge.
However, the Lasso estimator can be biased and performs inadequately, especially when the magnitude of true parameters is not too small (\citealt{fan2001variable}). Such estimation bias typically leads to a suboptimal cumulative regret bound in sample size dimension for Lasso-based algorithms. One way to address this performance issue is to construct new penalty functions that could render unbiased estimators and improve the sparse structure discovery under high-dimensional data with limited samples. In this research, we will adopt the novel MCP method.

\subsection{Parameter Vector Estimation}  \label{sec:Model_estimation}
For notation convenience, we will omit parameters' subscripts corresponding to the choice of arms, as long as doing so will not cause any misinterpretation. Consider an oracle estimator for an arbitrary arm, $\bm{\beta}^{oracle}$, which is the parameter estimator when the decision-maker has perfect knowledge of the index set for significant covariates $\mathcal{S}$. In other words, the oracle estimator can be determined by setting $\beta_{j}=0 \text{ for } j\in \mathcal{S}^{c}$ and solving
\vspace{-7pt}
\begin{equation}\label{oracle estimator}
\bm{\beta}^{oracle}(\bm{X},\bm{r})\doteq \arg \min_{{\substack{\bm{\beta}_{\mathcal{S}^c}=0\\ \bm{\beta}_{\mathcal{S}}}}\newline
}\ \left\{\frac{1}{|\mathcal{A}|}\sum_{j\in\mathcal{A}}f(r_j|\bm{x}_j^T\bm{\beta})\right\},
\end{equation}
where $\mathcal{A}$ is  the available historical data samples and $f(\cdot|\cdot)$ is the negative logarithm of the reward density function defined early. When solving for the oracle estimator, the decision-maker can  directly ignore insignificant covariates by forcing their corresponding coefficients to be zero and essentially reduce the high-dimensional problem to a low-dimensional counterpart. The statistical performance of the oracle estimator is provided in the following lemma.

\begin{lemma}\label{MW_lemma_1}
	Let $n$ be the sample size. Under assumption A.1, A.4, and A.5, the following inequality for the oracle estimator holds
	\begin{align}
	\mathbbm{P}\left(\|\bm{\beta}^{oracle}-\bm{\beta}^{true}\|_2\le \sqrt{\frac{8s^2\sigma_2\sigma^2x_{\max}^2}{\mu_0^2n}}\right)\ge 1-\delta_1(n),\label{MW_lemma_1:1}
	\end{align}
	where $\delta_1(n)\doteq2\exp(-\frac{C_hn\mu_0}{2sx_{\max}^2})+2s\exp(-\frac{\mu_0 n}{4s\sigma_2x_{\max}^2})$, and $C_h$ and $\mu_0$ are positive constants.
\end{lemma}

Since there are only $|\mathcal{S}|$ significant covariates, which is upper-bounded by $s$, are free to change in Equation \eqref{oracle estimator}, the optimal statistical performance of the likelihood estimation is commonly recognized as $O(\sqrt{s/n})$ in the literature (\citealt{fan2018lamm,zhao2018pathwise}), which ignores the dependence of the largest eigenvalue in the objective function's Hessian matrix. In Equation \eqref{MW_lemma_1:1}, we explicitly include its influence and can directly verify that the largest eigenvalue in the objective function's Hessian matrix is universally upper bounded by $\sigma_2sx_{\max}^2$
and therefore Equation \eqref{MW_lemma_1:1} reduces to $O(\sqrt{s/n})$ dependence. In other words, the oracle estimator attains the optimal statistical performance.

However, the significant covariates index set $\mathcal{S}$ is typically unknown to the decision-maker in practice, and we will rely on the MCP method to recover this latent sparse structure. To better understand the rationale behind the MCP method, we start with the following weighted Lasso estimator:
\begin{equation}
\bm{\beta}^{W}(\bm{X},\bm{y},\bm{w} )\doteq \arg \min_{\beta }\
\left\{\frac{1}{|\mathcal{A}|}\sum_{j\in\mathcal{A}}f(r_j|\bm{x}_j^T\bm{\beta}) +\sum_{i=1}^{d}w_{i}|\beta_{i}|\right\} ,
\label{Wlasso_estimator}
\end{equation}
where $\bm{w} = (w_1,w_2,...,w_d)$ is a positive weights vector chosen by the decision-maker. Note that
when we set $w_{i}=\lambda $ for all $i$,  $\bm{\beta}^{W}(\bm{X},\bm{y},\bm{w})$ reduces to the standard Lasso estimator, which can be biased when the magnitude of true parameters is not too small. To recover the sparse structure and provide an unbiased parameter estimator, an ideal way to select $\{w_{i}\}$ is to set $w_{i}=\lambda >0$ for all $i\in \mathcal{S}^{c}$ and $w_{j}=0$ for all $j\in \mathcal{S}$. By doing so, when the weight $\lambda $ is large enough, the weighted Lasso estimator converges to the oracle estimator $\bm{\beta}^{oracle}(\bm{X},\bm{r})$. The benefits of the weighted Lasso method have attracted considerable attention recently, and various mechanisms have been proposed in the literature aiming to improve the weight selection process (\citealt{zou2006adaptive,huang2008adaptive,candes2008enhancing}). The MCP method, adopted in our paper,  reflect such a process.

In particular, we define the following MCP penalty function:
\begin{equation}
P_{\lambda,a}(x) \doteq\ \int_{0}^{|x|}\max (0,\lambda-\frac{1}{a}|t|)dt,\notag
\end{equation}
where $a$ and $\lambda$ are positive parameters selected by the decision-maker, and the MCP estimator can be presented as follows:
\begin{align}
\bm{\beta}^{MCP}(\bm{X},\bm{r},\lambda)\doteq\arg \min_{\beta} \mathcal{L}_{\mathcal{A}_{k}}(\bm{\beta}) = \arg \min_{\beta}\left\{\frac{1}{|\mathcal{A}|}\sum_{j\in\mathcal{A}}f(r_j|\bm{x}_j^T\bm{\beta})+\sum_{i=1}^dP_{\lambda,a}(\beta_i)\right\}.
\label{MCP_estimator}
\end{align}
Denote the index set for non-zero coefficients solutions in Equation (\ref{MCP_estimator}) as $\mathcal{J}\doteq \{j:\hat{\beta_j}\ne0\}$.
If the absolute value of  the MCP estimator in $\mathcal{J}$  is greater than $a\lambda$, then $P_{\lambda,a}(\beta_j)$ become constant parameters for all $j\in \mathcal{J}$.
Therefore, we will have  $P_{\lambda,a}(\beta_j) =\frac{1}{2}a\lambda^2 $ for $j\in \mathcal{J}$ and  $P_{\lambda,a}(\beta_j) =0$ otherwise. In other words, the statistical performance of solving the MCP estimator is equivalent to solving the following problem: $
\arg \min_{\bm{\beta}_{\mathcal{J}^c}=0,\bm{\beta}_{\mathcal{J}}} \{ \frac{1}{|\mathcal{A}|}\sum_{j\in\mathcal{A}}f(r_j|\bm{x}_j^T\bm{\beta})^{2}\}$. If $\mathcal{J}=\mathcal{S}$, then the MCP estimator converges to the oracle estimator.

Solving the MCP estimator can be challenging. \cite{liufolded} have shown that it is an NP-complete problem to find the MCP estimator by globally solving Equation (\ref{MCP_estimator}).  In the next subsection, we propose a local linear approximation method, the 2-step Weighted Lasso (2sWL) procedure, to tackle this challenge, and we demonstrate that the estimator solved by the 2sWL procedure will match the oracle estimator $\bm{\beta}^{oracle}$ with high probability.

\subsection{2-Step Weighted Lasso Procedure}  \label{sec:Model_2sWL}

The 2sWL procedure consists of two steps. We first solve a standard Lasso problem by setting all positive weights in Equation (\ref{Wlasso_estimator}) to a given parameter $\lambda_0$. Then, we use the Lasso estimator obtained in the first step to update the weights vector $\bm{w}$ by taking the first-order derivatives of the MCP penalty function, and applying this updated weight vector, we re-solve the weighted Lasso problem in Equation (\ref{Wlasso_estimator}) to obtain the MCP estimator. The procedures of 2sWL at time $t$ can be described as follows:
\vspace{2mm}
\begin{center}
	\begin{tabular}{l}
		\hline
		\textbf{2-Step Weighted Lasso (2sWL) Procedure}:\label{2sWL} \\ \hline
		$\quad\quad$ \textbf{Require}: input parameters $a$ and $\lambda$ \\
		$\quad\quad$ \textbf{Step 1}: solve a standard Lasso problem \\
		$\quad\quad$$\quad\quad$$\quad\quad$$\quad\quad$ $\bm{\beta}_{1}=
			\bm{\beta}^W(\bm{X},\bm{y},\lambda)$; \\
		$\quad\quad$ \textbf{Step 2}: update $w_j = \begin{cases}
		P_{a,\lambda}^{'}(|\beta_{1,j}|)\quad& \text{, for }\beta_{1,j}\ne 0\\
		\lambda \quad& \text{, for }\beta_{1,j}= 0
		\end{cases}$ \\
		$\quad\quad$$\quad\quad$$\quad\quad$and solve a weighted Lasso Problem \\
		$\quad\quad$$\quad\quad$$\quad\quad$$\quad\quad$$\bm{\hat\beta}_{2sWL}=
			\bm{\beta}^W(\bm{X},\bm{y},\bm{w}).$ \\ \hline
	\end{tabular}
\end{center}
\vspace{2mm}

As the 2sWL procedure is equivalent to solving the Lasso problem twice, the worst-case computation complexity for 2sWL is on same order as for the standard Lasso problem. In practice, we can initialize the second step procedure with a warm start from the first step of the Lasso solution, which further reduces the computation time.

The following proposition shows that the MCP estimator identified by the 2sWL procedure can recover the oracle estimator with high probability.
\begin{proposition}\label{MW_proposition:1}
	Under assumptions A.1, A.4, and A.5, if $\min\{|\beta^{true}_j|,\beta^{true}\ne0, j=1,2,...,d\}\ge \left(\frac{96s}{\kappa}+a\right)\lambda$, $a>\frac{96s}{\kappa}$, and $\lambda>\tilde{C}_3/n$, the MCP estimator solved under the 2sWL procedure, $\bm{\beta}^{MCP}$ satisfies the following inequality:
	\begin{align}
	\mathbbm{P}\left(\|\bm{\beta}^{MCP}-\bm{\beta}^{true}\|_2\le \sqrt{\frac{8s^2\sigma_2\sigma^2x_{\max}^2}{\mu_0^2n}}\right)\ge& 1-\delta_1(n)- \delta_2(n,\lambda)-\delta_3(n),
	\end{align}
	where $\delta_3(n) \doteq \exp\left(-\min\left\{1,\kappa^2/\left(192s\sigma_3x_{\max}^2(2+\sqrt{\sigma_3}x_{\max})\right)^2\right\}n\right)$,  $\delta_2(n,\lambda)\doteq d\exp\left(-\frac{n\lambda^2}{2x_{\max}^2}\right)+d\exp\left(-\frac{n\lambda^2}{2x_{\max}^2}\left((\frac{1}{4}-\frac{24s}{\kappa a})\min\left\{1,\frac{\mu_0}{8sx_{\max}^2}\right\}\right)^2\right)$, and $\tilde{C}_3$ is a positive constant independent of $d$ and $n$.
\end{proposition}

Comparing to the oracle estimator $\beta^{oracle}$ in Lemma \ref{MW_lemma_1}, the probability bound on the MCP estimator under the 2sWL procedure has two extra terms $\delta_2(n,\lambda)$ and $\delta_3(n)$, which depend on the covariate dimension $d$ and the sample size $n$. Note that as the sample size increases, these two extra terms decrease to $0$ at an exponential rate. In other words, as the sample size increases, $\bm{\beta}^{MCP}$ matches the true parameters with high probability and converges to the true parameters at the optimal convergence rate.

\subsection{$\epsilon$-decay Random Sampling Method} \label{sec:Model_sampling}
As bandit models involve exploitation and exploration, samples generated under exploitation typically are not iid. These non-iid samples pose challenges to the existing MCP literature, which relies on the assumption that samples are iid in establishing the convergence rate and regret bounds (see the proof of Proposition \ref{MW_proposition:1} in \S\ref{sec:Model_2sWL}).

In this research, to ensure that there are some iid samples generated in the online learning and decision-making process, we propose a $\epsilon$-decay random sampling method, in which the decision-maker draws random samples, with decreasing probability, by randomly selecting decisions from the decision set with equal probability. In particular, the $\epsilon$-decay random sampling method can be described as follows:

\noindent\textbf{\emph{$\epsilon$-decay Random Sampling Method}}: At time $t$, the decision-maker will draw a random sample, with probability $\min\{1,t_0/t\}$, where $t_0$ is a pre-determined positive constant. If the seller has decided to draw a random sample at time $t$, then the decision-maker will randomly select a decision from his decision set with equal probability. Otherwise, the decision-maker will follow a bi-level decision structure, which will be specified later, to determine the optimal decision to maximize his expected reward.

The $\epsilon$-decay random sampling method can balance the exploitation and exploration trade-off by ensuring that the decision-maker does not explore too much to significantly sacrifice his revenue performance (as the number of random samples exponentially decays in time) but has sufficient random samples to guarantee the quality of the parameter vector estimation. In particular, we can bound the random sample size in the following proposition.
\begin{proposition}\label{MW_Sampling}
	Let $ C_0\ge 10$, $T>\frac{(t_0+1)^2}{e^2}$, and $t_0=2C_0 |\mathcal{K}|$. Under the $\epsilon$-decay random sampling method, the random sample size $n_k$ for arm $k\in\mathcal{K}$ up to time $T$ is bounded by
	\begin{align}
	C_0(1+\log(T+1)-\log(t_0+1))\le n_k \le3C_0(1+\log(T)-\log(t_0))\notag
	\end{align}
	with probability at least $ 1-{2}/({T+1})$.
\end{proposition}

\subsection{G-MCP-Bandit Algorithm} \label{sec:Model_algorithm}

After establishing the MCP estimator's statistical property and the $\epsilon$-decay random sampling method, we are ready to present the proposed G-MCP-Bandit algorithm. The execution of the  G-MCP-Bandit algorithm can be summarized as follows:
\begin{center}
	\begin{tabular}{l}
		\hline
		\textbf{G-MCP-Bandit Algorithm}\label{fcp_bandit} \\ \hline
		\textbf{Require}: Input parameters $t_0,h,\lambda_{1,0},\lambda_{2,0}, a$. \\
		$\quad \quad \quad \;\;\;\;\;$  Initialize $\bm{\beta}^{random}_i(0)=\bm{\beta}^{whole}_i(0) = \bm{0}$, and $\mathcal{R}_{\pi_0}=\mathcal{W}_{\pi_0} = \phi$ for all $i\in\mathcal{K}$. \\
		$\quad $ \textbf{For} $t=1,2,....$ \textbf{do} \\
		$\quad $$\quad $ Observe $\bm{x}_{t}$.\\
		$\quad $$\quad $ Draw a binary random  variable $\mathcal{D}_t$, where $\mathcal{D}_t=1$ with probability $\min\{1,t_0/t\}$.\\
		$\quad $$\quad $ \textbf{If} $\mathcal{D}_t=1$ \\
		$\quad $$\quad $$\quad $$\quad $Assign $\pi _{t}$ to a random decision $k\in \mathcal{K}$ with probability $\mathbbm{P}(\pi_t =k)=1/|\mathcal{K}|$.\\
		$\quad $$\quad $$\quad $$\quad $Play decision $\pi _{t}$, observe  $r_t$, and update $\mathcal{R}_{\pi_t} = \mathcal{R}_{\pi_{t-1}}\cup\{\bm{x}_t,r_t\}$ and $\mathcal{W}_{\pi_t} = \mathcal{W}_{\pi_{t-1}}\cup\{\bm{x}_t,r_t\}$.\\
		$\quad $$\quad $ \textbf{Else} \\
		$\quad $$\quad $$\quad $$\quad $Construct the optimal decision set:\\
		$\quad $$\quad $$\quad $$\quad $$\quad \Pi_t = \left\{i: \mathbbm{E}[R_i|\bm{x}_t,\bm{\beta}_i^{random}(t-1)]\ge\max_{j\in\mathcal{K}}\mathbbm{E}[R_j|\bm{x}_t,\bm{\beta}^{random}_j(t-1)]-\frac{1}{2}h,\ i\in\mathcal{K}\right\}$.\\
		$\quad $$\quad $$\quad $$\quad $\textbf{If} $\Pi_t$ is a singleton\\
		$\quad $$\quad $$\quad $$\quad $$\quad $ Set  $\pi _{t}=\Pi_t$. \\

		$\quad $$\quad $$\quad $$\quad $\textbf{Else}  \\
		$\quad $$\quad $$\quad $$\quad $$\quad $ Set  $\pi_t = \arg\max_{k\in\Pi_t}\mathbbm{E}[R_k|\bm{x}_t,\bm{\beta}^{whole}_k(t-1)]$.\\

		$\quad $$\quad $$\quad $$\quad $\textbf{End If} \\

		$\quad $$\quad $$\quad $$\quad $Play decision $\pi _{t}$, observe $r_t$, and update  $\mathcal{W}_{\pi_t}=\mathcal{W}_{\pi_{t-1}}\cup\{\bm{x}_t,r_t\}$.\\
		$\quad $$\quad $ \textbf{End If} \\
		$\quad $$\quad $ For all $k\in\mathcal{K}$, { set $\lambda_{1}(t) = \lambda_{1,0}\sqrt{1+\frac{\log d}{\log(t+1)}}$ and $\lambda_{2}(t) = \lambda_{2,0}\sqrt{\frac{\log (t+1)+\log d}{t+1}}$.}\\
		$\quad $$\quad$ Update parameters $\bm{\beta}^{random}_k(t)$ via the 2sWL procedure with $(\mathcal{R}_{\pi_t},\lambda_{1}(t))$.\\
		$\quad $$\quad$ Update parameters $\bm{\beta}^{whole}_k(t)$ via the 2sWL procedure with $(\mathcal{W}_{\pi_t},\lambda_{2}(t))$.\\

		$\quad $ \textbf{End for} \\ \hline
	\end{tabular}
\end{center}
\medskip

Specifically, the decision-maker will start by assigning values for system parameters ($t_0$, $\mathcal{K}$, $s_{\max}$, and $h$), which can be optimized through tuning, and initialing two parameter vector estimators ($\bm{\beta}^{random}$ and $\bm{\beta}^{whole}$) and two sample datasets ($\mathcal{R}_{\pi_0}$ and $\mathcal{W}_{\pi_0}$, which represent the random sample set and the whole sample set, respectively). Then, for an incoming user at time $t$, the decision-maker will draw a random sample with probability $\min\{1,t_0/t\}$. There are two possibilities:
\begin{itemize}
	\item
	If the decision-maker decides to draw a random sample, then he will randomly choose a decision $k$ from his decision set $\mathcal{K}$ with equal probability of $1/|\mathcal{K}|$; then, he will implement the chosen decision (i.e., $\pi _{t}=k$), observe the user's response, and claim the corresponding reward; finally, the decision-maker will include the user's covariate vector and the corresponding reward $\{\bm{x}_t,r_t\}$ in both sample datasets, $\mathcal{R}_{\pi_t}$ and $\mathcal{W}_{\pi_t}$.
	\item
	If the decision-maker decides not to draw a random sample on this incoming user, then he will use the bi-level decision structure to determine his decision. In the upper-level decision-making process, the decision-maker will first construct an optimal decision set $\Pi_t$. Specifically, all decisions in the optimal decision set  $\Pi_t$ are estimated, based on the random sample MCP estimator $\bm{\beta}^{random}$, to yield expected rewards within $h/2$ of the maximum possible reward. If there is only one decision in the optimal decision set $\Pi_t$, then the decision-maker will implement this decision as the optimal decision; otherwise, the decision-maker will perform the lower-level decision-making process, in which the decision-maker will estimate, by using the whole sample MCP estimator $\bm{\beta}^{whole}$, the rewards for all decisions in the optimal decision set  $\Pi_t$ and select the decision that generates the highest expected reward. Then, observing the user's response to the optimal decision and collecting the corresponding reward, the decision-maker will only update the whole sample dataset $\mathcal{W}_{\pi_t}$ by appending  the user's covariate vector and the corresponding reward $\{\bm{x}_t,r_t\}$.
\end{itemize}

Finally, the decision-maker will reset two parameters, $\lambda_1$ and $\lambda_2$, and use the 2sWL procedure to update the random sample parameter vector estimator  $\bm{\beta}^{random}$ and the whole sample parameter vector estimator  $\bm{\beta}^{whole}$, based on sample data sets $\mathcal{R}_{\pi_t}$ and $\mathcal{W}_{\pi_t}$, respectively.

The expected cumulative regret upper bound for the G-MCP-Bandit algorithm can be established in the following theorem.

\begin{theorem} \label{GMCP_cum_regret}
	Under assumptions A.1-A.5, let $t_0=2C_0|\mathcal{K}|$, $T\ge T_0$, $\lambda_1 = C_5\sqrt{1+\frac{\log d}{\log(T+1)}}$, $\lambda_2 = C_4\sqrt{\frac{\log(T+1)+\log d}{T+1}}$, and $a\ge\frac{2304 s}{\kappa p^*}$. The cumulative regret of the G-MCP-Bandit algorithm up to time $T$ is upper bounded:
	\begin{align}
	R^{C}(T)
	&\le R_{\max}(T_0+|\mathcal{K}|)+(6R_{\max}|\mathcal{K}|C_0+41R_{\max}|\mathcal{K}|+2e^{4\sigma_2x_{\max}b}CR_{\max}^3|\mathcal{K}|x_{\max}^2C_{\bm{\beta}}s^3)\log (T+1)\notag\\
	&=O(|\mathcal{K}|s^2(s+\log d)\log T)\notag,
	\end{align}
	where         $T_0\ge \max\{\frac{16s^2x_{\max}^2}{p^*\mu_0},\frac{8}{C_1p^*}, \frac{2x_{\max}^2}{C_4^2}, \frac{2x_{\max}^2}{((\frac{1}{4}-\frac{192s}{p^*\kappa a})\min\{1,\frac{\mu_0p^*}{64sx_{\max}^2}\})^2C_4^2},\frac{32\log(s)s\sigma_3x_{\max}^2}{\mu_0p^*},
	(\frac{C_4(768s+a)}{\beta_{\min}})^4(1+\log d)^2,\frac{16s\log(s)x_{\max}^2}{C_hp^*\mu_0},\frac{128}{p^*}\}$ and $C_0$, $C_1$, $C_4$, $C_5$, $C_h$, and $C_{\beta}$ are constants independent of $T$.
\end{theorem}

Theorem \ref{GMCP_cum_regret} first shows that the expected cumulative regret of the G-MCP-Bandit algorithm over $T$ users is upper-bounded by $O(\log T)$. From \cite{goldenshluger2013linear}, we know that under low-dimensional settings, the expected cumulative regret under a linear bandit model is lower-bounded by $O(\log T)$, which is directly applicable to the high-dimensional settings. Note that the linear model is a special case of the generalized linear model. Therefore, the expected cumulative regret of the G-MCP-Bandit algorithm is also lower-bounded by $O(\log T)$. In other words, the G-MCP-Bandit algorithm achieves the optimal expected cumulative regret in the sample size dimension. This result comes from the facts that we can ensure $O(\log T)$ random samples at time $T$ via the $\epsilon$-decay random sampling method (Proposition \ref{MW_Sampling}) and that the MCP estimator is able to match the oracle estimator with high probability (Proposition  \ref{MW_proposition:1}). Further, when compared to the Lasso-Bandit algorithm proposed by \cite{bastani2015online} for the linear model under high-dimensional settings, the G-MCP-Bandit algorithm reduces the dependence of the expected cumulative regret on the sample size dimension from $O(\log^2 T)$ to $O(\log T)$. As the G-MCP-Bandit algorithm achieves the optimal expected cumulative regret and improves on the cumulative regret performance from existing high-dimensional bandit algorithms in the sample size dimension, we expect that the G-MCP-Bandit algorithm will be able to improve the learning process of the parameter vector estimation with limited samples and perform favorably in the cumulative regret performance even in sample-poor regions.

Theorem \ref{GMCP_cum_regret} also demonstrates that the cumulative regret of the G-MCP-Bandit algorithm in the high-dimensional covariate vector $d$ is upper-bounded by $O(\log d)$. This bound presents a significant improvement over other classic bandit algorithms (\citealt{goldenshluger2013linear,abbasi2012online,dani2008stochastic}), which yield polynomial dependence on $d$, and is also a tighter bound than the Lasso-type algorithm (i.e., $O(\log^2 d)$ in \citealt{bastani2015online}). This improvement is of particular importance in high-dimensional settings, in which the covariate dimension can be extremely large,  and it suggests that the G-MCP-Bandit algorithm can bring substantial regret reduction comparing to existing bandit algorithms, which we will illustrate through experiments in \S \ref{sec:empirical}.

\section{Key Steps of Regret Analysis for the G-MCP-Bandit Algorithm} \label{sec:steps}

In this section, we provide the abridged technical proofs for Theorem \ref{GMCP_cum_regret} -- the main theorem in this paper. Specifically, we briefly lay out four key steps in establishing the expected cumulative regret upper bound for the G-MCP-Bandit algorithm. In the first step, we highlight the influence of non-iid data, inherited from the multi-armed bandit model, and provide the statistical convergence property for the MCP estimator under partially iid samples. Applying these results to the G-MCP-Bandit algorithm, in the second and third steps, we can establish the convergence properties for both the random sample estimator, which is based on samples generated only through the $\epsilon$-decay random sampling method, and the whole sample estimator, which uses all available samples. Finally, in the last step, we establish the total expected cumulative regret by separating the regret up to time $T$ into three segments and providing a bound for each segment. The main structure and sequence of our proving steps described above are first introduced by \citet{bastani2015online}, which presents their expected regret analysis for a linear bandit model (i.e., LASSO-Bandit algorithm) in a similar sequence. We will largely follow their presentation structure, but with different steps, proving techniques, and convergence properties, to illustrate the key steps in analyzing the G-MCP-Bandit algorithm.

\subsection{General Non-iid Sample Estimator}

Note that the restricted eigenvalue condition (A.4 in \S\ref{sec:Model}), a necessary condition for high-dimensional statistics, is typically established in the literature for iid samples. Yet, in this research, we consider the G-MCP-Bandit algorithm, under which only part of the samples are iid, so we first need to show that the restricted eigenvalue condition continues to hold for partially iid samples (Lemma \ref{lemma:2.9} in E-Companion). Then, we can establish some general results for the MCP estimator under non-iid data.

We denote $\mathcal{W}$ as the whole sample set that contains all users' covariate vectors $\bm{X}$ and the corresponding rewards $\bm{r}$ for an arbitrary decision $k\in\mathcal{K}$ up to time $T$, and $\bm{\beta}^{MCP}$ as the MCP estimator for the parameter vector corresponding to decision $k$. Note that as samples in $\mathcal{W}$ are not iid, standard MCP convergence results (\citealt{fan2014strong,fan2018lamm}) cannot be directly applied. Recall that we proposed the $\epsilon$-decay random sampling method and that samples generated under this method are iid. Therefore, there exists a subset $\mathcal{A}\subseteq \mathcal{W}$ such that all samples in this subset are iid from the distribution $\mathcal{P}_{\bm{X}}$. The next step is to show that when the cardinality of $\mathcal{A}$ (i.e., $|\mathcal{A}|$) is large enough, $\bm{\beta}^{MCP}$ will converge to the true parameters $\bm{\beta}^{true}$.

\begin{proposition}\label{MW_Non_iid_MCP}
	Denote the whole sample size as $n$ and the sub-sample set, containing only iid random samples, as $\mathcal{A}$. Under assumptions A.1, A.4, and A.5, if
	$\beta_{\min}\ge (\frac{96ns}{\kappa|\mathcal{A}|}+a)\lambda$ and $a>\frac{96ns}{\kappa|\mathcal{A}|}$, then for $\zeta\le \frac{\mu_0|\mathcal{A}|\sqrt{C_2\lambda}}{2n}$, the following inequality hold for the MCP estimator under the 2sWL procedure $\bm{\beta}^{MCP}$:
	\begin{align}
	\mathbbm{P}\left(\|\bm{\beta}^{MCP}-\bm{\beta}^{true}\|_2\le \frac{2n\zeta}{|\mathcal{A}|\mu_0}\right)\ge 1-\delta_2(|\mathcal{A}|/n,\lambda)-\delta_3(|\mathcal{A}|)-\delta_4(|\mathcal{A}|/n,\zeta).\label{eq:MW_Non_iid_MCP:1}
	\end{align}
	Moreover, if $\lambda>\frac{C_3n}{|\mathcal{A}|^2}$ and $|\mathcal{A}|\ge \frac{2s^2x_{\max}^2}{\mu_0}$, then we have the following result:
	\begin{align}
	\mathbbm{P}\left(\|\bm{\beta}^{MCP}-\bm{\beta}^{true}\|_2\le \sqrt{\frac{8s^2\sigma_2\sigma^2x_{\max}^2n}{\mu_0^2|\mathcal{A}|^2}}\right)\ge 1-\delta_1(|\mathcal{A}|/n)-\delta_2(|\mathcal{A}|/n,\lambda)-\delta_3(|\mathcal{A}|),\label{eq:MW_Non_iid_MCP:2}
	\end{align}
	where $C_2$ and $C_3$ are positive constants and $
	\delta_4(|\mathcal{A}|/n,\zeta)\doteq
	2s\exp\left(-\frac{|\mathcal{A}|\mu_0}{4sLx_{\max}^2}\right)+s\exp\left(-\frac{n\zeta^2}{2\sigma^2x_{\max}^2}\right)
	$.
\end{proposition}

Proposition \ref{MW_Non_iid_MCP} describes the statistical properties of the non-iid MCP estimators under the 2sWL procedure. First, if we don't require the iid sample size $|\mathcal{A}|$ to be sufficiently large, then the MCP estimator's statistical performance is given by Equation \eqref{eq:MW_Non_iid_MCP:1}. If we set $\zeta$ to be on the order of $O(s/\sqrt{n})$, then $\|\bm{\beta}^{MCP}-\bm{\beta}^{true}\|$ is on the order of $O(\sqrt{s^2n/|\mathcal{A}|^2})$, which matches the result of Equation \eqref{eq:MW_Non_iid_MCP:2}. Meanwhile, however, $\delta_4(|\mathcal{A}|/n,\zeta)$ in Equation \eqref{eq:MW_Non_iid_MCP:1} becomes a positive constant asymptotically, which implies that when $|\mathcal{A}|$ is not large enough, the MCP estimator may not warrant good statistical performance. Yet, when we have sufficient iid samples (i.e., $|\mathcal{A}|\ge \frac{2s^2x_{\max}^2}{\mu_0}$), Equation \eqref{eq:MW_Non_iid_MCP:2} suggests that the MCP estimator not only guarantees a better statistical convergence ($O(\sqrt{s^2n/|\mathcal{A}|^2})$) but also attains probability $1$ when the whole sample size $n$ and the iid sample size $|\mathcal{A}|$ go to infinity.

Moreover, Proposition \ref{MW_Non_iid_MCP} shows the necessity of generating iid random samples in high-dimension bandit settings. Non-iid samples are inevitable in online learning and decision-making process, so ensuring desired asymptotical performance of the parameter vector estimation in high-dimensional settings can only be achieved through generating sufficient number of iid samples, as shown in Proposition \ref{MW_Non_iid_MCP}. We will show in next two subsections that the size of iid samples generated under the $\epsilon$-decay random sampling method is on the order of $O(\log T)$ and that the size can be further improved to the order of $O(T)$ under the bi-level decision structure in the G-MCP-Bandit algorithm.

\subsection{Estimator from Random Samples up to Time $T$}
In Proposition \ref{MW_Non_iid_MCP}, we show that the MCP estimator will converge to the oracle parameter as long as the sample set contains a sufficient number of iid samples. Recall that in our proposed G-MCP-Bandit algorithm, samples generated by  the $\epsilon$-decay random sampling method are iid, and the size of these iid samples is on the order of $O(\log (T))$; see Proposition \ref{MW_Sampling}. Combining these observations, we can establish the statistical performance of the MCP estimator under the G-MCP-Bandit algorithm in the following proposition.

\begin{proposition} \label{MW_Estimator_Random}
  	If assumptions A.1, A.3, A.4, and A.5 hold, then the MCP estimator under the G-MCP-Bandit algorithm $\bm{\beta}^{MCP}$ will satisfy the following inequality:
	\begin{align}
	\mathbbm{P}\left(\|\bm{\beta}^{MCP}-\bm{\beta}^{true}\|_1\le  \min\left\{\frac{1}{\sigma_2x_{\max}},\frac{h}{4e\sigma_2R_{\max}x_{\max}}\right\}\right)\ge 1-\frac{15}{T+1}.\notag
	\end{align}
\end{proposition}

\subsection{Estimator from Whole Samples up to Time $T$}

In addition to the iid samples generated by the $\epsilon$-decay random sampling method, other samples can also be iid and used to improve the statistical performance of the MCP estimator. To intuit, recall that in the G-MCP-Bandit algorithm, when the user is not selected to perform a random sampling, the decision-maker will use the bi-level structure to determine the optimal decision to maximize his expected reward. In the upper-level decision-making process, only iid samples will be used (as $\bm{\beta}^{random}$  is the MCP estimator based on samples generated only by the  $\epsilon$-decay random sampling method) to determine the candidate(s) for the optimal decision set. From Proposition \ref{MW_Estimator_Random}, we know that this random sample MCP estimator will not be far away from its true parameter values. In other words,  if we define the event that the random sample MCP estimator at time $t$ is within a given distance from its true parameter as event  $\mathcal{E}_6$:
\begin{align}
\mathcal{E}_6\dot=  \left\{\|\bm{\beta}^{random}_k(t)-\bm{\beta}^{true}_k\|_1\le \min\left\{\frac{1}{\sigma_2x_{\max}},\frac{h}{4e\sigma_2R_{\max}x_{\max}}\right\},\ k\in\mathcal{K}\right\},\label{eq:event_6}
\end{align}
then event  $\mathcal{E}_6$ will happen with high probability. Further, conditioning on event $\mathcal{E}_6$, we can directly verify that for any $\bm{x}\in U_k$, $k\in\mathcal{K}$, the following inequality holds:
\begin{align}
\mathbbm{E}(R_k|\bm{x},\bm{\beta}^{random}_k(t))\ge\max_{j\ne k}\mathbbm{E}(R_j|\bm{x},\bm{\beta}^{random}_j(t))+ \frac{h}{2}.\label{Xue_h_half}
\end{align}
Therefore, if using Equation \eqref{Xue_h_half} as the selecting criterion, the decision-maker will be able to choose the optimal decision $k$ for any $x\in U_k, k\in\mathcal{K}$ with high probability.
Formally, we can bound the total number of times under which event $x\in U_k$ and event $\mathcal{E}_6$  happen simultaneously.
In particular, we define $
M(i) \dot=  \mathbbm{E}\left[\sum_{j=1}^{T+1}\mathbbm{1}(\bm{x}_j\in U_k,\mathcal{E}_6, x_j\notin\mathcal{R}_k)|\mathcal{F}_{i}\right]$ for $i \in \{0, 1,2,.., T+1\}$,
where $\mathcal{F}_i = \{(\bm{x}_j,r_j)$ for $j\le i\}$ and $\mathcal{R}_k$ is the set containing iid samples generated through the $\epsilon$-decay random sampling method for arm $k$. Then, $\{M(i)\}$ is a martingale with bounded difference $|M(i)-M(i+1)|\le 1$ for $i=0,1,2,...,T$, and we can bound the value of $M(T+1)$ in the following proposition:
\begin{proposition}\label{MW_size_iid}
	If $T\ge \max\{30,4C_0|\mathcal{K}|\}$, then $\mathbbm{P}\left(M(T+1)\le \frac{p^*(T+1)}{8}\right)\le \exp\left(-\frac{(p^*)^2T}{128}\right)$.
\end{proposition}

Intuitively, Proposition \ref{MW_size_iid} suggests that with high probability, the actual iid sample size in $U_k$ for decision $k$ will be on the order of $O(T)$ instead of $O(\log T)$. This improvement is the reason why the whole sample MCP estimator $\bm{\beta}^{whole}$ used in the lower-level decision-making process has a better statistical performance, compared to the random sample MCP estimator $\bm{\beta}^{random}$ used in the upper-level decision-making process. Specifically, we can establish the convergence property for the whole sample MCP estimator in the following proposition.

\begin{proposition}\label{MW_all_sample_estimator}
	If assumptions A.1, A.3, A.4, and A.5 hold, then at time $T$ the whole sample MCP estimator under the G-MCP-Bandit algorithm $\bm{\beta}^{whole}$ will satisfy the following inequality:
	\begin{align}
	\mathbbm{P}\left(\|\bm{\beta}^{whole}(T)-\bm{\beta}^{true}\|_2 \le \sqrt{C_{\bm{\beta}}\frac{s^2}{T+1}}\right)\ge 1- \frac{13}{T+1},\notag
	\end{align}
	where $C_{\beta}$ is a positive constant.
\end{proposition}

\subsection{Cumulated Regret Up To Time $T$}

Finally, to bound the cumulative regret for the G-MCP-Bandit algorithm, we need to divide the time, up to time $T$, into three groups and provide a upper bound for each group.

The first group contains all samples before time $T_0$ and all random samples up to time $T$. Note that before time $T_0$ (the explicit expression for $T_0$ is given in the proof of Theorem \ref{GMCP_cum_regret} in E-Companion), the decision-maker does not have sufficient samples to accurately estimate covariate parameter vectors. Hence, the reward under the G-MCP-Bandit algorithm will suffer and be sub-optimal compared to that of the oracle case. We can bound the cumulative regret by the worst case performance: $R_{\max}T_0+R_{\max}|\mathcal{K}|(2+6C_0\log T)$, where the first part of this cumulative regret is for all samples before time $T_0$ and the second part is for all random samples up to time $T$.

Next, we will segment the $t > T_0$ case into two groups, depending on whether we can accurately estimate covariate parameter vectors by using only random samples. In particular, the second group includes cases where $t > T_0$ and the random-sample-based estimators are not accurate (i.e., event $\mathcal{E}_6$ doesn't hold). Under those scenarios, inevitably, the decision-maker's decisions will be suboptimal with high probability. However, note that as the size of iid samples increases in $t$, the probability of event $\mathcal{E}_6$ not occurring decreases.
We can bound the cumulative regret for the second group by $15R_{\max}|\mathcal{K}|\log(T+1)$.

The last group includes scenarios where $t > T_0$ and the random sample estimators are accurate enough. 
Benefiting from the improved estimation accuracy (Proposition \ref{MW_all_sample_estimator}), we can bound the cumulative regret for the last group as $(26R_{\max}|\mathcal{K}|+4e^{4\sigma^2x_{\max}b}CR_{\max}^3|\mathcal{K}|x_{\max}^2C_{\beta}s^3)\log(T)$. Combining the cumulative regret for all three groups, Theorem \ref{GMCP_cum_regret} directly follows.

\section{Empirical Experiments}  \label{sec:empirical}

In this section, we will benchmark the G-MCP-Bandit algorithm to one high-dimensional bandit algorithm, Lasso-Bandit by \citealt{bastani2015online}, and to two other bandit algorithms that were not specifically developed for high-dimensional problems, OLS-Bandit by \citealt{goldenshluger2013linear} and OFUL by \citealt{abbasi2011improved}. In particular, we seek answers to the following two questions: How does the performance of the G-MCP-Bandit algorithm compare to other bandit algorithms? And how is the performance of the G-MCP-Bandit algorithm influenced by the data availability ($T$), the data dimensions ($s$ and $d$), and the size of the decision set ($K$)?

To this end, we start with two synthetic-data-based experiments in \S\ref{sec:empirical_synthetic} and conduct two additional experiments based on real datasets, the warfarin dosing patient data in \S \ref{sec:empirical_warfarin} and the Tencent search advertising data in \S \ref{sec:empirical_tencent}, respectively. Note that the algorithms and theoretical bounds of OFUL, OLS-Bandit, and Lasso-Bandit are developed under the assumption that the reward function follows the linear model, which is a special case in the G-MCP-Bandit algorithm. Therefore, for fair comparison, we specify the underlying reward function for the G-MCP-Bandit algorithm to follow the same linear model (i.e., the reward under decision $k$ for a user with covariate vector $\bm{x}$ takes the form of $R_{k}(\bm{x})=\bm{x}^T\bm{\beta}_k^{true}+\epsilon$, where $\epsilon$ is a $\sigma$-gaussian random variable) in all experiments, except the Tencent search advertising data experiment, in which we explore the performance of the G-MCP-Bandit model under both the linear model and the logistic model.

\subsection{Synthetic Data (Linear Model)}  \label{sec:empirical_synthetic}

In the first synthetic data experiment, we fix the size of the decision set $K$ and focus on the impacts of the data dimensions, $s$ and $d$, and the data availability, $T$, on learning algorithms' cumulative regret performance. In particular, we consider a two-arm bandit setting (i.e., $K=2$). To simulate different sparsity levels, we vary the covariate dimension $d=\{10,10^2,10^3,10^4\}$ and keep the dimension for significant covariates unchanged at $s=5$. Therefore, as the covariate dimension $d$ increases, the data become sparser. The underlying true parameter vectors for covariates are arbitrarily set to be $\bm{\beta}_{1} = (1,2,3,4,5,0,0,...)$ for the first arm and  $\bm{\beta}_2 = 1.1\cdot\bm{\beta}_1$ for the second arm. For each incoming user, we randomly draw her covariate vector from $N(0,I_{d\times d})$ and the error term in the linear model $\epsilon$ from $N(0,1)$. Finally, we use the same parameter $\lambda $ value in both the Lasso-Bandit algorithm and the G-MCP-Bandit algorithm and select the unique parameter for the G-MCP-Bandit algorithm $a$ at $2$. For each algorithm, we perform 100 trials and report the average cumulative regret for OFUL, OLS-Bandit, Lasso-Bandit, and G-MCP-Bandit (under the linear model) in Figure \ref{fig_simu_influ_d}.

\begin{figure}[h!]
	\FIGURE
	{
		\subfigure[d=100]{
			\includegraphics[width=0.5\linewidth]{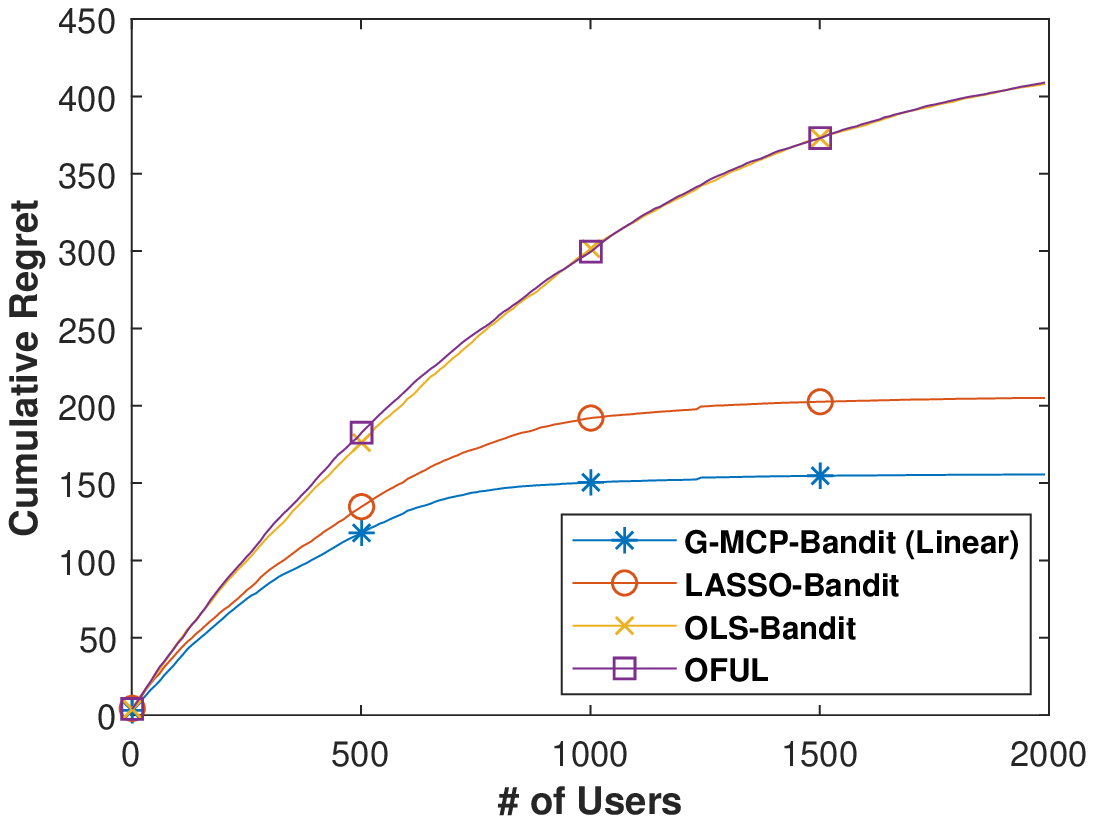}}
		\subfigure[T=1000]{
			\includegraphics[width=0.5\linewidth]{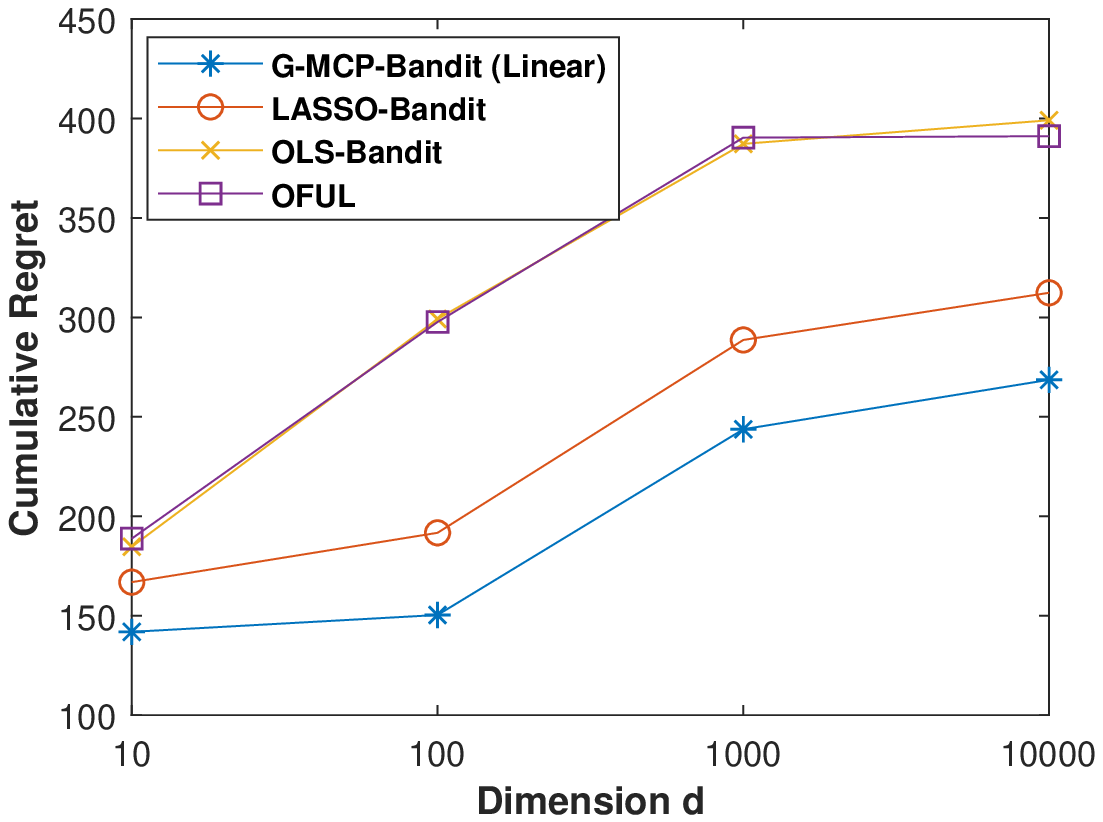}}
	}
	{Synthetic study 1: The impact of $T$ and $d$ on the cumulative regret, where $K=2$ and $s=5$.\label{fig_simu_influ_d}}{}
\end{figure}

Overall, we observe that the G-MCP-Bandit algorithm significantly outperforms OFUL, OLS-Bandit, and Lasso-Bandit and achieves the lowest cumulative regret. Facing only two decisions/arms, the decision-maker can easily identify the optimal arm, and therefore OFUL and OLS-Bandit, both of which are not specifically designed for high-dimensional settings, perform nearly identically. Lasso-Bandit and G-MCP-Bandit could benefit from their abilities to recover the sparse structure and identify the significant covariates. Therefore, compared to OFUL and OLS-Bandit, Lasso-Bandit and G-MCP-Bandit can improve their parameters estimations, especially under high-dimensional settings, and perform substantially better. Further, the improvement of the cumulative regret performance of G-MCP-Bandit over Lasso-Bandit follows from the facts that the MCP estimator is unbiased and could improve the sparse structure discovery. Next, we will discuss the influence of sample size $T$ and the covariate dimension $d$ on these algorithms' cumulative regret performance.

Figure \ref{fig_simu_influ_d}(a) illustrates the influence of the sample size $T$ on the cumulative regret for the case where $d=100$ (other cases exhibit a similar pattern and are therefore omitted){\footnote{In all four experiments where $d\in\{10,10^2,10^3,10^4\}$, we simulated the sample size up to $10,000$ and observe that the G-MCP-Bandit algorithm's cumulative regret seems to be stabilized before $T=2000$. Therefore, we only plot for the first $2000$ samples to avoid duplications.}}. As we have proven that G-MCP-Bandit provides the optimal time dependence under both low-dimensional and high-dimensional settings (Theorem \ref{GMCP_cum_regret}), G-MCP-bandit is guaranteed to strictly improve on the cumulative regret performance from Lasso-Bandit, especially when T is not too small. Note that facing insufficient samples, all algorithms fail to accurately learn parameter vectors and therefore perform poorly. As the sample size increases, the G-MCP-bandit algorithm is able to, in an expeditious fashion, unveil the underlying sparse data structure, accurately estimate parameters vectors, and outperform all other benchmarks. For example, we observe that the regret reduction of G-MCP-Bandit over all other algorithms is larger than $10\%$ when the sample size $T$ is larger than $350$.  This observation echoes our theoretical findings that the G-MCP-Bandit algorithm attains the optimal regret bound in sample size dimension $O(\log T)$.

We also observe that the benefits of G-MCP-Bandit over other three algorithms appear to increase in the data sparsity level. Figure \ref{fig_simu_influ_d}(b) presents the influence of the covariate dimension $d$ on the cumulative regret for the case where $T=1000$. Recall that we fixed the dimension for significant covariates $s=5$. Therefore, as the covariate dimension $d$ increases, the data become sparser (i.e., $d/s$ increases). As expected, the cumulative regret for all four algorithms increases in the covariate dimension $d$, but at different rates. On the one hand, both OLS-Bandit and OFUL lack the ability to recover the sparse data structure and are ill suited for high-dimensional problems. On the other hand, Lasso-Bandit and G-MCP-Bandit, which adopt different statistical learning methods for the sparse structure discovery and are designed for high-dimensional problems, have lower cumulative regret that increases in $d$ at a slower rate. Further, we notice that the G-MCP-Bandit algorithm has the least increase in cumulative regret among all four algorithms, which confirms our theoretical finding in Theorem \ref{GMCP_cum_regret}: The G-MCP-Bandit algorithm has a better dependence on the covariate dimension $O(\log d)$ than Lasso-Bandit $O(\log^2 d)$, OFUL, and OLS-Bandit (the last two algorithms have polynomial bounds in $d$).

In the second synthetic data experiment, we study the influence of the size of decision set by varying $K=\{2,5,10,20,50,100\}$ and keeping the data dimensions unchanged ($s=5$ and $d=100$). For each decision, we randomly draw the coefficients parameter vector for its significant covariates 
from a standard normal distribution. 
Finally, we keep other parameters the same as in the first synthetic data experiment. Figure \ref{fig_simu_influ_arm} plots the average cumulative regret for OFUL, OLS-Bandit, Lasso-Bandit, and G-MCP-Bandit (under the linear model).

\begin{figure}[h!]
	\FIGURE
	{
		\subfigure[K=2]{
			\includegraphics[width=0.35\linewidth]{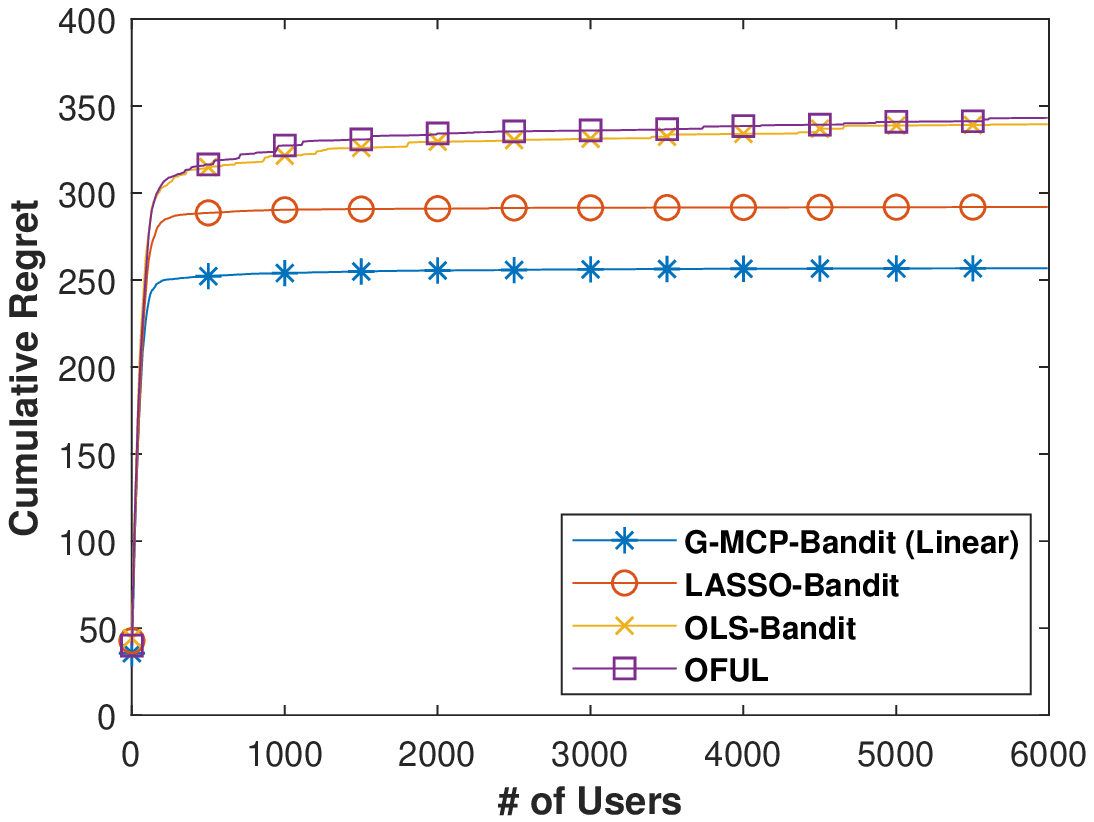}}
		\subfigure[K=10]{
			\includegraphics[width=0.35\linewidth]{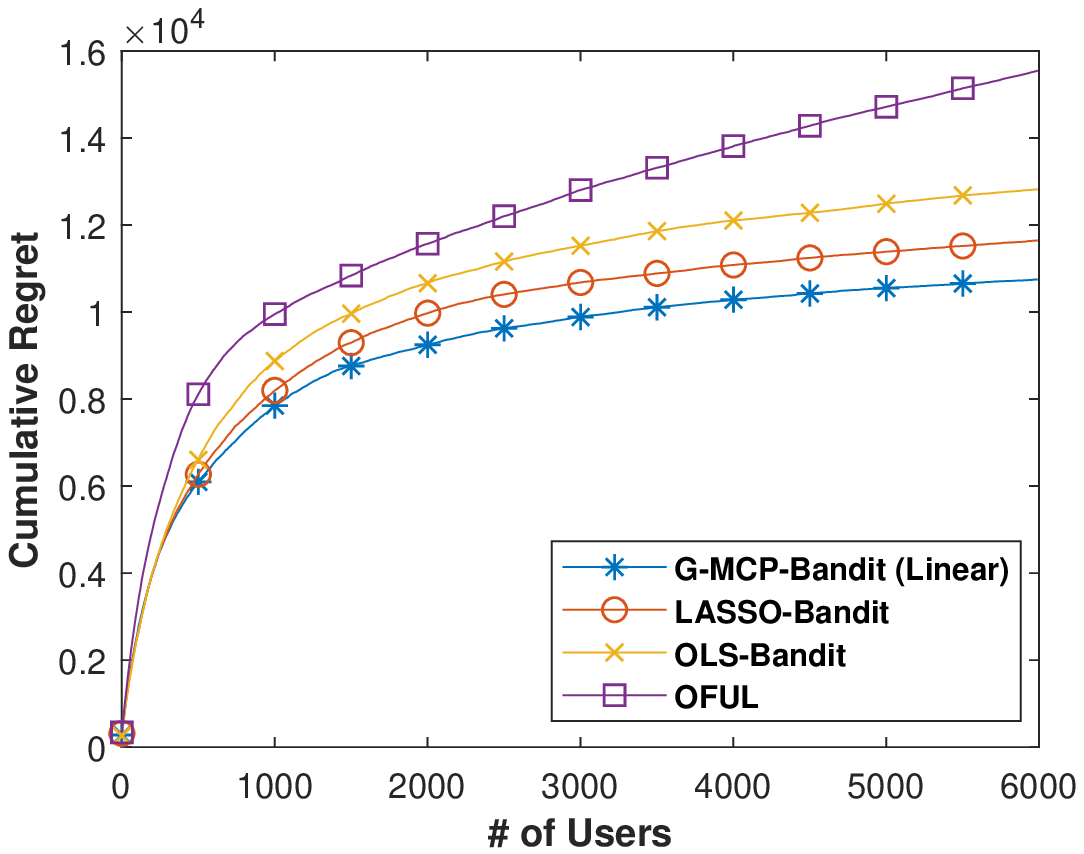}}
		\subfigure[T=6000]{
			\includegraphics[width=0.35\linewidth]{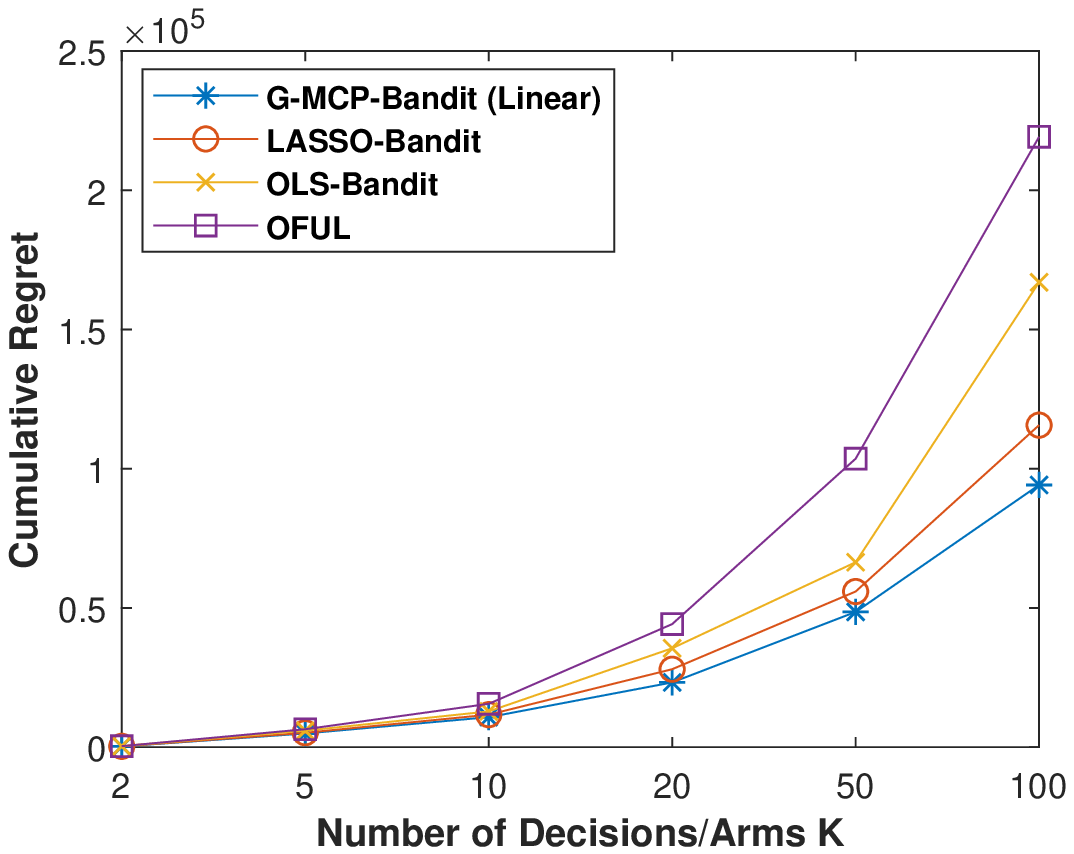}}
	}
	{Synthetic study 2: The impact of $T$ and $K$ on the cumulative regret, where $d=100$ and $s=5$.\label{fig_simu_influ_arm}}{}
\end{figure}

We observe that the benefits of adopting G-MCP-Bandit over the other three algorithms increases with the size of the decision set. In particular,  as $K$ increases, the cumulative regret gap between G-MCP-Bandit and any other algorithm grows; see Figure \ref{fig_simu_influ_arm}(c). This observation is as expected. To intuit, note that as we add more possible decisions into the decision set, the complexity and difficulty for the decision-maker to select the optimal decision grow for two main reasons. First, the decision-maker will need more samples to identify the significant covariates and estimate the parameter vectors. Second, as the number of decisions increases, the process of comparing the expected rewards among all decisions and selecting the optimal decision becomes more vulnerable to estimation errors. Therefore, we should expect that as the number of arms increases, the amount of samples required for these algorithms to accurately learn the parameter vectors and select the optimal decision will increase as well.

Figure \ref{fig_simu_influ_arm}(a) and Figure \ref{fig_simu_influ_arm}(b) plot the cumulative regret for the case of two arms and ten arms, respectively. Clearly, the decision-maker needs far more samples before his cumulative regret can be stabilized in the case of ten arms than in the case of two arms. 
Therefore, the cumulative regret performance under all algorithms suffers from the increasing size of the decision set. 
As discussed earlier, the G-MCP-Bandit algorithm attains the optimal bound in the sample size dimension and is able to learn the sparse data structure and provide accurate unbiased estimators for parameters vectors. Hence, we observe that the benefits of adopting the G-MCP-Bandit algorithm over other algorithms are amplified as the number of arms increases, as illustrated in Figure \ref{fig_simu_influ_arm}(c).

\subsection{Warfarin Dosing Patient Data (Linear Model)} \label{sec:empirical_warfarin}

In the first real-data-based experiment, we considers a health care problem in which physicians determine the optimal personalized warfarin dosage for incoming patients (\citealt{international2009estimation}).  Using the same dataset, \citet{bastani2015online} demonstrate that the Lasso-Bandit algorithm outperforms other existing bandit algorithms, including OFUL-LS (\citealt{abbasi2011improved}), OFUL-EG (\citealt{abbasi2012online}), and OLS-Bandit (\citealt{goldenshluger2013linear}). The warfarin dosing patient data contains detailed covariates (the size of covariates used in our experiment is $93$) for $5,700$ patients, including demographic, diagnosis, and genetic information that can be used to predict the optimal warfarin dosage.

We apply the G-MCP-Bandit algorithm to the warfarin dosing patient dataset to evaluate its performance in practical decision-making contexts where the technical assumptions specified early in \S\ref{sec:Model} may not hold. Following \citet{bastani2015online}, we formulate this problem as a 3-armed bandit with covariates under the linear model. 

\begin{figure}[h!]
	\FIGURE
	{
		\includegraphics[width=0.5 \linewidth]{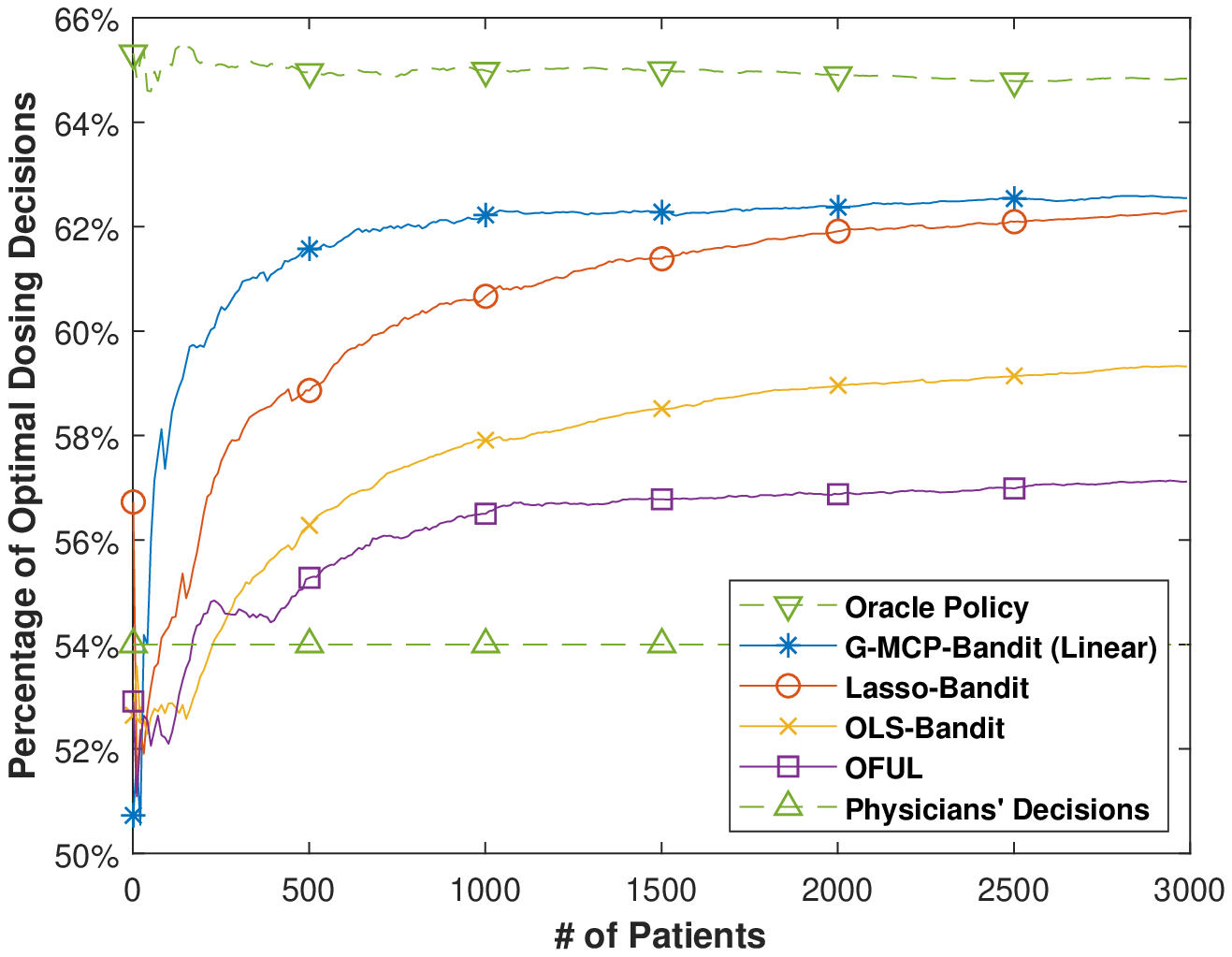}
	}
	{Warfarin dosing experiment: The percentage of optimal warfarin dosing decisions. \label{fig_warfrain}}{}
\end{figure}

Figure \ref{fig_warfrain} compares the average fraction of optimal/correct dosing decisions under G-MCP-Bandit (under the linear model) to those under OFUL, OLS-Bandit, Lasso-Bandit, actual physicians' decisions, and the oracle policy. We observe that as long as the sample size is not too small (e.g., the number of patients exceeds 40), the G-MCP-Bandit algorithm will outperform physicians' decisions, OLS-Bandit,  Lasso-Bandit, and OFUL. However, when there are very limited samples ($<40$ patients), the physicians' static decisions (i.e., always recommend medium dose) perform the best, with a stable optimal percentage of $54\%$. Without sufficient samples, all learning algorithms are unable to accurately learn the parameter vectors for patients' covariates, and consequently they behave suboptimally.

As the sample size increases, all learning algorithms are able to update their estimation of parameter vectors and eventually outperform the physicians' static decisions. Among all learning algorithms, the G-MCP-Bandit algorithm requires the fewest samples (i.e., $T>40$ for G-MCP-Bandit, $T>90$ for Lasso-Bandit, $T>180$ for OFUL, $T>220$ for OLS-Bandit) to provide better dosing decisions than physicians.

\subsection{Tencent Search Advertising Data (Linear \& Logistic Models)}\label{sec:empirical_tencent}

In the last experiment, we scale up the dataset's dimensionality to consider a search advertising problem at Tencent. The Tencent search advertising dataset is collected by Tencent's proprietary search engine, soso.com, and it documents the interaction sessions between users and the search engine (\citealt{Tencent2012}). In the dataset, each session contains a user's demographic information (age and gender), the query issued by the user (combinations of keywords), ads information (title, URL address, and advertiser ID), the user's response (click or not), etc. This dataset is high-dimensional with sparse data structure and contains millions of observations and covariates. To put the size of the dataset into perspective, it contains $149,639,105$ session entries, more than half a million ads, more than one million unique keywords, and more than $26$ million unique queries.

For illustration purposes, we focus on a three-ad experiment\footnote{We have extended the experiment to include more ads, but we find that doing so will not qualitatively change our observations and insights but considerably increases the computation time. Therefore, we decide to focus on this three-ad experiment in the paper.} (with ad IDs 21162526, 3065545, and 3827183). Each of these three ads has an average CTR higher than $2\%$ and more than $100,000$ session entries, which provide reasonably accurate estimation for parameter vectors (see next paragraph for more discussions). In total, there are $849,338$ session entries with $169,744$ unique queries and $8$ covariates for users' demographic information. As the search engine receives payment from advertisers only when the user has clicked the sponsored ad, we arbitrarily assume that advertisers will award the search engine $\$1$, $\$5$, and $\$10$ for each clicked ad, respectively.

Figure \ref{fig_tencent} plots the the average revenue performance under OFUL, OLS-Bandit, Lasso-Bandit, the random policy, the oracle policy, and G-MCP-Bandit (under both linear and logistic models). It is worth noting that the ``true'' oracle policy is impossible to implement, as the true parameter vectors are unknown, or at least have considerable variance even when all session entries in the dataset are used for estimation. Therefore, the oracle policy in the experiment represents the scenario when the search engine has access to all data to estimate these parameter vectors and make ad selection decisions. In addition, we introduce the random policy as another benchmark to simulate the scenario in which the search engine will randomly recommend an ad with equal probability to an incoming user. Finally, note that the CTR prediction is binary in nature (i.e., click or not). We therefore include the G-MCP-Bandit algorithm under the logistic model and compare it to the G-MCP-Bandit algorithm under the linear model to study the influence of the underlying model choice. In the experiment, we simulate incoming users by permuting their covariate vectors randomly. For each algorithm, we perform 100 trials and report the average revenue with $5000$ users, which seems to be sufficient for the G-MCP-Bandit algorithm to converge.

\begin{figure}[h!]
	\FIGURE
	{
		\includegraphics[width=0.5 \linewidth]{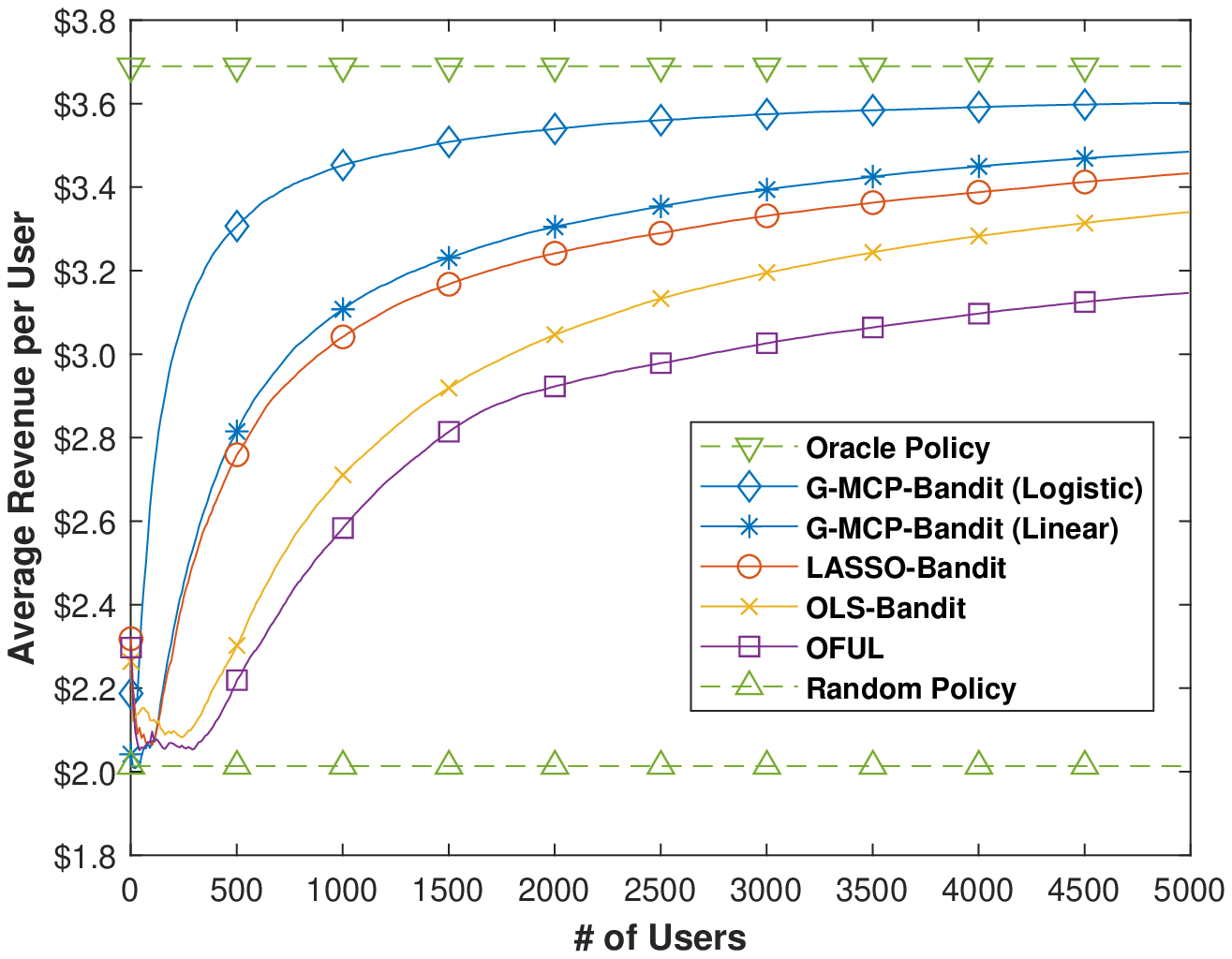}
	}
	{Tencent search advertising experiment: The average revenue under different algorithms.\label{fig_tencent}}{}
\end{figure}

We can show that all learning algorithms generate higher average revenue than the random policy for any number of users and that the G-MCP-Bandit algorithm outperforms other algorithms under most scenarios. Specifically, when comparing all algorithms under the same linear model, we observe that the G-MCP-Bandit algorithm (under the linear model) has better average revenue performance than OFUL, OLS-Bandit, and Lasso-Bandit as soon as there are more than $140$ users. This observation is consistent with that in warfarin dosing experiment in \S\ref{sec:empirical_warfarin} and suggests that compared to other benchmark algorithms, the G-MCP-Bandit algorithm can improve the parameter vector estimation under high-dimensional data with limited samples and achieve better revenue performance.

Further, we find the choice of underlying models can significantly influence the G-MCP-Bandit algorithm's average revenue performance. Note that the advertisers award the search engine only when users have clicked the recommended ads. Therefore, the search engine's reward function is binary in nature. When comparing the G-MCP-Bandit algorithm under the logistic model to that under the linear model, both of which are special cases of the G-MCP-Bandit algorithm, we observe that the former always dominates the latter for any number of users. In addition, the G-MCP-Bandit algorithm under the logistic model merely needs $20$ users to outperform the other three algorithms. This observation suggests that understanding the underlying managerial problem and identifying the appropriate model for the G-MCP-Bandit algorithm can be critical and bring substantial revenue improvement for the decision-maker.

\section{Conclusion}

In this research, we develop the G-MCP-Bandit algorithm for online learning and decision-making processes in high-dimensional settings under limited samples. We adopt the matrix perturbation technique to derive new oracle inequality for the MCP estimator under non-iid samples and further propose a linear approximation method, the 2sWL procedure, to overcome the computational and statistical challenges associated with solving the MCP estimator (an NP-complete problem) under the bandit setting. We demonstrate that the MCP estimator solved by the 2sWL procedure matches the oracle estimator with high probability and converges to the true parameters with the optimal convergence rate. Further, we show that the cumulative regret of the G-MCP-Bandit algorithm over the sample size $T$ is bounded by $O(\log T)$, which is the lowest theoretical bound for all possible algorithms under both low-dimensional and high-dimensional settings. In the covariate dimension $d$, the cumulative regret of the G-MCP-Bandit algorithm is bounded by $O(\log d)$, which is also a tighter bound than existing bandit algorithms. Finally, we illustrate that compared to other benchmark algorithms, the G-MCP-Bandit algorithm performs favorably in synthetic-data-based and real-data-based experiments.

Implementing the G-MCP-Bandit algorithm under high-dimensional data with a large decision set in an online setting can be challenging in practice, and addressing these challenges can extend this research to several directions. One of the major challenges is the lengthy computation time, when the covariate dimension and the decision set are large. In particular, during a collaboration with a leading online marketplace, we adopted the G-MCP-Bandit algorithm, aiming to improve its product recommendation system. Using its offline datasets (with $5$ million covariates and $100$ to $500$ products),  we showed that the G-MCP-Bandit algorithm improved the prediction of the conversion rate by $15\%$ and the expected revenue by $5\%$ on average, but a single server could take hours to execute the algorithm. Hence, in order to practically implement the G-MCP-Bandit algorithm in \emph{online} settings, parallel computation techniques must be developed to tremendously reduce the computation time. Other challenges for the G-MCP-Bandit algorithm are how to simultaneously recommend multiple products and how to dynamically update the recommendation if the user did not click the recommended products but keeps refreshing the window. Tackling these challenges requires an integration of the assortment optimization and Bayesian learning into the G-MCP-Bandit algorithm.


\medskip \medskip \medskip
\linespread{1}
\bibliographystyle{informs2014}
\bibliography{fcp_mab}

\begin{thebibliography}{49}
\providecommand{\natexlab}[1]{#1}
\providecommand{\url}[1]{\texttt{#1}}
\providecommand{\urlprefix}{URL }

\bibitem[{Abbasi-Yadkori et~al.(2011)Abbasi-Yadkori, P{\'a}l, \protect\BIBand{}
  Szepesv{\'a}ri}]{abbasi2011improved}
Abbasi-Yadkori Y, P{\'a}l D, Szepesv{\'a}ri C (2011) Improved algorithms for
  linear stochastic bandits. \emph{Advances in Neural Information Processing
  Systems}, 2312--2320.

\bibitem[{Abbasi-Yadkori \protect\BIBand{} Szepesvari(2012)}]{abbasi2012online}
Abbasi-Yadkori Y, Szepesvari C (2012) \emph{Online learning for linearly
  parametrized control problems} (University of Alberta).

\bibitem[{Agrawal \protect\BIBand{} Goyal(2013)}]{agrawal2013thompson}
Agrawal S, Goyal N (2013) Thompson sampling for contextual bandits with linear
  payoffs. \emph{International Conference on Machine Learning}, 127--135.

\bibitem[{Auer(2002)}]{auer2002using}
Auer P (2002) Using confidence bounds for exploitation-exploration trade-offs.
  \emph{Journal of Machine Learning Research} 3(Nov):397--422.

\bibitem[{Bastani \protect\BIBand{} Bayati(2015)}]{bastani2015online}
Bastani H, Bayati M (2015) Online decision-making with high-dimensional
  covariates .

\bibitem[{B{\"u}hlmann \protect\BIBand{} Van
  De~Geer(2011)}]{buhlmann2011statistics}
B{\"u}hlmann P, Van De~Geer S (2011) \emph{Statistics for high-dimensional
  data: methods, theory and applications} (Springer Science \& Business Media).

\bibitem[{Candes et~al.(2008)Candes, Wakin, \protect\BIBand{}
  Boyd}]{candes2008enhancing}
Candes EJ, Wakin MB, Boyd SP (2008) Enhancing sparsity by reweighted ℓ1
  minimization. \emph{Journal of Fourier analysis and applications}
  14(5-6):877--905.

\bibitem[{Consortium et~al.(2009)}]{international2009estimation}
Consortium IWP, et~al. (2009) Estimation of the warfarin dose with clinical and
  pharmacogenetic data. \emph{N Engl J Med} 2009(360):753--764.

\bibitem[{Dani et~al.(2008)Dani, Hayes, \protect\BIBand{}
  Kakade}]{dani2008stochastic}
Dani V, Hayes TP, Kakade SM (2008) Stochastic linear optimization under bandit
  feedback .

\bibitem[{Deshpande \protect\BIBand{} Montanari(2012)}]{deshpande2012linear}
Deshpande Y, Montanari A (2012) Linear bandits in high dimension and
  recommendation systems. \emph{2012 50th Annual Allerton Conference on
  Communication, Control, and Computing (Allerton)}, 1750--1754 (IEEE).

\bibitem[{Elmachtoub et~al.(2017)Elmachtoub, McNellis, Oh, \protect\BIBand{}
  Petrik}]{elmachtoub2017practical}
Elmachtoub AN, McNellis R, Oh S, Petrik M (2017) A practical method for solving
  contextual bandit problems using decision trees. \emph{arXiv preprint
  arXiv:1706.04687} .

\bibitem[{Fan et~al.(2014{\natexlab{a}})Fan, Han, \protect\BIBand{}
  Liu}]{fan2014challenges}
Fan J, Han F, Liu H (2014{\natexlab{a}}) Challenges of big data analysis.
  \emph{National science review} 1(2):293--314.

\bibitem[{Fan \protect\BIBand{} Li(2001)}]{fan2001variable}
Fan J, Li R (2001) Variable selection via nonconcave penalized likelihood and
  its oracle properties. \emph{Journal of the American statistical Association}
  96(456):1348--1360.

\bibitem[{Fan et~al.(2018)Fan, Liu, Sun, \protect\BIBand{} Zhang}]{fan2018lamm}
Fan J, Liu H, Sun Q, Zhang T (2018) I-lamm for sparse learning: Simultaneous
  control of algorithmic complexity and statistical error. \emph{Annals of
  statistics} 46(2):814.

\bibitem[{Fan et~al.(2014{\natexlab{b}})Fan, Xue, \protect\BIBand{}
  Zou}]{fan2014strong}
Fan J, Xue L, Zou H (2014{\natexlab{b}}) Strong oracle optimality of folded
  concave penalized estimation. \emph{Annals of statistics} 42(3):819.

\bibitem[{Goldenshluger \protect\BIBand{}
  Zeevi(2013)}]{goldenshluger2013linear}
Goldenshluger A, Zeevi A (2013) A linear response bandit problem.
  \emph{Stochastic Systems} 3(1):230--261.

\bibitem[{Huang et~al.(2008)Huang, Ma, \protect\BIBand{}
  Zhang}]{huang2008adaptive}
Huang J, Ma S, Zhang CH (2008) Adaptive lasso for sparse high-dimensional
  regression models. \emph{Statistica Sinica} 1603--1618.

\bibitem[{Liu et~al.(2017)Liu, Yao, Li, \protect\BIBand{} Ye}]{liufolded}
Liu H, Yao T, Li R, Ye Y (2017) Folded concave penalized sparse linear
  regression: Sparsity, statistical performance, and algorithmic theory for
  local solutions. \emph{Mathematical Programming} 1–34,
  \urlprefix\url{http://dx.doi.org/10.1007/s10107-017-1114-y}.

\bibitem[{Liu et~al.(2016)Liu, Yao, Li et~al.}]{liu2016global}
Liu H, Yao T, Li R, et~al. (2016) Global solutions to folded concave penalized
  nonconvex learning. \emph{The Annals of Statistics} 44(2):629--659.

\bibitem[{Loh \protect\BIBand{} Wainwright(2013)}]{loh2013regularized}
Loh PL, Wainwright MJ (2013) Regularized m-estimators with nonconvexity:
  Statistical and algorithmic theory for local optima. \emph{Advances in Neural
  Information Processing Systems}, 476--484.

\bibitem[{McCullagh \protect\BIBand{} Nelder(1989)}]{mcCullagh1989generalized}
McCullagh P, Nelder J (1989) \emph{Generalized linear models} (Chapman and
  Hall/CRC).

\bibitem[{Meinshausen et~al.(2006)Meinshausen, B{\"u}hlmann
  et~al.}]{meinshausen2006high}
Meinshausen N, B{\"u}hlmann P, et~al. (2006) High-dimensional graphs and
  variable selection with the lasso. \emph{The annals of statistics}
  34(3):1436--1462.

\bibitem[{Meinshausen et~al.(2009)Meinshausen, Yu
  et~al.}]{meinshausen2009lasso}
Meinshausen N, Yu B, et~al. (2009) Lasso-type recovery of sparse
  representations for high-dimensional data. \emph{The Annals of Statistics}
  37(1):246--270.

\bibitem[{Mitchell(2012)}]{Ben2012}
Mitchell J (2012) How google search really works.
  \url{https://readwrite.com/2012/02/29/interview_changing_engines_mid-flight_qa_with_goog/#awesm=~oiNkM4tAX3xhbP},
  accessed: Oct 22nd, 2018.

\bibitem[{Montgomery et~al.(2012)Montgomery, Peck, \protect\BIBand{}
  Vining}]{montgomery2012introduction}
Montgomery DC, Peck EA, Vining GG (2012) \emph{Introduction to linear
  regression analysis}, volume 821 (John Wiley \& Sons).

\bibitem[{Negahban et~al.(2009)Negahban, Yu, Wainwright, \protect\BIBand{}
  Ravikumar}]{negahban2009unified}
Negahban S, Yu B, Wainwright MJ, Ravikumar PK (2009) A unified framework for
  high-dimensional analysis of $ m $-estimators with decomposable regularizers.
  \emph{Advances in Neural Information Processing Systems}, 1348--1356.

\bibitem[{OxfordDictionaries(2018)}]{DICT2018}
OxfordDictionaries (2018) How many words are there in the english language?
  \url{https://en.oxforddictionaries.com/explore/how-many-words-are-there-in-the-english-language/},
  accessed: Oct 22nd, 2018.

\bibitem[{Rigollet \protect\BIBand{} Zeevi(2010)}]{rigollet2010nonparametric}
Rigollet P, Zeevi A (2010) Nonparametric bandits with covariates. \emph{arXiv
  preprint arXiv:1003.1630} .

\bibitem[{Robbins(1952)}]{robbins1952some}
Robbins H (1952) Some aspects of the sequential design of experiments.
  \emph{Bulletin of the American Mathematical Society} 58(5):527--535.

\bibitem[{Rudelson et~al.(2013)Rudelson, Vershynin et~al.}]{rudelson2013hanson}
Rudelson M, Vershynin R, et~al. (2013) Hanson-wright inequality and
  sub-gaussian concentration. \emph{Electron. Commun. Probab} 18(82):1--9.

\bibitem[{Rusmevichientong \protect\BIBand{}
  Tsitsiklis(2010)}]{rusmevichientong2010linearly}
Rusmevichientong P, Tsitsiklis JN (2010) Linearly parameterized bandits.
  \emph{Mathematics of Operations Research} 35(2):395--411.

\bibitem[{Russo \protect\BIBand{} Van~Roy(2014)}]{russo2014learning}
Russo D, Van~Roy B (2014) Learning to optimize via posterior sampling.
  \emph{Mathematics of Operations Research} 39(4):1221--1243.

\bibitem[{Scott(2010)}]{scott2010modern}
Scott SL (2010) A modern bayesian look at the multi-armed bandit. \emph{Applied
  Stochastic Models in Business and Industry} 26(6):639--658.

\bibitem[{Scott(2015)}]{scott2015multi}
Scott SL (2015) Multi-armed bandit experiments in the online service economy.
  \emph{Applied Stochastic Models in Business and Industry} 31(1):37--45.

\bibitem[{Shaheen(2018)}]{Shaheen2018}
Shaheen J (2018) How long should multi-channel advertising campaigns last?
  \url{https://technologytherapy.com/long-multi-channel-advertising-campaigns-last/},
  accessed: Oct 22nd, 2018.

\bibitem[{Shewan(2017)}]{WordStream2018}
Shewan D (2017) The comprehensive guide to online advertising costs.
  \url{https://www.wordstream.com/blog/ws/2017/07/05/online-advertising-costs},
  accessed: Oct 22nd, 2018.

\bibitem[{Slivkins(2014)}]{slivkins2014contextual}
Slivkins A (2014) Contextual bandits with similarity information. \emph{The
  Journal of Machine Learning Research} 15(1):2533--2568.

\bibitem[{Tencent(2012)}]{Tencent2012}
Tencent (2012) Predict the click-through rate of ads given the query and user
  information. \url{https://www.kaggle.com/c/kddcup2012-track2}, accessed: Oct
  22nd, 2018.

\bibitem[{Tibshirani(1996)}]{tibshirani1996regression}
Tibshirani R (1996) Regression shrinkage and selection via the lasso.
  \emph{Journal of the Royal Statistical Society. Series B (Methodological)}
  267--288.

\bibitem[{Tropp et~al.(2015)}]{tropp2015introduction}
Tropp JA, et~al. (2015) An introduction to matrix concentration inequalities.
  \emph{Foundations and Trends{\textregistered} in Machine Learning}
  8(1-2):1--230.

\bibitem[{Tsybakov et~al.(2004)}]{tsybakov2004optimal}
Tsybakov AB, et~al. (2004) Optimal aggregation of classifiers in statistical
  learning. \emph{The Annals of Statistics} 32(1):135--166.

\bibitem[{Van~de Geer et~al.(2008)}]{van2008high}
Van~de Geer SA, et~al. (2008) High-dimensional generalized linear models and
  the lasso. \emph{The Annals of Statistics} 36(2):614--645.

\bibitem[{WordStream(2017)}]{WordStream2017}
WordStream (2017) Average ctr (click-through rate): Learn how your ctr
  compares. \url{https://www.wordstream.com/average-ctr}, accessed: Oct 22nd,
  2018.

\bibitem[{Zhang et~al.(2008)Zhang, Huang et~al.}]{zhang2008sparsity}
Zhang CH, Huang J, et~al. (2008) The sparsity and bias of the lasso selection
  in high-dimensional linear regression. \emph{The Annals of Statistics}
  36(4):1567--1594.

\bibitem[{Zhang et~al.(2012)Zhang, Zhang et~al.}]{zhang2012general}
Zhang CH, Zhang T, et~al. (2012) A general theory of concave regularization for
  high-dimensional sparse estimation problems. \emph{Statistical Science}
  27(4):576--593.

\bibitem[{Zhang et~al.(2010)}]{zhang2010nearly}
Zhang CH, et~al. (2010) Nearly unbiased variable selection under minimax
  concave penalty. \emph{The Annals of statistics} 38(2):894--942.

\bibitem[{Zhao et~al.(2014)Zhao, Liu, \protect\BIBand{}
  Zhang}]{zhao2014pathwise}
Zhao T, Liu H, Zhang T (2014) Pathwise coordinate optimization for sparse
  learning: Algorithm and theory. \emph{arXiv preprint arXiv:1412.7477} .

\bibitem[{Zhao et~al.(2018)Zhao, Liu, Zhang et~al.}]{zhao2018pathwise}
Zhao T, Liu H, Zhang T, et~al. (2018) Pathwise coordinate optimization for
  sparse learning: Algorithm and theory. \emph{The Annals of Statistics}
  46(1):180--218.

\bibitem[{Zou(2006)}]{zou2006adaptive}
Zou H (2006) The adaptive lasso and its oracle properties. \emph{Journal of the
  American statistical association} 101(476):1418--1429.

\end{thebibliography}

\medskip

\linespread{1}
\ECSwitch
\small


\ECHead{\large  Electronic Companion to ``Online Learning and Decision-Making under Generalized Linear Model with High-Dimensional Data"}
\medskip
To simplify the notation in the E-companion, we denote $\nabla_{\mathcal{A}}F(\bm{x})$ as the vector with $(\nabla_{\mathcal{A}}F(\bm{x}))_i = (\nabla F(\bm{x}))_i,\ i\in\mathcal{A}$, where $(\cdot)_i$ is the $i$-th element in the vector. Similarly we denote $\nabla^2_{\mathcal{A},\mathcal{B}}F(\bm{x})$ as the matrix with $(\nabla^2_{\mathcal{A},\mathcal{B}}F(\bm{x}))_{ij} = (\nabla^2 F(\bm{x}))_{ij},\ i\in\mathcal{A}, j\in\mathcal{B}$, where $(\cdot)_{ij}$ is the element in $i$-th column and $j$-th row.

\begin{proof}
	{{\bf Proof of  Lemma \ref{MW_lemma_1}}} Lemma \ref{MW_lemma_1} directly follows Lemma \ref{lemma:2.11} in ``Appendix: Supplemental Lemmas and Proofs'' at the end of this Electronic Companion by setting $|\mathcal{A}| = n$.
\end{proof}

\begin{proof}
	{\bf Proof of Proposition \ref{MW_proposition:1}} Proposition \ref{MW_proposition:1} follows Lemma \ref{corollary:2.14} by setting $|\mathcal{A}| = n$.
\end{proof}

\begin{proof}
	{\bf Proof of Proposition \ref{MW_Sampling}}

	Under the $\epsilon$-decay random sampling method, the probability of randomly drawing arm $k$ at time $t$ is $\min\{1,t_0/t\}/|\mathcal{K}|$, where $|\mathcal{K}|$ is the number of arms. Hence, at time $T$, the expected total number of times at which arm $k$ were randomly drawn is 
	\begin{align}
	\mathbbm{E}[n_k] = \frac{1}{|\mathcal{K}|}\sum_{t=1}^T\min\left\{1,\frac{t_0}{t}\right\}.
	\end{align}
	When $T> t_0$,
	\begin{align}\label{eq:lemma_2.1_1}
	\mathbb{E}[n_k] = \frac{1}{|\mathcal{K}|}\left(t_0+ \sum_{t=t_0+1}^T\frac{t_0}{t}\right) = \frac{t_0}{|\mathcal{K}|}\left(1+\sum_{t = t_0+1}^T\frac{1}{t}\right)
	\end{align}
	Since the function $f(t) = 1/t$ is decreasing in $t$ , it can be upper and lower bounded:
	\begin{align}
	\int_{t}^{t+1}\frac{1}{t}dt<\frac{1}{t}&<\int_{t-1}^t\frac{1}{t}dt,\ t\ge2.
	\end{align}
	As $t_0\ge 1$, the following inequality hold for $t$ from $t_0+1$ to $T$:
	\begin{align}
\log(T+1)-\log(t_0+1)&< \sum_{t = t_0+1}^T\frac{1}{t}< \log(T)-\log(t_0)\label{eq:lemma_2.1_2}
	\end{align}
	Combining \eqref{eq:lemma_2.1_1} and \eqref{eq:lemma_2.1_2},  we can bound  $\mathbbm{E}[n_k]$ as follow.
	\begin{align}
	\frac{1}{|\mathcal{K}|}t_0(1+\log(T+1)-\log(t_0+1))< \mathbbm{E}[n_k]< \frac{1}{|\mathcal{K}|}t_0(1+\log(T)-\log(t_0)).\label{eq:lemma_2.1_3.5}
	\end{align}
	Since $n_k=\sum_{t=1}^{T}\mathbbm{1}\{\textrm{random sampling for arm } k \textrm{ at } t\}$, we can view $n_k$ as the summarization of bounded iid random variables. Via Chernoff bound, we can build the connect between $n_k$ and $\mathbb{E}[n_k]$:
	\begin{align}
	&\mathbb{P}\left(\frac{1}{2}\mathbb{E}[n_k]\le n_k\le\frac{3}{2}\mathbb{E}[n_k]\right)> 1- 2\exp\left(-\frac{1}{10}\mathbbm{E}[n_k]\right)\label{eq:lemma_2.1_5}.
	\end{align}
	We then relax the $\mathbb{E}[n_k]$ in \eqref{eq:lemma_2.1_5} with the upper and lower bounds provided in \eqref{eq:lemma_2.1_3.5} and the following result is attained.
	\begin{align}
	&\mathbbm{P}\left(\frac{t_0(1+\log(T+1)-\log(t_0+1))}{2|\mathcal{K}|}\le n_k\le \frac{3t_0(1+\log(T)-\log(t_0))}{2|\mathcal{K}|}\right)
	&\ge 1-2\left(\frac{t_0+1}{e(T+1)}\right)^{\frac{t_0}{10|\mathcal{K}|}}.\label{eq:lemma_2.1_5.1}
	\end{align}
	When $t_0=2C_0 |\mathcal{K}|, \ C_0\ge 10$, and $T>\frac{(t_0+1)^2}{e^2}$, we can simplify the right-hand size of  \eqref{eq:lemma_2.1_5.1}. \begin{align}
	 1-2\left(\frac{t_0+1}{e(T+1)}\right)^{\frac{t_0}{10|\mathcal{K}|}}\ge 1-2\left(\frac{e\sqrt{T+1}}{e(T+1)}\right)^{C_0/5}\ge 1-\frac{2}{T+1}.\label{eq:lemma_2.1_5.2}
	\end{align}
\end{proof}
\begin{proof}
	{\bf Proof of Proposition \ref{MW_Non_iid_MCP}} 
	In the first step of 2sWL procedure, we are essentially solving the Lasso problem. From Lemma \ref{lemma:2.10}, we have $\|\bm{\beta}^{lasso}-\bm{\beta}^{true}\|_1\le  \frac{96ns\lambda}{|\mathcal{A}|\kappa}$ which high probability. 
    As we assume $\beta_{\min}\ge \left(\frac{96ns}{|\mathcal{A}|\kappa}+a\right)\lambda$ and $\|\bm{\beta}^{lasso}-\bm{\beta}^{true}\|_{\infty}\le\|\bm{\beta}^{lasso}-\bm{\beta}^{true}\|_1$ we have the follow statements hold. 
	\begin{align}
	|\beta_{i}^{lasso}|\ge a\lambda,\ i\in\mathcal{S}\textrm{ and }|\beta_{i}^{lasso}|\le \frac{96ns\lambda}{|\mathcal{A}|\kappa},\ i \in\mathcal{S}^c,\label{eq:lemma:2.12_3}
	\end{align}
	where we ignore the subscript in $\mathcal{S}_k$ to simplify the notation.
	Combining \eqref{eq:lemma:2.12_3} and $P_{\lambda}^{'}(|x|) = \max\{0,\lambda-|x|/a\}$, we have the following two results.
	\begin{align}
	P_{\lambda}^{'}(|\beta_{i}^{lasso}|) &= 0\quad i \in\mathcal{S}\label{eq:lemma:2.12_4}\\
	P_{\lambda}^{'}(|\beta_{i}^{lasso}|) &\ge P_{\lambda}^{'}\left(\frac{96ns\lambda}{|\mathcal{A}|\kappa}\right) = \left(\lambda- \frac{96ns\lambda}{|\mathcal{A}|\kappa a}\right)\quad i \in\mathcal{S}^c\label{eq:lemma:2.12_5}.
	\end{align}
	Define the event $\mathcal{E}_2$ as follow.
	\begin{align}
	\mathcal{E}_2 = \left\{\|\nabla_{\mathcal{S}^c}\mathcal{L}(\bm{\beta}^{oracle})\|_{\infty}< \lambda - \frac{96ns\lambda}{|\mathcal{A}|\kappa a}\right\}     .\label{eq:event_2}
	\end{align}
	From the convexity of $\mathcal{L}(\bm{\beta})$, we can build a lower bound on the optimal objective function value in the second step of 2sWL.
	\begin{align}
	\mathcal{L}(\bm{\beta}^*)+\sum_j P_{\lambda}^{\prime}(|\beta_j^{lasso}|)\cdot|\beta_j^*|     \ge \mathcal{L}(\bm{\beta}^{oracle})+\nabla\mathcal{L}(\bm{\beta}^{oracle})^T(\bm{\beta}^*-\bm{\beta}^{true})+\sum_j P_{\lambda}^{\prime}(|\beta_j^{lasso}|)\cdot|\beta_j^*|,\label{eq:lemma:2.12_6}
	\end{align}
	where $\bm{\beta}^*$ is the optimal solution of the second step of the 2sWL procedures. 
	From the definition of oracle solution, we have
	\begin{align}
	\bm{\beta}^{oracle} =\arg\min_{\bm{\beta}_{\mathcal{S}^c}=0}\mathcal{L}(\bm{\beta})\Rightarrow 1)\ \nabla_{\mathcal{S}}\mathcal{L}(\bm{\beta}^{oracle})=0\textrm{ and } 2)\ \bm{\beta}_{\mathcal{S}^c}=0.\label{eq:lemma:2.12_7}
	\end{align}
	Combining \eqref{eq:lemma:2.12_4}, \eqref{eq:lemma:2.12_5}, \eqref{eq:lemma:2.12_6}, and \eqref{eq:lemma:2.12_7}, we have
	\begin{align}
	\mathcal{L}(\bm{\beta}^*)+\sum_{j\in\mathcal{S}^c} P_{\lambda}^{\prime}(|\beta_j^{lasso}|)\cdot|\beta_j^*|   &\ge \mathcal{L}(\bm{\beta}^{oracle})+\nabla_{\mathcal{S}^c}\mathcal{L}(\bm{\beta}^{oracle})^T(\bm{\beta}^*_{\mathcal{S}^c}-\bm{\beta}^{oracle}_{\mathcal{S}^c})+\sum_{j\in\mathcal{S}^c} P_{\lambda}^{\prime}(|\beta_j^{lasso}|)\cdot|\beta_j^*|\notag\\
	&=\mathcal{L}(\bm{\beta}^{oracle})+ \sum_{j\in\mathcal{S}^c}\left(\nabla_{j} \mathcal{L}(\bm{\beta}^{oracle})(\beta_{j}^*-0)+P_{\lambda}^{\prime}(|\beta_j^{lasso}|)\cdot|\beta_j^*|\right)\notag\\
	&=\mathcal{L}(\bm{\beta}^{oracle})+\sum_{j\in\mathcal{S}^c}P_{\lambda}^{\prime}(|\beta_j^{lasso}|)\cdot|\beta_j^{oracle}|\notag\\
	&\quad\quad+\sum_{j\in\mathcal{S}^c}\left(\nabla_{j} \mathcal{L}(\bm{\beta}^{oracle})\textrm{sign}(\beta_{j}^*)+P_{\lambda}^{\prime}(|\beta_j^{lasso}|)\right)|\beta_j^*|.\label{eq:lemma:2.12_8}
	\end{align}
	Using $\mathcal{E}_2$ defined in \eqref{eq:event_2}, \eqref{eq:lemma:2.12_8} can be simplified as follows:
	\begin{align}
	\mathcal{L}(\bm{\beta}^*)+\sum_{j\in\mathcal{S}^c} P_{\lambda}^{\prime}(|\beta_j^{lasso}|)\cdot|\beta_j^*|   &\ge\mathcal{L}(\bm{\beta}^{oracle})+\sum_{j\in\mathcal{S}^c}P_{\lambda}^{\prime}(|\beta_j^{lasso}|)\cdot|\beta_j^{oracle}|+c_0\sum_{j\in\mathcal{S}^c}|\beta_j^*|,\label{eq:lemma:2.12_9}
	\end{align}
	where $c_0$ is a positive constant. Since $\bm{\beta}^*$ is the optimal solution of the second step in 2sWL, per \eqref{eq:lemma:2.12_9} we must have $\beta_j^*=0$ for all $j\in\mathcal{S}^c$. Together with the uniqueness of the solution of \eqref{oracle estimator}, $\bm{\beta}^{oracle}$ is also the unique optimal solution to the second step in 2sWL, i.e, $\bm{\beta}^{MCP} = \bm{\beta}^{oracle}$. Therefore once event $\mathcal{E}_2$  happens, with high probability $\bm{\beta}^{MCP}$ becomes the oracle solution, which enjoy the optimal statistical performance. We then need to consider the chance that $\mathcal{E}_2$ happens and the result is summarized in Lemma \ref{lemma:2.17}. Per Lemma \ref{lemma:2.17}, the following $\mathcal{E}_3,\mathcal{E}_4$ and $\mathcal{E}_5$ implies $\mathcal{E}_2$. \begin{align}
	\mathcal{E}_3&=\left\{\|\nabla_{\mathcal{S}^c}\mathcal{L}(\bm{\beta}^{true})\|_{\infty}\le \left(1-\frac{96ns}{|\mathcal{A}|\kappa a}\right)\frac{\lambda}{4}\right\},\notag\\
	\mathcal{E}_4&=\left\{\|\nabla_{\mathcal{S}}\mathcal{L}(\bm{\beta}^{true})\|_{\infty}\le \left(1-\frac{96ns}{|\mathcal{A}|\kappa a}\right)\frac{\mu_0|\mathcal{A}|\lambda}{8snx_{\max}^2}\right\},\notag\\
	\mathcal{E}_5&=\left\{\|\bm{\beta}^{oracle}-\bm{\beta}^{true}\|_2\le\sqrt{C_2\lambda} \right\},\notag
	\end{align}
	where $C_2$ is a positive constant. Now, we can bound the probability of events $\mathcal{E}_3$, $\mathcal{E}_4$, and $\mathcal{E}_5$ happen simultaneously. From Assumption {\bf A.5} and Hoeffding bound 
	 we have the following inequality for $t_1>0$:
	\begin{align}
	&\mathbbm{P}\left(\|\nabla_{\mathcal{S}} \mathcal{L}(\bm{\beta}^{true})\|_{\infty}\ge t_1\right)=\mathbbm{P}\left(\frac{1}{n}\sum_{j=1}^nx_{j{\mathcal{S}}}^Tf^{'}(r_j|\bm{x}_{j,\mathcal{S}}^T\bm{\beta}^{true})\|_{\infty}\ge t_1\right)
	\le s\exp\left(-\frac{nt_1^2}{2\sigma^2x_{\max}^2}\right).\label{eq:lemma:2.12_20}
	\end{align}
	Similarly for $t_2>0$, we have the following result:
	\begin{align}
	&\mathbbm{P}\left(\|\nabla_{\mathcal{S}^c} \mathcal{L}(\bm{\beta}^{true})\|_{\infty}\ge t_2\right)\le (d-s)\exp\left(-\frac{nt_2^2}{2\sigma^2x_{\max}^2}\right)\label{eq:lemma:2.12_21}.
	\end{align}
	By setting $t_1 =t_2= (\frac{1}{4}-\frac{24ns}{|\mathcal{A}|\kappa a})\min\left\{1,\frac{\mu_0|\mathcal{A}|}{8snx_{\max}^2}\right\}\lambda$, we have
	\begin{align}
	\mathbbm{P}\left((\mathcal{E}_4^{\prime})^c\cup(\mathcal{E}_5^{\prime})^c\right)&\le d\exp\left(-\frac{n\lambda^2\left((\frac{1}{4}-\frac{24ns}{|\mathcal{A}|\kappa a})\min\left\{1,\frac{\mu_0|\mathcal{A}|}{8snx_{\max}^2}\right\}\right)^2}{2x_{\max}^2}\right)\label{eq:lemma:2.12_22}.
	\end{align}

	We can further bound event $\mathcal{E}_5$ via Lemma \ref{lemma:2.11}. 
	We can have the following result by setting $t$ in Lemma  \ref{lemma:2.11} satisfying $t \le \frac{\mu_0|\mathcal{A}|\sqrt{C_2\lambda}}{2n}$.
	\begin{align}
	\mathbbm{P}\left(\|\bm{\beta}^{oracle}-\bm{\beta}^{true}\|_2\le  \sqrt{C_2\lambda}\right)\le  2s\exp\left(-\frac{\mu_0|\mathcal{A}|}{4s\sigma_2x_{\max}^2}\right)+s\exp\left(-\frac{nt^2}{2\sigma^2x_{\max}^2}\right).\label{eq:lemma:2.12_23}
	\end{align}
	Moreover, if $\lambda\ge \frac{C_3n}{|\mathcal{A}|^2}$ and $C_3\doteq \frac{8s^2\sigma_2\sigma^2x_{\max}^2}{\mu_0^2C_2}$, we have $\sqrt{C_2\lambda}\ge \sqrt{\frac{8s^2\sigma_2\sigma^2x_{\max}^2n}{\mu_0^2|\mathcal{A}|^2}}$. From \eqref{eq:lemma:2.11_0.5} in Lemma \ref{lemma:2.11}, 
	 the following result hold for $|\mathcal{A}|\ge \frac{2s^2x_{\max}^2}{\mu_0}$:
	\begin{align}
	\mathbbm{P}\left(\|\bm{\beta}^{oracle}-\bm{\beta}^{true}\|_2\le \sqrt{C_2\lambda}\right)\ge 1-2s\exp\left(-\frac{\mu_0|\mathcal{A}|}{4s\sigma_2x_{\max}^2}\right)-2\exp\left(-\frac{C_h|\mathcal{A}|\mu_0}{2sx_{\max}^2}\right).\label{eq:lemma:2.12_24}
	\end{align}

	Combining Lemma \ref{lemma:2.10}, \eqref{eq:lemma:2.12_22} and \eqref{eq:lemma:2.12_23} 
	, we have the following inequality for $t\le \frac{\mu_0|\mathcal{A}|\sqrt{C_2\lambda}}{2n}$:
	\begin{align}
	\mathbbm{P}\left(\|\bm{\beta}^{MCP}-\bm{\beta}^{true}\|_2\le \frac{2nt}{|\mathcal{A}|\mu_0}\right)\ge 1-\delta_2(|\mathcal{A}|/n,\lambda)-\delta_3(|\mathcal{A}|)-\delta_4(|\mathcal{A}|/n,t).
	\end{align}
	Similarly, if we ensure $\lambda>\frac{C_3n}{|\mathcal{A}|^2}$ and $|\mathcal{A}|\ge \frac{2s^2x_{\max}^2}{\mu_0}$, then the following result comes directly from Lemma \ref{lemma:2.10}, \eqref{eq:lemma:2.12_22} and \eqref{eq:lemma:2.12_24}. 
	\begin{align}
	\mathbbm{P}\left(\|\bm{\beta}^{MCP}-\bm{\beta}^{true}\|_2\le \sqrt{\frac{8s^2\sigma_2\sigma^2x_{\max}^2n}{\mu_0^2|\mathcal{A}|^2}}\right)\ge1-\delta_1(|\mathcal{A}|)-\delta_2(|\mathcal{A}|/n,\lambda)-\delta_3(|\mathcal{A}|).
	\end{align}

\end{proof}

\begin{proof}
	{\bf Proof of Proposition \ref{MW_Estimator_Random}} Directly from Lemma \ref{corollary:2.14}.
\end{proof}

\begin{proof}
	{\bf Proof of Proposition \ref{MW_size_iid}}

	%
	%
	Since $\{M(i)\}$ is a martingale with bounded difference 1,     we can use $M(0)$ to bound the value of $M(T+1)$ with Azuma's inequality as follow:
	\begin{align}
	\mathbbm{P}\left(|M(T+1)-M(0)|\ge \frac{1}{2}M(0)\right)\le \exp\left(\frac{-M(0)^2/4}{2(T+2)}\right)\notag\\
	\Rightarrow \mathbbm{P}\left(M(T+1)\le \frac{1}{2}M(0)\right)\le \exp\left(\frac{-M(0)^2/4}{2(T+2)}\right).\notag
	\end{align}
	The term $M(0)$ can be expressed as follows:
	\begin{align}
	M(0) &=  \mathbbm{E}\left[\sum_{i=1}^{T+1}\mathbbm{1}(\bm{x}_i\in U_k,\mathcal{E}_6, \bm{x}\notin\mathcal{R}_k))\right]\notag\\
	&=\sum_{i=1}^{T+1}\mathbbm{P}(\bm{x}_i\in U_k,\mathcal{E}_6, x\notin\mathcal{R}_k).\label{eq:whole_sample_1}
	\end{align}
	As $\{\bm{x}\in U_k\}$ is independent of $\{\mathcal{E}_6, \bm{x}\notin\mathcal{R}_k\}$ and $\{\bm{x}\notin\mathcal{R}_k\}$ is independent on $\{\mathcal{E}_6\}$, \eqref{eq:whole_sample_1} implies the following inequality:
	\begin{align}
	M(0)&=\sum_{i=1}^{T+1}\mathbbm{P}(\bm{x}_i\in U_k)\mathbbm{P}(\mathcal{E}_6)\mathbbm{P}(\bm{x}\notin\mathcal{R}_k)\notag\\
	&\ge \sum_{i=1}^{T+1}p^*(1-\frac{15}{T+1})(1-\frac{2C_0|\mathcal{K}|}{T+1}),\label{eq:whole_sample_1.1}
	\end{align}
	where \eqref{eq:whole_sample_1.1} uses assumption {\bf A.3}, Lemma \ref{corollary:2.14}, and Proposition \ref{MW_Sampling}.

	When $T\ge \max\{30,4C_0|\mathcal{K}|\}$, we have
	\begin{align}
	\frac{15}{T+1}\le \frac{1}{2}\\
	\frac{2C_0|\mathcal{K}|}{T+1}\le \frac{1}{2},
	\end{align}
	which implies that
	\begin{align}
	M(0)\ge&\sum_{i=1}^{T+1}\frac{p^*}{4} = \frac{p^*(T+1)}{4}.
	\end{align}
	Therefore, the following inequalities hold:
	\begin{align}
	&\mathbbm{P}\left(M(T+1)\le \frac{p^*(T+1)}{8}
	\right)\le  \mathbbm{P}\left(M(T+1)\le M(0)
	\right)\le \exp\left(\frac{-(p^*)^2(T+1)^2/64}{2(T+2)}\right)\notag\\
	\Rightarrow&\mathbbm{P}\left(M(T+1)\le \frac{p^*(T+1)}{8}
	\right)\le \exp\left(-\frac{(p^*)^2((T+2)^2+1-2(T+2))}{128(T+2)}\right)\notag\\
	\Rightarrow&\mathbbm{P}\left(M(T+1)\le \frac{p^*(T+1)}{8}
	\right)\le \exp\left(-\frac{(p^*)^2T}{128}-\frac{p^*}{128(T+2)}\right)\notag\\
	\Rightarrow&\mathbbm{P}\left(M(T+1)\le \frac{p^*(T+1)}{8}
	\right)\le \exp\left(-\frac{(p^*)^2T}{128}\right)
	\end{align}

\end{proof}

\begin{proof}
	{\bf Proof of Proposition \ref{MW_all_sample_estimator}}
	According to Lemma \ref{lemma:2.16}, when event $\mathcal{E}_6$ defined by \eqref{eq:event_6} happens, the following inequality must hold for any $\bm{x}\in U_i$,
	\begin{align}
	\mathbbm{E}(R_i|\bm{x},\bm{\beta}^{random}_i(t))\ge \max_{j\ne i}\mathbbm{E}(R_j|\bm{x},\bm{\beta}^{random}_j(t))+\frac{h}{2}\notag.
	\end{align}
	Therefore, the lower-level decision-making process of the algorithm, in which the decision-maker will successfully select arm $i$ for $x$ by using the random sample estimator
	, will maintain the iid property of $\bm{x}$ 
	 since it can be viewed as rejection sampling. From Proposition \ref{MW_size_iid}
	 , we have
	\begin{align}
	\mathbbm{P}\left(M(T+1)\le \frac{p^*(T+1)}{8}
	\right)\le \exp\left(-\frac{(p^*)^2T}{128}\right).\label{eq:thm_18_1}
	\end{align}
	Since $M(T+1) =         \mathbbm{E}\left[\sum_{j=1}^{T+1}\mathbbm{1}(\bm{x}_j\in U_k,\mathcal{E}_6, \bm{x}_j\notin\mathcal{R}_k)|\mathcal{F}_{T+1}\right]=\sum_{j=1}^{T+1}\mathbbm{1}(\bm{x}_j\in U_k,\mathcal{E}_6, \bm{x}_j\notin\mathcal{R}_k)$, the amount of iid samples among the whole sample for arm $k$ up to time $T+1$ will be lower bounded by $M(T+1)$. Denote $\mathcal{A}$ and $n$ as the set of iid samples belonging to $U_i$ in the whole sample set and size of the whole sample respectively. The follow inequalities hold:
	\begin{align}
	\mathbbm{P}\left(|\mathcal{A}|\ge \frac{p^*(T+1)}{8}
	\right)\ge1- \exp\left(-\frac{(p^*)^2T}{128}\right),\ n\le T+1.\label{eq:thm_18_2}
	\end{align}
	%

	From Lemma \ref{lemma:2.13} with $|\mathcal{A}|\ge\frac{p^*(T+1)}{8}$, $n \le(T+1)$, $\lambda= C_4\sqrt{\frac{\log(T+1)+\log d}{T+1}}$, and $T\ge\max \left\{\frac{16s^2x_{\max}^2}{p^*\mu_0},\left(\frac{(768s+a)C_4\kappa p^*}{\kappa p^*\beta_{\min}}\right)^4(1+\log d)^2\right\}$, the following results can be obtained:
	\begin{align}
	\lambda \ge \frac{C_3n}{|\mathcal{A}|^2},\
	|\mathcal{A}|\ge \frac{2s^2x_{\max}^2}{\mu_0},\
	a >\frac{96ns}{\kappa|\mathcal{A}|} \textrm{ and }
	\beta_{\min}\ge(\frac{96ns}{\kappa|\mathcal{A}|}+a)\lambda.\notag
	\end{align}
We then have the following result.
	\begin{align}
	\mathbbm{P}\left(\|\bm{\beta}^{oracle}-\bm{\beta}^{true}\|\ge \sqrt{\frac{512s^2\sigma_2\sigma^2x_{\max}^2}{\mu_0^2(p^*)^2(T+1)}}\right)\le \delta_1\left(\frac{p^*(T+1)}{8}\right)+\delta_2\left(\frac{p^*}{4}\right)+\delta_3\left(\frac{p^*(T+1)}{8}\right)
	\label{eq:thm_18_f}\end{align}
	As we require $T\ge \max\left\{\frac{8}{C_1p^*}, \frac{2x_{\max}^2}{C_4^2}, \frac{2x_{\max}^2}{\left((\frac{1}{4}-\frac{192s}{p^*\kappa a})\min\left\{1,\frac{\mu_0p^*}{64sx_{\max}^2}\right\}\right)^2C_4^2},\frac{32\log(s)s\sigma_3x_{\max}^2}{\mu_0p^*}, \frac{16s\log(s)x_{\max}^2}{C_hp^*\mu_0}\right\}$ and $\lambda = C_4\sqrt{\frac{\log(T+1)+\log d}{T+1}}$, we can easily verify that
	\begin{align}
	\delta_1\left(\frac{p^*(T+1)}{8}\right)+\delta_2\left(\frac{p^*}{4}\right)+\delta_3\left(\frac{p^*(T+1)}{8}\right)\le \frac{12}{T+1}.\label{eq:thm_18_ff}
	\end{align}
Proposition \ref{MW_all_sample_estimator} 
directly follows combining \eqref{eq:thm_18_ff}, \eqref{eq:thm_18_f}, \eqref{eq:thm_18_2}, and $T\ge \frac{128}{p^*}$.


\end{proof}

\begin{proof} {\bf  Proof of Theorem \ref{GMCP_cum_regret}}
	We divide the time, up to time $T$, into three groups and derive the cumulative regret bound for each group separately. Consider the following three groups:
	\begin{enumerate}
		\item $x_i\in \mathcal{R}_k, k\in \mathcal{K}$ and $T\le T_0$.
		\item $x_i\notin \mathcal{R}_k, k\in \mathcal{K}$, $T>T_0$ and $\mathcal{E}_6$ doesn't hold,
		\item $x_i\notin \mathcal{R}_k, k\in \mathcal{K}$ $T>T_0$ and $\mathcal{E}_6$ holds.
	\end{enumerate}

	{\bf\textit {Regret in part 1:}} Denote the regret for the first part as $R_1(T)$.
	\begin{align}
	R_1(T) \le R_{\max}\left(\sum_{i=T_0}^T\mathbbm{1}(x_i\in \mathcal{R}_k, k\in \mathcal{K})+T_0\right)\le R_{\max}\left(\sum_{k\in\mathcal{K}}n_k+T_0\right).
	\end{align}
	From Proposition \ref{MW_Sampling}
	, we know that
	\begin{align}
	\mathbbm{P}\left(n_k\le \frac{3t_0(1+\log(T)-\log(t_0))}{2|\mathcal{K}|}\right)\ge 1-\frac{2}{T+1}.
	\end{align}
	If we require $t_0 = 2C_0|\mathcal{K}|$, $C_0\ge 10$, and $T\ge \max\{(t_0+1)^2/e^2-1,e\}$, then the above equation can be simplified to
	\begin{align}
	\mathbbm{P}\left(n_k\le 6C_0\log T\right)\ge 1-\frac{2}{T+1}&\Rightarrow\mathbbm{P}\left(n_k> 6C_0\log T\right)\le \frac{2}{T+1}
	\end{align}
	which implies
	\begin{align}
	\mathbbm{P}\left(\sum_{k\in\mathcal{K}}n_k>6C_0|\mathcal{K}|\log T\right)\le \mathbbm{P}\left(\cup_{k\in\mathcal{K}}(n_k>6C_0\log T)\right)\le \sum_{k\in\mathcal{K}}\mathbbm{P}\left(n_k>6C_0\log T\right)\le \frac{2|\mathcal{K}|}{T+1},
	\end{align}
	and
	\begin{align}
	R_1(T)\le R_{\max}\left(\sum_{k\in\mathcal{K}}n_k+T_0\right)&=R_{\max}\left(\sum_{k\in\mathcal{K}}n_k|\sum_{k\in\mathcal{K}}n_k>6C_0|\mathcal{K}|\log T\right)\mathbbm{P}\left(\sum_{k\in\mathcal{K}}n_k>6C_0|\mathcal{K}|\log T\right)\notag\\
	&+R_{\max}\left(\sum_{k\in\mathcal{K}}n_k|\sum_{k\in\mathcal{K}}n_k\le 6C_0|\mathcal{K}|\log T\right)\mathbbm{P}\left(\sum_{k\in\mathcal{K}}n_k\le 6C_0|\mathcal{K}|\log T\right)\notag\\
	&+R_{\max}T_0\notag\\
	&\le R_{\max}T\frac{2|\mathcal{K}|}{T+1}+R_{\max}6C_0|\mathcal{K}|\log T\left(1-\frac{2|\mathcal{K}|}{T+1}\right)+R_{\max}T_0\notag\\
	&\le 2R_{\max}|\mathcal{K}|+6R_{\max}C_0|\mathcal{K}|\log T+R_{\max}T_0\notag\\
	&\le R_{\max}|\mathcal{K}|(2+6C_0\log T)+R_{\max}T_0.
	\end{align}

	{\bf \textit {Regret in part 2:}} Denote the regret for the second part as $R_2(T)$.From Lemma \ref{corollary:2.14}, we know that
	\begin{align}
	&\mathbbm{P}\left(\|\bm{\beta}^{random}(t)-\bm{\beta}^{true}\|_1\le  \min\left\{\frac{1}{\sigma_2x_{\max}},\frac{h}{4e\sigma_2R_{\max}x_{\max}}\right\}\right)\ge 1-\frac{15}{T+1},\ k\in\mathcal{K}\notag\\
	\Rightarrow&\mathbbm{P}(\mathcal{E}_6(T))\ge 1-\frac{15|\mathcal{K}|}{T+1}.
	\end{align}

    Therefore, $R_2(T)$ can be bounded as follows:
	\begin{align}
	R_2(T) &\le \mathbbm{E}[\sum_{i=1}^T\mathbbm{1}(\mathcal{E}_6(i)^c)R_{\max}]\notag\\
	&=\sum_{i=1}^T\mathbbm{E}[\mathbbm{1}(\mathcal{E}_6(i)^c)]R_{\max}\notag\\
	&=\sum_{i=1}^T\mathbbm{P}(\mathcal{E}_6(i)^c)R_{\max}\notag\\
	&\le 15R_{\max}|\mathcal{K}|\log(T+1).
	\end{align}

	{\bf \textit {Regret in part 3:}} Denote the regret for the third part as $R_3(T)$. Without loss of generality, we assume that arm $i$ is true optimal arm at time $t$. Then, the regret at time $t$ can be bounded as follows:
	\begin{align}
	r_t&=\mathbbm{E}\left(\mathbbm{1}\left(j=\arg \max_{k\in\mathcal{K}} \mathbbm{E}[R_k|\bm{x}_{t},\bm{\beta}^{whole}_k(t)]\right)(\mathbbm{E}[R_i|\bm{x}_t,\bm{\beta}^{true}_i]-\mathbbm{E}[R_j|\bm{x}_t,\bm{\beta}^{true}_j])\right)\notag\\
	&\le\mathbbm{E}\left(\sum_{j\ne i}\mathbbm{1}\left(\mathbbm{E}[R_j|\bm{x}_{t},\bm{\beta}^{whole}_j(t)]> \mathbbm{E}[R_i|\bm{x}_{t},\bm{\beta}^{whole}_i(t)]\right)(\mathbbm{E}[R_i|\bm{x}_t,\bm{\beta}^{true}_i]-\mathbbm{E}[R_j|\bm{x}_t,\bm{\beta}^{true}_j])\right).
	\end{align}
	Denote $\mathcal{E}(t,\delta)_{8,k}  = \{\mathbbm{E}[R_i|\bm{x}_{t},\bm{\beta}^{true}_i]> \mathbbm{E}[R_k|\bm{x}_{t},\bm{\beta}^{true}_k]+\delta\},\ k\ne i, k\in\mathcal{K}$. Then we have the following bound:
	\begin{align}
	&r_t\notag\\
	\le&\mathbbm{E}\left(\sum_{j\ne i}\mathbbm{1}\left(\left\{\mathbbm{E}[R_j|\bm{x}_{t},\bm{\beta}^{whole}_j(t)]> \mathbbm{E}[R_i|\bm{x}_{t},\bm{\beta}^{whole}_i(t)]\right\}\cap \mathcal{E}(t,\delta)_{8,j}\right)(\mathbbm{E}[R_i|\bm{x}_t,\bm{\beta}^{true}_i]-\mathbbm{E}[R_j|\bm{x}_t,\bm{\beta}^{true}_j])\right)\notag\\
	+&\mathbbm{E}\left(\sum_{j\ne i}\mathbbm{1}\left(\left\{\mathbbm{E}[R_j|\bm{x}_{t},\bm{\beta}^{whole}_j(t)]> \mathbbm{E}[R_i|\bm{x}_{t},\bm{\beta}^{whole}_i(t)]\right\}\cap \mathcal{E}(t,\delta)_{8,j}^c\right)(\mathbbm{E}[R_i|\bm{x}_t,\bm{\beta}^{true}_i]-\mathbbm{E}[R_j|\bm{x}_t,\bm{\beta}^{true}_j])\right)\notag\\
	\le&\mathbbm{E}\left(\sum_{j\ne i}\mathbbm{1}\left(\left\{\mathbbm{E}[R_j|\bm{x}_{t},\bm{\beta}^{whole}_j(t)]> \mathbbm{E}[R_i|\bm{x}_{t},\bm{\beta}^{whole}_i(t)]\right\}\cap \mathcal{E}(t,\delta)_{8,j}\right)(2R_{\max})\right)\label{eq:first_part}\\
	+&\mathbbm{E}\left(\sum_{j\ne i}\mathbbm{1}\left(\left\{\mathbbm{E}[R_j|\bm{x}_{t},\bm{\beta}^{whole}_j(t)]> \mathbbm{E}[R_i|\bm{x}_{t},\bm{\beta}^{whole}_i(t)]\right\}\cap \mathcal{E}(t,\delta)_{8,j}^c\right)(\delta)\right)\label{eq:second_part}
	\end{align}
	The term in \eqref{eq:second_part} can be bounded as follows:
	\begin{align}
	&\mathbbm{E}\left(\sum_{j\ne i}\mathbbm{1}\left(\left\{\mathbbm{E}[R_j|\bm{x}_{t},\bm{\beta}^{whole}_j(t)]> \mathbbm{E}[R_i|\bm{x}_{t},\bm{\beta}^{whole}_i(t)]\right\}\cap \mathcal{E}(t,\delta)_{8,j}^c\right)(\delta)\right)\notag\\
	\le& \mathbbm{E}\left(\sum_{j\ne i}\mathbbm{1}\left( \mathcal{E}(t,\delta)_{8,j}^c\right)(\delta)\right)\notag\\
	=&\sum_{j\ne i}\mathbbm{P}\left( \mathcal{E}(t,\delta)_{8,j}^c\right)\delta\notag\\
	=&(|\mathcal{K}|-1)CR_{\max}\delta^2\le CR_{\max}|\mathcal{K}|\delta^2,
	\end{align}
	where the last inequality comes from assumption {\bf A.2}. 	Now we consider the term in \eqref{eq:first_part}, which can be bounded as follows:
	\begin{align}
	&\mathbbm{E}\left(\sum_{j\ne i}\mathbbm{1}\left(\left\{\mathbbm{E}[R_j|\bm{x}_{t},\bm{\beta}^{whole}_j(t)]> \mathbbm{E}[R_i|\bm{x}_{t},\bm{\beta}^{whole}_i(t)]\right\}\cap \mathcal{E}(t,\delta)_{8,j}\right)(2R_{\max})\right)\notag\\
	\le&\mathbbm{E}\left(\sum_{j\ne i}\mathbbm{1}\left(\mathbbm{E}[R_j|\bm{x}_{t},\bm{\beta}^{whole}_j(t)]-\mathbbm{E}[R_j|\bm{x}_{t},\bm{\beta}^{true}_j]> \mathbbm{E}[R_i|\bm{x}_{t},\bm{\beta}^{whole}_i(t)]-\mathbbm{E}[R_i|\bm{x}_{t},\bm{\beta}^{true}_i]+\delta\right)(2R_{\max})\right)\notag\\
	\le&\mathbbm{E}\left(\sum_{j\ne i}\mathbbm{1}\left(\left|\mathbbm{E}[R_j|\bm{x}_{t},\bm{\beta}^{whole}_j(t)]-\mathbbm{E}[R_j|\bm{x}_{t},\bm{\beta}^{true}_j]\right|> -\left|\mathbbm{E}[R_i|\bm{x}_{t},\bm{\beta}^{whole}_i(t)]-\mathbbm{E}[R_i|\bm{x}_{t},\bm{\beta}^{true}_i]\right|+\delta\right)(2R_{\max})\right)\notag\\
	\le&\mathbbm{E}\left(\sum_{j\ne i}\mathbbm{1}\left(R_{\max}\sigma_2e^{2\sigma_2x_{\max}b}x_{\max}\|\bm{\beta}^{true}_k-\bm{\beta}^{whole}_k(t)\|_1> -R_{\max}\sigma_2e^{2\sigma_2x_{\max}b}x_{\max}\|\bm{\beta}^{true}_i-\bm{\beta}^{whole}_i(t)\|_1+\delta\right)(2R_{\max})\right)\notag\\
	\le&\mathbbm{E}\left(\sum_{j\ne i}\mathbbm{1}\left(\|\bm{\beta}^{true}_k-\bm{\beta}^{whole}_k(t)\|_1+\|\bm{\beta}^{true}_i-\bm{\beta}^{whole}_i(t)\|_1\ge\frac{\delta}{R_{\max}\sigma_2e^{2\sigma_2x_{\max}b}x_{\max}}\right)(2R_{\max})\right)\label{eq:final_0}
	,
	\end{align}
	where the second last inequality comes from the first part of the Lemma \ref{lemma:2.16} and $\|\bm{\beta}\|_1\le b$ in assumption {\bf A.1}. 
	Denote event $\mathcal{E}_9$ as follow:
	\begin{align}
	\mathcal{E}_9 = \{\|\bm{\beta}^{whole}_k(t)-\bm{\beta}^{true}_k\|_1\ge \frac{\delta}{2R_{\max}\sigma_1e^{2\sigma_2x_{\max}b}x_{\max}},\ k\in\mathcal{K}\}.\label{eq:event_9}
	\end{align}

	Combining \eqref{eq:final_0} and \eqref{eq:final_3}, we have:
	\begin{align}
	&\mathbbm{E}\left(\sum_{j\ne i}\mathbbm{1}\left(\|\bm{\beta}^{true}_j-\bm{\beta}^{whole}_j(t)\|_1\|\bm{\beta}^{true}_i-\bm{\beta}^{whole}_i(t)\|_1\ge\frac{\delta}{R_{\max}\sigma_2e^{2\sigma_2x_{\max}b}x_{\max}}\right)(2R_{\max})\right)\notag\\
	=&\mathbbm{E}\left.\left(\sum_{j\ne i}\mathbbm{1}\left(\|\bm{\beta}^{true}_j-\bm{\beta}^{whole}_j(t)\|_1+\|\bm{\beta}^{true}_i-\bm{\beta}^{whole}_i(t)\|_1\ge\frac{\delta}{R_{\max}\sigma_2e^{2\sigma_2x_{\max}b}x_{\max}}\right|\mathcal{E}_9\right)\mathbbm{1}(\mathcal{E}_9)(2R_{\max})\right)\notag\\
	+&\mathbbm{E}\left.\left(\sum_{j\ne i}\mathbbm{1}\left(\|\bm{\beta}^{true}_j-\bm{\beta}^{whole}_j(t)\|_1+\|\bm{\beta}^{true}_i-\bm{\beta}^{whole}_i(t)\|_1\ge\frac{\delta}{R_{\max}\sigma_2e^{2\sigma_2x_{\max}b}x_{\max}}\right|\mathcal{E}_9^c\right)\mathbbm{1}(\mathcal{E}_9^c)(2R_{\max})\right)\notag\\
	\le& \mathbbm{E}\left(\sum_{j\ne i}\mathbbm{1}\left(\left.\frac{1}{2}\delta+\frac{1}{2}\delta\ge\delta\right| \mathcal{E}_9\right)\mathbbm{1}\left(\mathcal{E}_9\right)(2R_{\max})\right)+0\notag\\
	=& \mathbbm{E}\left(\mathbbm{1}\left(\mathcal{E}_9(t)\right)(2R_{\max})\right)\le 2R_{\max}\mathbbm{P}(\mathcal{E}_9).\end{align}
	From Proposition \ref{MW_all_sample_estimator} 
	we have the following inequality:
	\begin{align}
	\mathbbm{P}\left(\|\bm{\beta}^{whole}_k(t)-\bm{\beta}^{true}_k\|_2\ge \sqrt{C_{\bm{\beta}}\frac{s^2}{T}}\right)\le \frac{13}{T+1}.\label{eq:final_1}
	\end{align}
	As $\|\bm{\beta}^{whole}_k(t)-\bm{\beta}^{true}_k\|_2\ge \frac{1}{\sqrt{s}}\|\bm{\beta}^{whole}_k(t)-\bm{\beta}^{true}_k\|_1$, thus \eqref{eq:final_1} implies
	\begin{align}
	\mathbbm{P}\left(\|\bm{\beta}^{whole}_k(t)-\bm{\beta}^{true}_k\|_1\ge \sqrt{C_{\bm{\beta}}\frac{s^3}{T+1}}\right)\le \frac{13}{T+1}\label{eq:final_3}.
	\end{align}
	Furthermore, by setting $\delta = 2R_{\max}\sigma_2e^{2\sigma_2x_{\max}b}x_{\max}\sqrt{C_{\bm{\beta}}\frac{s^3}{T+1}}$, we have the following result:
	\begin{align}
	r_t \le 2R_{\max}\mathbbm{P}(\mathcal{E}_9)+CR_{\max}|\mathcal{K}|\delta^2 \le \frac{26R_{\max}|\mathcal{K}|}{T+1}+CR_{\max}|\mathcal{K}|\frac{4R^2_{\max}\sigma_2^2e^{4\sigma_2x_{\max}b}x_{\max}^2C_{\bm{\beta}}s^3}{T+1} = \frac{C_{R_3}}{T+1}
	\end{align}
	where $C_{R_3} = 26R_{\max}|\mathcal{K}|+4e^{4\sigma_2x_{\max}b}CR_{\max}^3|\mathcal{K}|x_{\max}^2C_{\bm{\beta}}s^3$. Hence, the third part of the regret can be bounded as follows:
	\begin{align}
	R_3(T)= \sum_{i=1,i\in\mathcal{R}(T)}^Tr_t      \le \sum_{i=1}^T\frac{C_{R_3}}{T}\le \int_{1}^T\frac{C_{R_3}}{t}dt\le C_{R_3}\log(T)
	\end{align}

	Finally, the total regret bound can be obtained by combining the bounds for these three parts:
	\begin{align}
	R_{1}(T)+R_{2}(T)+R_{3}(T)&\le R_{\max}[|\mathcal{K}|(2+6C_0\log T)+T_0]+15R_{\max}|\mathcal{K}|\log(T+1)+C_{R_3}\log(T)\notag\\
	&\le R_{\max}(T_0+|\mathcal{K}|)+(6R_{\max}|\mathcal{K}|C_0+41R_{\max}|\mathcal{K}|+4e^{4\sigma_2x_{\max}b}CR_{\max}^3|\mathcal{K}|x_{\max}^2C_{\bm{\beta}}s^3)\log (T+1)\notag\\
	&=O(|\mathcal{K}|s^2(s+\log d)\log T).\notag
	\end{align}
\end{proof}

\newpage

\noindent {\large \textbf{Appendix: Supplemental Lemmas and Proofs}}

\begin{lemma}\label{lemma:2.15}
	Let $\mathcal{A}$ be the set of iid samples. Under assumption {A.1} and {A.5}, there exists a constant $\mu_0>0$ such that for feasible $\bm{\xi}$ defined in assumption A.4 we have
	\begin{align}
	\mathbbm{P}\left(\lambda_{\min}(\nabla^2_{\mathcal{S},\mathcal{S}} \mathcal{L}(\bm{\xi}))\ge  \frac{|\mathcal{A}|}{2n}\mu_0\right)
	&\le1-2s\exp\left(-\frac{|\mathcal{A}|\mu_0}{4s\sigma_2x_{\max}^2}\right).\label{eq:lemma:2.11_7}
	\end{align}
\end{lemma}

\begin{proof} {Proof of Lemma \ref{lemma:2.15}}
	Denote $\bm{z}_j^{\prime} = \bm{x}_{j,\mathcal{S}}\sqrt{f^{''}(r_j,|\bm{x}_{j,\mathcal{S}}^T\bm{\xi})}$. We can rewrite
	\begin{align}
	\nabla^2_{\mathcal{S},\mathcal{S}} \mathcal{L}(\bm{\xi}) &= \frac{1}{n}\sum_{i=1}^n\bm{x}_{i,\mathcal{S}}x_{i,\mathcal{S}}^Tf^{''}(r_i|\bm{x}_{i,\mathcal{S}}^T\bm{\xi})=\frac{1}{n}\sum_{i=1}^n\bm{z}_i^{\prime}(\bm{z}_i^{\prime})^T\label{eq:lemma:2.11_3}\\
	&=\frac{1}{n}\sum_{j\in\mathcal{A}^c}\bm{z}_j^{\prime}(\bm{z}_j^{\prime})^T+\frac{1}{n}\sum_{j\in\mathcal{A}}\bm{z}_j^{\prime}(\bm{z}_j^{\prime})^T  \\
	&\succeq \lambda_{\min}\left(\frac{1}{n}\sum_{j\in\mathcal{A}}\bm{z}_j^{\prime}(\bm{z}_j^{\prime})^T\right)+0
	\end{align}

	Then, we can bound $\lambda_{\min}\left(\frac{1}{n}\sum_{j\in\mathcal{A}^c}\bm{z}_j^{\prime}(\bm{z}_j^{\prime})^T\right)$ via Theorem 5.1.1 in \cite{tropp2015introduction} with $\epsilon=1/2$: 
	\begin{align}
	\mathbbm{P}\left(\lambda_{\min}(\frac{1}{n}\sum_{j\in\mathcal{A}}\bm{z}_j^{\prime}(\bm{z}_j^{\prime})^T)\le \frac{1}{2}\lambda_{\min}(\mathbbm{E}[\frac{1}{n}\sum_{j\in\mathcal{A}}\bm{z}_j^{\prime}(\bm{z}_j^{\prime})^T])\right)&\le s\left(\frac{\exp(-1/2)}{\sqrt{1/2}}\right)^{\lambda_{\min}(\mathbbm{E}[\frac{1}{n}\sum_{j\in\mathcal{A}}\bm{z}_j^{\prime}(\bm{z}_j^{\prime})^T])/(s\sigma_2x_{\max}^2/n)}\label{eq:lemma:2.11_5.0}\\
	\Rightarrow \mathbbm{P}\left(\lambda_{\min}(\frac{1}{n}\sum_{j\in\mathcal{A}}^n\bm{z}_j^{\prime}(\bm{z}_j^{\prime})^T)\le \frac{1}{2}\lambda_{\min}(\frac{|\mathcal{A}|}{n}\mathbbm{E}[\bm{z}_j^{\prime}(\bm{z}_j^{\prime})^T])\right)&\le s\exp\left(-\frac{\log(2)n\lambda_{\min}(\frac{|\mathcal{A}|}{n}\mathbbm{E}[\bm{z}_j^{\prime}(\bm{z}_j^{\prime})^T])}{4s\sigma_2x_{\max}^2}\right)\notag\\
	\Rightarrow \mathbbm{P}\left(\lambda_{\min}(\frac{1}{n}\sum_{j\in\mathcal{A}}^n\bm{z}_j^{\prime}(\bm{z}_j^{\prime})^T)\le \frac{|\mathcal{A}|}{2n}\lambda_{\min}(\mathbbm{E}[\bm{z}_j^{\prime}(\bm{z}_j^{\prime})^T])\right)&\le 2s\exp\left(-\frac{|\mathcal{A}|\lambda_{\min}(\mathbbm{E}[\bm{z}_j^{\prime}(\bm{z}_j^{\prime})^T])}{4s\sigma_2x_{\max}^2}\right),\label{eq:lemma:2.11_6}
	\end{align}
	where \eqref{eq:lemma:2.11_5.0} uses $
	0\le \lambda_{\min}(\frac{1}{n}\bm{z}_j^{\prime}(\bm{z}_j^{\prime})^T)\le \lambda_{\max}(\frac{1}{n}\bm{z}_j^{\prime}(\bm{z}_j^{\prime})^T)\le \frac{s}{n}(z_{\max}^{\prime})^2 = \frac{s}{n}\sigma_2x_{\max}^2$ and the last inequality comes from the assumption {\bf A.1}. 
	As we only consider the significant dimensions, under assumption {\bf A.4}, we can verify that there exists a $\mu_0>0$ such that
	$\mathbbm{E}[\bm{z}_j^{\prime}(\bm{z}_j^{\prime})^T]=\mathbbm{E}[\nabla_{\mathcal{S},\mathcal{S}}^2\mathcal{L}_{\mathcal{A}}(\bm{\xi})]\succeq \mu_0 I$. Then, we have $\mathbbm{1}\left(\lambda_{\min}(\frac{1}{n}\sum_{j\in\mathcal{A}}^n\bm{z}_j^{\prime}(\bm{z}_j^{\prime})^T)\le \frac{|\mathcal{A}|}{2n}\lambda_{\min}(\mathbbm{E}[\bm{z}_j^{\prime}(\bm{z}_j^{\prime})^T])\right)\ge \mathbbm{1}\left(\lambda_{\min}(\frac{1}{n}\sum_{j\in\mathcal{A}}^n\bm{z}_j^{\prime}(\bm{z}_j^{\prime})^T)\le \frac{|\mathcal{A}|}{2n}\mu_0\right) $.
	Thus \eqref{eq:lemma:2.11_6} implies
	\begin{align}
	\mathbbm{P}\left(\lambda_{\min}(\frac{1}{n}\sum_{j\in\mathcal{A}}^n\bm{z}_j^{\prime}(\bm{z}_j^{\prime})^T)\le  \frac{|\mathcal{A}|}{2n}\mu_0\right)
	&\le2s\exp\left(-\frac{|\mathcal{A}|\lambda_{\min}(\mathbbm{E}[\bm{z}_j^{\prime}(\bm{z}_j^{\prime})^T])}{4s\sigma_2x_{\max}^2}\right)\notag\\
	\Rightarrow
	\mathbbm{P}\left(\lambda_{\min}(\frac{1}{n}\sum_{j\in\mathcal{A}}^n\bm{z}_j^{\prime}(\bm{z}_j^{\prime})^T)\le  \frac{|\mathcal{A}|}{2n}\mu_0\right)
	&\le2s\exp\left(-\frac{|\mathcal{A}|\mu_0}{4s\sigma_2x_{\max}^2}\right)\notag\\
	\Rightarrow\mathbbm{P}\left(\lambda_{\min}(\nabla^2_{\mathcal{S},\mathcal{S}} \mathcal{L}(\bm{\xi}))\le  \frac{|\mathcal{A}|}{2n}\mu_0\right)
	&\le2s\exp\left(-\frac{|\mathcal{A}|\mu_0}{4s\sigma_2x_{\max}^2}\right).\notag
	\end{align}
	Lemma \ref{lemma:2.15} follows immediately.
\end{proof}

\vspace{15mm}

	\begin{lemma}\label{lemma:2.11}
		Let the whole sample size be $n$ and iid random sample set be $\mathcal{A}$. If assumptions {\bf A.1},{\bf A.4} and {\bf A.5} hold, there exist $\mu_0>0$ such that for $t>0$ we have: 
		\begin{align}
		\mathbbm{P}\left(\|\bm{\beta}^{MCP}-\bm{\beta}^{true}\|\ge \frac{2nt}{|\mathcal{A}|\mu_0}\right)\le 2s\exp\left(-\frac{|\mathcal{A}|\mu_0}{4s\sigma_2x_{\max}^2}\right)+s\exp\left(-\frac{nt^2}{2\sigma^2x_{\max}^2}\right).\label{eq:lemma:2.11_0}
		\end{align}
		Furthermore, if $|\mathcal{A}|\ge \frac{2s^2x_{\max}^2}{\mu_0}$ we have:
		\begin{align}
		\mathbbm{P}\left(\|\bm{\beta}^{MCP}-\bm{\beta}^{true}\|_2\le \sqrt{\frac{8s^2\sigma_2\sigma^2x_{\max}^2n}{\mu_0^2|\mathcal{A}|^2}}\right)\ge 1-2s\exp\left(-\frac{\mu_0|\mathcal{A}|}{4s\sigma_2x_{\max}^2}\right)-2\exp\left(-\frac{C_h|\mathcal{A}|\mu_0}{2sx_{\max}^2}\right).\label{eq:lemma:2.11_0.5}
		\end{align}
	\end{lemma}

\begin{proof}
	{Proof of Lemma \ref{lemma:2.11}}

	From the definition of oracle solution, we know
	\begin{align}
	\nabla_{\mathcal{S}} \mathcal{L}(\bm{\beta}^{oracle}) = 0.\label{eq:lemma:2.11_1}
	\end{align}
	Expanding \eqref{eq:lemma:2.11_1} at $\bm{\beta}^{true}$ we will have the following results for $\xi\in\{\tau\bm{\beta}^{oracle}+(1-\tau)\bm{\beta}^{true},\tau\in[0,1]\}$:
	\begin{align}
	\nabla_{\mathcal{S}} \mathcal{L}(\bm{\beta}^{true})+\nabla^2_{\mathcal{S},\mathcal{S}} \mathcal{L}(\xi)(\bm{\beta}^{oracle}-\bm{\beta}^{true})&=0\notag\\
	\nabla^2_{\mathcal{S},\mathcal{S}} \mathcal{L}(\xi)(\bm{\beta}^{oracle}-\bm{\beta}^{true})&=-\nabla_{\mathcal{S}} \mathcal{L}(\bm{\beta}^{true})\notag\\        (\bm{\beta}^{oracle}-\bm{\beta}^{true})^T\nabla^2_{\mathcal{S},\mathcal{S}} \mathcal{L}(\xi)(\bm{\beta}^{oracle}-\bm{\beta}^{true})&=-(\bm{\beta}^{oracle}-\bm{\beta}^{true})^T\nabla_\mathcal{S} \mathcal{L}(\bm{\beta}^{true})\notag\\
	\lambda_{\min}(\nabla^2_{\mathcal{S},\mathcal{S}} \mathcal{L}(\xi))\|(\bm{\beta}^{oracle}-\bm{\beta}^{true})\|_2^2&\le \|(\bm{\beta}^{oracle}-\bm{\beta}^{true})\|_2\|\nabla_{\mathcal{S}} \mathcal{L}(\bm{\beta}^{true})\|_2\notag\\
	\lambda_{\min}(\nabla^2_{\mathcal{S},\mathcal{S}} \mathcal{L}(\xi))\|(\bm{\beta}^{oracle}-\bm{\beta}^{true})\|_2&\le\|\nabla_{\mathcal{S}} \mathcal{L}(\bm{\beta}^{true})\|_2 .\label{eq:lemma:2.11_2}
	\end{align} 
	The $\lambda_{\min}(\nabla^2_{\mathcal{S},\mathcal{S}} \mathcal{L}(\xi))$ term on the left hand side of \refeq{eq:lemma:2.11_2} can be lower bounded away 0 via Lemma \ref{lemma:2.15} with high probability. Thus we only need to construct the upper bound for right-hand side of \eqref{eq:lemma:2.11_2}. The right-hand-side of \eqref{eq:lemma:2.11_2} can be expanded as follows:
	\begin{align}
	\|\nabla_{\mathcal{S}} \mathcal{L}(\bm{\beta}^{true})\|_2 = \left\|\frac{1}{n}\sum_{j=1}^nx_{j\mathcal{S}}^Tf^{'}(r_j|\bm{x}_{j,\mathcal{S}}^T\bm{\beta}^{true})\right\|_2. \end{align}
	Under assumption {\bf A.5}, we have $|f^{'}(r_j|\bm{x}_{j,\mathcal{S}}^T\bm{\beta}^{true})|\le \sigma$. Combining with $\mathbbm{E}[f^{'}(r_j|\bm{x}_{j,\mathcal{S}}^T\bm{\beta}^{true})] = 0$, we can verify that $f^{'}(r_j|\bm{x}_{j,\mathcal{S}}^T\bm{\beta}^{true})$ is a $\sigma$-subgaussian random variable.
	From hoeffding inequality, there exists a $t>0$ such that
	\begin{align}
	&\mathbbm{P}\left(\left|\frac{1}{n}\sum_{j=1}^nx_{ji}^Tf^{'}(r_j|\bm{x}_j^T\bm{\beta}^{true})\right|\ge t\right)\le \exp\left(-\frac{nt^2}{2\sigma^2x_{\max}^2}\right)\quad \forall i\in\mathcal{S}.
	\end{align}
	Hence, we have
	\begin{align}
	\mathbbm{P}\left(\|\nabla_{\mathcal{S}} \mathcal{L}(\bm{\beta}^{true})\|_2\ge t\right)=\mathbbm{P}\left(\left\|\frac{1}{n}\sum_{j=1}^nx_{ji}^Tf^{'}(r_j|\bm{x}_j^T\bm{\beta}^{true})\right\|_2\ge t\right)&\le \mathbbm{P}\left(\sqrt{|\mathcal{S}|}\left\|\frac{1}{n}\sum_{j=1}^nx_{ji}^Tf^{'}(r_j|\bm{x}_j^T\bm{\beta}^{true})\right\|_{\infty}\ge t\right)\label{eq:lemma:2.11_8}
	\\
	&\le s\exp\left(-\frac{nt^2}{2\sigma^2x_{\max}^2}\right)\label{eq:lemma:2.11_9},
	\end{align}
	where the inequality in \eqref{eq:lemma:2.11_8} follows from the fact that $\|p\|\le \sqrt{\|p\|_{0}}\|p\|_{\infty}$ holds for any vector $p$ and the inequality in \eqref{eq:lemma:2.11_9} follows from $|\mathcal{S}|\le s$. Combining \eqref{eq:lemma:2.11_9}, \eqref{eq:lemma:2.11_2} and Lemma \ref{lemma:2.15}, we have
	\begin{align}
	\mathbbm{P}\left(\|\bm{\beta}^{oracle}-\bm{\beta}^{true}\|\ge \frac{2nt}{|\mathcal{A}|\mu_0}\right)\le 2s\exp\left(-\frac{|\mathcal{A}|\mu_0}{4s\sigma_2x_{\max}^2}\right)+s\exp\left(-\frac{nt^2}{2\sigma^2x_{\max}^2}\right)
	\end{align}
	Now, the first half of Lemma \ref{lemma:2.11} has been proven, and we switch to the second half.  

	Denote $\bm{\epsilon} = [\epsilon_1,\epsilon_2,...,\epsilon_n]$ where $\epsilon_j =f^{'}(r_j|\bm{x}_{j,\mathcal{S}}^T\bm{\beta}^{true}),\ j=1,2,..,n$. Then $\nabla_{\mathcal{S}}\mathcal{L}(\bm{\beta^{true}})$ can be rewritten as $\nabla_{\mathcal{S}}\mathcal{L}(\bm{\beta^{true}}) = \frac{1}{n}\bm{X}_{\mathcal{S}}\bm{\epsilon}$ with $\bm{X}_{\mathcal{S}} = [\bm{x}_{1,\mathcal{S}},...,\bm{x}_{n,\mathcal{S}}]$. 
	Using the Hanson-Wright inequality (Theorem 1.1 in \citealt{rudelson2013hanson}), we have 
	\begin{align}
	&\mathbbm{P}\{|\bm{\epsilon}^T(\frac{1}{n}\bm{X}_{\mathcal{S}}^T\bm{X}_{\mathcal{S}})\bm{\epsilon}- \mathbbm{E}[\bm{\epsilon}^T(\frac{1}{n}\bm{X}_{\mathcal{S}}^T\bm{X}_{\mathcal{S}})\bm{\epsilon}]|> \mathbbm{E}[\bm{\epsilon}^T(\frac{1}{n}\bm{X}_{\mathcal{S}}^TX_{\mathcal{S}})\bm{\epsilon}] \}\notag\\
	\le& 2\exp \left(-C_h\min\left\{\frac{\mathbbm{E}[\bm{\epsilon}^T(\frac{1}{n}\bm{X}_{\mathcal{S}}^T\bm{X}_{\mathcal{S}})\bm{\epsilon}]}{\sigma^2\|\frac{1}{n}\bm{X}_{\mathcal{S}}^T\bm{X}_{\mathcal{S}}\|_2},\frac{(\mathbbm{E}[\bm{\epsilon}^T(\frac{1}{n}\bm{X}_{\mathcal{S}}^T\bm{X}_{\mathcal{S}})\bm{\epsilon}])^2}{\sigma^4\|\frac{1}{n}\bm{X}_{\mathcal{S}}^T\bm{X}_{\mathcal{S}}\|_F^2}\right\}\right)\notag\\
	\le& 2\exp\left(-C_h\min\left\{\frac{\lambda_{\min}(\frac{1}{n}\bm{X}_{\mathcal{S}}^T\bm{X}_{\mathcal{S}})}{\lambda_{\max}(\frac{1}{n}\bm{X}_{\mathcal{S}}^T\bm{X}_{\mathcal{S}})}\frac{\mathbbm{E}[\bm{\epsilon}^T\bm{\epsilon}]}{\sigma^2},\frac{\lambda_{\min}(\frac{1}{n}\bm{X}_{\mathcal{S}}^T\bm{X}_{\mathcal{S}})^2}{\lambda_{\max}(\frac{1}{n}\bm{X}_{\mathcal{S}}^T\bm{X}_{\mathcal{S}})^2}\frac{\mathbbm{E}[\bm{\epsilon}^T\bm{\epsilon}]^2}{s\sigma^4}\right\}\right)\notag\\
	\le& 2\exp\left(-C_h\min\left\{ n\frac{\lambda_{\min}(\frac{1}{n}\bm{X}_{\mathcal{S}}^T\bm{X}_{\mathcal{S}})}{\lambda_{\max}(\frac{1}{n}\bm{X}_{\mathcal{S}}^T\bm{X}_{\mathcal{S}})},\frac{n^2}{s}\frac{\lambda_{\min}(\frac{1}{n}\bm{X}_{\mathcal{S}}^T\bm{X}_{\mathcal{S}})^2}{\lambda_{\max}(\frac{1}{n}\bm{X}_{\mathcal{S}}^T\bm{X}_{\mathcal{S}})^2}\right\}\right)\notag\\
	\le& 2\exp\left(- n\frac{C_h\lambda_{\min}(\frac{1}{n}\bm{X}_{\mathcal{S}}^T\bm{X}_{\mathcal{S}})}{\lambda_{\max}(\frac{1}{n}\bm{X}_{\mathcal{S}}^T\bm{X}_{\mathcal{S}})}\right)\label{eq:lemma:2.11_10},
	\end{align}
	where $C_h$ is a positive constant. The last inequality, \eqref{eq:lemma:2.11_10}, holds when $ n \ge s\frac{\lambda_{\max}(\frac{1}{n}\bm{X}^T_{\mathcal{S}}\bm{X}_{\mathcal{S}})}{\lambda_{\min}(\frac{1}{n}\bm{X}_{\mathcal{S}}^T\bm{X}_{\mathcal{S}})}$.
	Define the event $\mathcal{E}_1$ as follows:
	\begin{align}
	\mathcal{E}_1 = \left\{|\bm{\epsilon}^T(\frac{1}{n}\bm{X}_{\mathcal{S}}^T\bm{X}_{\mathcal{S}})\bm{\epsilon}- \mathbbm{E}[\bm{\epsilon}^T(\frac{1}{n}\bm{X}_{\mathcal{S}}^T\bm{X}_{\mathcal{S}})\bm{\epsilon}]|\le \mathbbm{E}[\bm{\epsilon}^T(\frac{1}{n}\bm{X}_{\mathcal{S}}^T\bm{X}_{\mathcal{S}})\bm{\epsilon}]\right\}.         \label{eq:event_1}
	\end{align}
	Under event $\mathcal{E}_1$, we have
	\begin{align}
	\|\frac{1}{n}\bm{X}_{\mathcal{S}}\bm{\epsilon}\|_2\le \sqrt{\frac{1}{n}\bm{\epsilon}^T(\frac{1}{n}\bm{X}_{\mathcal{S}}^T\bm{X}_{\mathcal{S}})\bm{\epsilon}}&\le \sqrt{\frac{2}{n}\mathbbm{E}[\bm{\epsilon}^T(\frac{1}{n}\bm{X}_{\mathcal{S}}^T\bm{X}_{\mathcal{S}})\bm{\epsilon}]}.\label{eq:lemma:2.11_11_1}
	\end{align}
	Let $\bm{P}_j = \bm{X}_{\mathcal{S}}(\bm{X}_{\mathcal{S}}^T\bm{X}_{\mathcal{S}})^{-1}\bm{X}_{\mathcal{S}}^T$. We have $(\bm{P}_j\bm{\epsilon})^T(\frac{1}{n}\bm{X}_{\mathcal{S}}^T\bm{X}_{\mathcal{S}})(\bm{P}_j\bm{\epsilon})=\bm{\epsilon}^T(\frac{1}{n}\bm{X}_{\mathcal{S}}^T\bm{X}_{\mathcal{S}})\bm{\epsilon}$, and \eqref{eq:lemma:2.11_11_1} implies the following result:
	\begin{align}
	\|\frac{1}{n}\bm{X}_{\mathcal{S}}\bm{\epsilon}\|_2&\le \sqrt{\frac{2}{n}\mathbbm{E}[(\bm{P}_j\bm{\epsilon})^T(\frac{1}{n}\bm{X}_{\mathcal{S}}^T\bm{X}_{\mathcal{S}})(\bm{P}_j\bm{\epsilon})]}\notag\\
	&\le \sqrt{\frac{2}{n}\lambda_{\max}(\frac{1}{n}\bm{X}_{\mathcal{S}}^T\bm{X}_{\mathcal{S}})\mathbbm{E}[\|\bm{P}_j\bm{\epsilon}\|_2^2]}\notag\\
	&=\sqrt{2\lambda_{\max}(\frac{1}{n}\bm{X}_{\mathcal{S}}^T\bm{X}_{\mathcal{S}})\sigma_2\sigma^2\frac{s}{n}},\label{eq:lemma:2.11_11}
	\end{align}
	where the last inequality comes from $\mathbbm{E}[\bm{P}_j\bm{\epsilon}]=0$ and $\mathbbm{E}[\|\bm{P}_j\bm{\epsilon}\|_2^2] = Var(\bm{P}_j\bm{\epsilon}) = s\sigma_2\sigma^2$ in which $\bm{\epsilon}_j$ is a  $\sqrt{\sigma_2}\sigma$-subguassian random variable. From \eqref{eq:lemma:2.11_11} and \eqref{eq:lemma:2.11_10}, we have the following inequalities:
	\begin{align}
	&\mathbbm{P}\left(\|\frac{1}{n}\bm{X}_{\mathcal{S}}\bm{\epsilon}\|\le \sqrt{2\lambda_{\max}(\frac{1}{n}\bm{X}_{\mathcal{S}}^T\bm{X}_{\mathcal{S}})\sigma_2\sigma^2\frac{s}{n}}\right)\le 2\exp\left(-n\frac{C_h\lambda_{\min}(\frac{1}{n}\bm{X}_{\mathcal{S}}^T\bm{X}_{\mathcal{S}})}{\lambda_{\max}(\frac{1}{n}\bm{X}_{\mathcal{S}}^T\bm{X}_{\mathcal{S}})}\right)\notag\\
	\Rightarrow& \mathbbm{P}\left(\|\nabla_{\mathcal{S}}\mathcal{L}(\bm{\beta}^{true})\|\le \sqrt{2\lambda_{\max}(\frac{1}{n}\bm{X}_{\mathcal{S}}^T\bm{X}_{\mathcal{S}})\sigma_2\sigma^2\frac{s}{n}}\right)\le 2\exp\left(-n\frac{C_h\lambda_{\min}(\frac{1}{n}\bm{X}_{\mathcal{S}}^T\bm{X}_{\mathcal{S}})}{\lambda_{\max}(\frac{1}{n}\bm{X}_{\mathcal{S}}^T\bm{X}_{\mathcal{S}})}\right)\notag\\
	\Rightarrow& \mathbbm{P}\left(\|\nabla_{\mathcal{S}}\mathcal{L}(\bm{\beta}^{true})\|\le \sqrt{\frac{2s^2\sigma_2\sigma^2x_{\max}^2}{n}}\right)\le 2\exp\left(-n\frac{C_h\lambda_{\min}(\frac{1}{n}\bm{X}_{\mathcal{S}}^T\bm{X}_{\mathcal{S}})}{sx_{\max}^2}\right)\label{eq:lemma:2.11_12},
	\end{align}
	where the last inequality, \eqref{eq:lemma:2.11_12}, uses $\lambda_{\max}(\frac{1}{n}\bm{X}_{\mathcal{S}}^T\bm{X}_{\mathcal{S}})\le sx_{\max}^2$.
	Combining \eqref{eq:lemma:2.11_12}, \eqref{eq:lemma:2.11_2} and Lemma \ref{lemma:2.15} we have the following result:
	\begin{align}
	\mathbbm{P}\left(\|\bm{\beta}^{oracle}-\bm{\beta}^{true}\|_2\le \sqrt{\frac{8s^2\sigma_2\sigma^2x_{\max}^2n}{\mu_0^2|\mathcal{A}|^2}}\right)\ge 1-2s\exp\left(-\frac{\mu_0|\mathcal{A}|}{4s\sigma_2x_{\max}^2}\right)-2\exp\left(-\frac{C_h|\mathcal{A}|\mu_0}{2sx_{\max}^2}\right)
	\end{align}
\end{proof}

\vspace{15mm}
\begin{lemma}\label{lemma:2.2}
	If there exists $K$ and $\sigma_0$ such that $K^2\left(E[\exp(\bm{z}_{t,i}^2/K^2)-1]\right)\le \sigma_0^2,$ then the following probability bound will hold for all $t>0$:
	\begin{equation}
	P\left\{\|\frac{1}{n}\sum_{j=1}^n\bm{z}_j\bm{z}_j^T-E[\bm{z}_j\bm{z}_j^T]\|_{\infty}\ge 2K^2t+2K\sigma_0\sqrt{2t}+2K\sigma_0\lambda\left(\frac{K}{\sigma_0},n,\binom{d}{2}\right)\right\}\le \exp\left(-nt\right)
	\end{equation}
	where
	$\lambda\left(\frac{K}{\sigma_0},n,\binom{d}{2}\right) = \sqrt{\frac{2\log(d(d-1))}{n}}+\frac{K\log(d(d-1))}{n}$.
\end{lemma}

\begin{proof} {Proof of ~\ref{lemma:2.2}}
	From the exercise 14.3 in \cite{buhlmann2011statistics}.
\end{proof}

\vspace{15mm}
\begin{lemma}\label{lemma:2.3}
	If there exist $\kappa_0$, $S$, and $\bm{z}_j,\ j =1,2,..,n$ such that
$\|u_S\|_1^2\le \frac{|S|}{\kappa_0}u^T\mathbbm{E}[\bm{z}_j\bm{z}_j^T]u$ holds for all $u\in \mathcal{U}\doteq \{u: \|u_{S^c}\|_1\le3\|u_{S}\|\}$ and  $	\left\|\frac{1}{n}\sum_{j=1}^n\bm{z}_j\bm{z}_j^T-\mathbbm{E}[\bm{z}_j\bm{z}_j^T]\right\|\le \frac{\kappa}{32|S|}$,
	then for all $u\in \mathcal{U}$, the follow inequality holds:
	\begin{align}
	\|u_S\|_1^2\le \frac{|S|}{\kappa_0/2}u^T\left[\frac{1}{n}\sum_{j=1}^n\bm{z}_j\bm{z}_j^T\right]u
	\end{align}
\end{lemma}

\begin{proof} {Proof of ~\ref{lemma:2.3}}
	From  Corollary 6.8 in \cite{buhlmann2011statistics}.
\end{proof}


\vspace{15mm}

\begin{lemma}\label{lemma:2.8}
	Let $\bm{x}_j$, $j = 1,2,...,n$, be random iid samples. Under assumptions {\bf A. 4} and {\bf A. 5},  the follow inequality holds for all $\bm{u}$ such that $\|\bm{u}_{\mathcal{S}^c}\|_1\le 3\|\bm{u}_{\mathcal{S}}\|_1$:
	\begin{align}
	\mathbbm{P}\left(\frac{\kappa}{2s}\|\bm{u}_{\mathcal{S}}\|_1^2\le u^T\nabla^2\mathcal{L}(\bm{\beta})\bm{u}\right)\ge 1-\exp(-C_1n),\label{eq:lemma:2.8_0}
	\end{align} where $C_1 =\min\left\{1,\kappa^2/\left(192s\sigma_2x_{\max}^2(2+\sqrt{\sigma_2}x_{\max})\right)^2\right\}$.
\end{lemma}

\begin{proof} {Proof of ~\ref{lemma:2.8}}
	From the definition of $\mathcal{L}(\bm{\beta})$, we have:
	\begin{align}
	\nabla^2\mathcal{L}(\bm{\xi}) = \frac{1}{n}\sum_{j=1}^n\bm{x}_j\bm{x}_j^Tf^{''}(r_j,\bm{x}_j^T\bm{\xi}).
	\end{align}
	Under assumption {\bf A. 5}, we know that $f$ is convex, and thus we have $f^{''}\ge 0$. Let $\bm{z}_j = \bm{x}_j\sqrt{f^{''}(r_j,\bm{x}_j^T\bm{\xi})}$. We can expand $\nabla^2\mathcal{L}(\bm{\xi})$ to $\nabla^2\mathcal{L}(\bm{\xi}) = \frac{1}{n}\sum_{j=1}^n\bm{z}_j\bm{z}_j^T$. Furthermore, under assumption {\bf A. 1} and {\bf A. 5}, we have $|f^{''}(r_j|\bm{x}_j^T\bm{\xi})|\le \sigma_2$ and $\|\bm{x}\|_{\infty}\le x_{\max}$, which implies that $\bm{z}_j$ is element-wise bounded by
	$ z_{\max}\doteq\|\bm{z}_j\|_{\infty} =\left\|\bm{x}_j\sqrt{f^{''}(r_j|\bm{x}_j^T\bm{\xi})}\right\|_{\infty}\le \sqrt{\sigma_2}x_{\max}$.
	Since $\bm{z}_j$ is bounded, it will satisfy the definition of the subguassian random variable. We can use the Lemma \ref{lemma:2.2} as a bridge to connect the sample matrix $\frac{1}{n}\sum_{j=1}^n\bm{z}_j\bm{z}_j^T$ to its population counterpart $\mathbbm{E}[\bm{z}_j\bm{z}_j^T]$. Let $K = z_{\max}$ and $\sigma_0 = \sqrt{2}z_{\max}$ and we will have $K^2\left(E[\exp(\bm{z}_{t,i}^2/K^2)-1]\right)\le z_{\max}^2(e-1)\le \sigma_0^2$ for all $t\ge 0$ and $i=1,2,...,d$.  Therefore, under Lemma \ref{lemma:2.2}, for $t>0$, we have:
	\begin{equation}
	P\left\{\left\|\frac{1}{n}\sum_{j=1}^n\bm{z}_j\bm{z}_j^T-E[\bm{z}_j\bm{z}_j^T]\right\|_{\infty}\ge 2z_{\max}^2t+4z_{\max}^2\sqrt{t}+\sqrt{8}z_{\max}^2\lambda\left(\frac{\sqrt{2}}{2},n,\binom{d}{2}\right)\right\}\le \exp\left(-nt\right)\label{eq:lemma:2.8_1}
	\end{equation}
	where
	$\lambda\left(\frac{\sqrt{2}}{2},n,\binom{d}{2}\right) = \sqrt{\frac{2\log(d(d-1))}{n}}+\frac{z_{\max}\log(d(d-1))}{n}$.
	\eqref{eq:lemma:2.8_1} indicates that when the sample size is large enough, $\frac{1}{n}\sum_{j=1}^n\bm{z}_j\bm{z}_j^T$ will not be far away from $\mathbbm{E}[\bm{z}_j\bm{z}_j^T]$ element-wise with high probability.

	Now we only need to show that if $\frac{1}{n}\sum_{j=1}^n\bm{z}_j\bm{z}_j^T$ is close enough to $\mathbbm{E}[\bm{z}_j\bm{z}_j^T]$, $\nabla^2\mathcal{L}$ satisfies \eqref{eq:lemma:2.8_0}. To this end, we need Lemma \ref{lemma:2.3}.
	We set $n\ge\log d/C_1$ and $t=C_1$ in \eqref{eq:lemma:2.8_1}. Then the following inequalities hold.
	\begin{align}
	2z_{\max}^2t+4z_{\max}^2\sqrt{t}\le2z_{\max}^2\sqrt{C_1}+4z_{\max}^2\sqrt{C_1}&= 6z_{\max}^2\sqrt{C_1}\label{eq:lemma:2.8_2}\\
	\sqrt{8}z_{\max}^2\lambda\left(\frac{\sqrt{2}}{2},n,\binom{d}{2}\right)\le \sqrt{8}z_{\max}^2\left(\sqrt{\frac{2\log(d^2)}{n}}+\frac{z_{\max}\log(d^2)}{n}\right)&\le 4\sqrt{2}z_{\max}^2(1+z_{\max})\sqrt{C_1},\label{eq:lemma:2.8_3}
	\end{align}
	where \eqref{eq:lemma:2.8_2} and \eqref{eq:lemma:2.8_3} use $\log d/n\le C_1\le 1$. Combining \eqref{eq:lemma:2.8_2} and \eqref{eq:lemma:2.8_3}, we have:
	\begin{align}
	2z_{\max}^2t+4z_{\max}^2\sqrt{t}+\sqrt{8}z_{\max}^2\lambda\left(\frac{\sqrt{2}}{2},n,\binom{d}{2}\right)&\le2z_{\max}^2\left(3+2\sqrt{2}(1+z_{\max})\right)\sqrt{C_1}\notag\\
	&\le6z_{\max}^2\left(2+z_{\max})\right)\sqrt{C_1}\le \frac{\kappa}{32s},\label{eq:lemma:2.8_4}
	\end{align}
	where \eqref{eq:lemma:2.8_4} uses $\sqrt{2}\le \frac{3}{2}$ and $C_1\le \kappa^2/\left(192s\sigma_2x_{\max}^2\left(2+\sqrt{\sigma_2}x_{\max})\right)\right)^2\le\kappa^2/\left(192sz_{\max}^2\left(2+z_{\max})\right)\right)^2$. Then,  \eqref{eq:lemma:2.8_1} can satisfy the following inequality:
	\begin{align}
	\mathbbm{P}\left\{\left\|\frac{1}{n}\sum_{j=1}^n\bm{z}_j\bm{z}_j^T-E[\bm{z}_j\bm{z}_j^T]\right\|_{\infty}\le \frac{\kappa}{32s}\right\}\ge 1-\exp\left(-C_1n\right).\label{eq:lemma:2.8_5}\end{align}
	From \eqref{eq:lemma:2.8_5} and Lemma \ref{lemma:2.3}, we can conclude that  $\|\bm{u}_S\|_1^2\le \frac{s}{\kappa/2}\bm{u}^T\nabla^2 \mathcal{L}(\bm{\beta})\bm{u} $ for $u$ such that $\|\bm{u}_{S^c}\|_1\le 3\|\bm{u}_S\|_1$ with probability  $1-\exp\left(-C_1n\right)$. The statement of Lemma \ref{lemma:2.8} follows immediately.
\end{proof}

\vspace{15mm}
\begin{lemma}\label{lemma:2.9}
	Let $\mathcal{A}^{iid}_k$ be the index set such that for all $i\in \mathcal{A}^{iid}_k$, $\bm{x}_i$ are random iid samples. If for all $\bm{u}$ such that $\|\bm{u}_{S^c}\|_1\le 3\|\bm{u}_S\|_1$, we have $\frac{\kappa}{2s}\|\bm{u}_S\|_1^2\le \bm{u}^T\nabla^2\mathcal{L}_{\mathcal{A}_k^{iid}}(\bm{\xi})\bm{u}$,
	then the follow inequality holds:
	\begin{align}
	\frac{|\mathcal{A}|\kappa}{2ns}\|\bm{u}_S\|_1^2\le \bm{u}^T\nabla^2\mathcal{L}(\bm{\xi})\bm{u},
	\end{align}
	where $\mathcal{L}_{\mathcal{A}}(\bm{\beta})$ denote the likelihood function with samples only in $\mathcal{A}$.
\end{lemma}

\begin{proof} {proof of ~\ref{lemma:2.9}}
	We can rewrite $\nabla\mathcal{L}(\bm{\beta})$ with $z_j$:
	\begin{align}
	\bm{u}^T\nabla^2\mathcal{L}(\xi)\bm{u} &= \bm{u}^T\left[\frac{1}{n}\sum_{j\in\mathcal{A}^{iid}_k}\bm{z}_j\bm{z}_j^T\right]\bm{u}+\bm{u}^T\left[\frac{1}{n}\sum_{j\in(\mathcal{A}^{iid}_k)^c}\bm{z}_j\bm{z}_j^T\right]\bm{}\notag\\
	&\ge\frac{|\mathcal{A}^{iid}_k|}{n}\bm{}^T\left[\frac{1}{|\mathcal{A}^{iid}_k|}\sum_{j\in\mathcal{A}^{iid}_k}\bm{z}_j\bm{z}_j^T\right]\bm{}\notag\\
	&\ge\frac{|\mathcal{A}^{iid}_k|}{n}\bm{}^T\nabla \mathcal{L}_{\mathcal{A}}(\bm{\beta})\bm{}\notag\\
	&\ge \frac{|\mathcal{A}^{iid}_k|}{n}\frac{\kappa}{2s}\|\bm{u}_S\|_1^2 \notag\\
	&=\frac{|\mathcal{A}^{iid}_k|\kappa}{2ns}\|\bm{u}_S\|_1^2.
	\end{align}
\end{proof}

\vspace{15mm}
	\begin{lemma}\label{lemma:2.10}
		Let the whole sample size be $n$ and the set for iid random sample in $U_{k}$ be $\mathcal{A}$. If assumptions {\bf A.4} and {\bf A	.5} hold, then the follow result holds:
		\begin{align}
		\mathbbm{P}\left(\|\bm{\beta}^{lasso}-\bm{\beta}^{true}\|_1\le \frac{96ns\lambda}{|\mathcal{A}|\kappa}\right)\ge 1-\exp\left(-C_1|\mathcal{A}|\right)-\exp\left(-\frac{n\lambda^2}{8x_{\max}^2}+\log d\right),
		\end{align}
		where $C_1 =\min\left\{1,\kappa^2/\left(192s\sigma_2x_{\max}^2(2+\sqrt{\sigma_2}x_{\max})\right)^2\right\}$.

	\end{lemma}

\begin{proof}{Proof of lemma ~\ref{lemma:2.10}}
	Let $\mathcal{L}_{\mathcal{A}}(\beta)$ be the loss function only includes samples in $\mathcal{A}$. Under assumption {\bf A.4}, we have
	\begin{align}
	\frac{\kappa}{s}\|\bm{u}_{\mathcal{S}}\|_1^2\le \bm{u}^T\mathbbm{E}[\nabla^2\mathcal{L}_{\mathcal{A}}(\bm{\xi})]\bm{u}\label{eq:lemma:2.10_1},
	\end{align}
	for all $u$ such that $\|\bm{u}_{\mathcal{S}^c}\|_1\le 3\|\bm{u}_{\mathcal{S}}\|_1$. The following result follows from \eqref{eq:lemma:2.10_1} and Lemma \ref{lemma:2.8}:
	\begin{align}
	\mathbbm{P}\left(\frac{\kappa}{2s}\|\bm{u}_{\mathcal{S}}\|_1^2\le \bm{u}^T\nabla^2\mathcal{L}_{\mathcal{A}}(\bm{\xi})\bm{u}\right)\ge 1-\exp(-C_1|\mathcal{A}|),
	\end{align}
	where $C_1$ is defined in Lemma \ref{lemma:2.8}. From Lemma \ref{lemma:2.9}, the follow inequality holds.
	\begin{align}
	\mathbbm{P}\left(               \frac{|\mathcal{A}|\kappa}{2ns}\|\bm{u}_{\mathcal{S}}\|_1^2\le \bm{u}^T\nabla^2\mathcal{L}(\bm{\xi})\bm{u}\right)\ge 1-\exp\left(-C_1|\mathcal{A}|\right)\label{eq:lemma:2.10_1.5}
	\end{align}
	holds for all $u$ such that $\|\bm{u}_{\mathcal{S}^c}\|_1\le 3\|\bm{u}_{\mathcal{S}}\|_1$.

	Since $\bm{\beta}^{lasso}$ is the optimal solution to the Lasso problem, we can ensure the following inequality:
	\begin{align}
	\mathcal{L}(\bm{\beta}^{lasso})+\lambda\|\bm{\beta}^{lasso}\|_1&\le \mathcal{L}(\bm{\beta}^{true})+\lambda\|\bm{\beta}^{true}\|_1\notag\\
	\mathcal{L}(\bm{\beta}^{lasso})-\mathcal{L}(\bm{\beta}^{true})+\lambda\|\bm{\beta}^{lasso}\|_1&\le \lambda\|\bm{\beta}^{true}\|_1\label{eq:lemma:2.10_1.6}\\
	\nabla\mathcal{L}(\bm{\beta}^{true})^T(\bm{\beta}^{lasso}-\bm{\beta}^{true})+\lambda\|\bm{\beta}^{lasso}\|_1&\le \lambda\|\bm{\beta}^{true}\|_1\label{eq:lemma:2.10_2}\\
	-\|\nabla\mathcal{L}(\bm{\beta}^{true})\|_{\infty}\|\bm{\beta}^{lasso}-\bm{\beta}^{true}\|_1+\lambda\|\bm{\beta}^{lasso}\|_1&\le \lambda\|\bm{\beta}^{true}\|_1\label{eq:lemma:2.10_3},
	\end{align}
	where \eqref{eq:lemma:2.10_2} uses the convexity of $\mathcal{L}(\bm{\beta}^{lasso})$. Denote event $\mathcal{E}_0$ as follow:
	\begin{align}
	\mathcal{E}_0 = \left\{\|\nabla\mathcal{L}(\bm{\beta}^{true})\|_{\infty}<\frac{1}{2}\lambda\right\}\label{eq:event_0}.
	\end{align}
	Under $\mathcal{E}_0$, \eqref{eq:lemma:2.10_3} can be further simplified into:
	\begin{align}
	-\frac{1}{2}\lambda\|\bm{\beta}^{lasso}-\bm{\beta}^{true}\|_1+\lambda\|\bm{\beta}^{lasso}\|_1&\le \lambda\|\bm{\beta}^{true}\|_1\notag\\
	-\frac{1}{2}\|\bm{\beta}^{lasso}-\bm{\beta}^{true}\|_1+\|\bm{\beta}^{lasso}\|_1&\le \|\bm{\beta}^{true}\|_1\notag\\
	-\frac{1}{2}\|\bm{\beta}^{lasso}_S-\bm{\beta}^{true}_S\|_1-\frac{1}{2}\|\bm{\beta}^{lasso}_{S^c}-\bm{\beta}^{true}_{S^c}\|_1+\|\bm{\beta}^{lasso}_S\|_1+\|\bm{\beta}^{lasso}_{S^c}\|_1&\le \|\bm{\beta}^{true}_S\|_1+ \|\bm{\beta}^{true}_{S^c}\|_1.
	\end{align}
	As $\bm{\beta}^{true}_{S^c} =\bm{0}$ by definition,  we then have:
	\begin{align}
	-\frac{1}{2}\|\bm{\beta}^{lasso}_S-\bm{\beta}^{true}_S\|_1-\frac{1}{2}\|\bm{\beta}^{lasso}_{S^c}-\bm{\beta}^{true}_{S^c}\|_1+\|\bm{\beta}^{lasso}_S\|_1+\|\bm{\beta}^{lasso}_{S^c}-\bm{0}\|_1\le \|\bm{\beta}^{true}_S\|_1+0\notag\\
	-\frac{1}{2}\|\bm{\beta}^{lasso}_S-\bm{\beta}^{true}_S\|_1-\frac{1}{2}\|\bm{\beta}^{lasso}_{S^c}-\bm{\beta}^{true}_{S^c}\|_1+\|\bm{\beta}^{lasso}_S\|_1+\|\bm{\beta}^{lasso}_{S^c}-\bm{\beta}^{true}_{S^c}\|_1\le \|\bm{\beta}^{true}_S\|_1+0\label{eq:lemma:2.10_4}
	\end{align}
	Rearranging \eqref{eq:lemma:2.10_4}, we will have :
	\begin{align}
	\frac{1}{2}\|\bm{\beta}^{lasso}_{S^c}-\bm{\beta}^{true}_{S^c}\|_1&\le \frac{1}{2}\|\bm{\beta}^{lasso}_S-\bm{\beta}^{true}_S\|_1+(\|\bm{\beta}^{true}_S\|_1-\|\bm{\beta}^{lasso}_S\|_1)\notag\\
	&\le \frac{1}{2}\|\bm{\beta}^{lasso}_S-\bm{\beta}^{true}_S\|_1+(\|\bm{\beta}^{lasso}_S-\bm{\beta}^{true}_S\|_1)\notag\\
	&\le \frac{3}{2}\|\bm{\beta}^{lasso}_S-\bm{\beta}^{true}_S\|_1\notag\\
	\Rightarrow     \|\bm{\beta}^{lasso}_{S^c}-\bm{\beta}^{true}_{S^c}\|_1&\le3\|\bm{\beta}^{lasso}_S-\bm{\beta}^{true}_S\|_1\label{eq:lemma:2.10_5}
	\end{align}
	Denote $\bm{u} = \bm{\beta}^{lasso}-\bm{\beta}^{true}$. Then, we have $\|\bm{u}_{S^c}\|_1\le 3\|\bm{u}_S\|_1$. Connecting \eqref{eq:lemma:2.10_1.5}, we can obtain
	\begin{align}
	\mathbbm{P}\left((\bm{\beta}^{lasso}-\bm{\beta}^{true})^T\nabla^2\mathcal{L}(\bm{\xi})(\bm{\beta}^{lasso}-\bm{\beta}^{true})\ge \frac{|\mathcal{A}|\kappa}{2ns} \|\bm{\beta}^{lasso}_S-\bm{\beta}^{true}_S\|_1^2\right)\ge 1-\exp\left(-C_1|\mathcal{A}|\right).\label{eq:lemma:2.10_6}
	\end{align}
	Now, we turn back to \eqref{eq:lemma:2.10_1.6} and use the second-order Taylor expansion on $\mathcal{L}(\bm{\beta}^{lasso})$ at $\bm{\beta}^{true}$:
	\begin{align}
	\mathcal{L}(\bm{\beta}^{lasso})-\mathcal{L}(\bm{\beta}^{true})+\lambda\|\bm{\beta}^{lasso}\|_1\le \lambda\|\bm{\beta}^{true}\|_1\notag\\
	\nabla\mathcal{L}(\bm{\beta}^{true})^T(\bm{\beta}^{lasso}-\bm{\beta}^{true})+\frac{1}{2}(\bm{\beta}^{lasso}-\bm{\beta}^{true})^T\nabla^2\mathcal{L}(\xi)(\bm{\beta}^{lasso}-\bm{\beta}^{true})+\lambda\|\bm{\beta}^{lasso}\|_1\le \lambda\|\bm{\beta}^{true}\|_1\label{eq:lemma:2.10_1.7}
	\end{align}
	Combining \eqref{eq:lemma:2.10_1.6} and \eqref{eq:lemma:2.10_1.7}, we know that with probability $1-\exp(-C_1n)$, the follow results hold:
	\begin{align}
	&-\|\nabla\mathcal{L}(\bm{\beta}^{true})\|_{\infty}\|(\bm{\beta}^{lasso}-\bm{\beta}^{true})\|_1+\frac{|\mathcal{A}|\kappa}{4ns}\|\bm{\beta}^{lasso}-\bm{\beta}^{true}\|^2_1+\lambda\|\bm{\beta}^{lasso}\|_1\le \lambda\|\bm{\beta}^{true}\|_1\notag\\
	\Rightarrow&-\|\nabla\mathcal{L}(\bm{\beta}^{true})\|_{\infty}\|(\bm{\beta}^{lasso}-\bm{\beta}^{true})\|_1+\frac{|\mathcal{A}|\kappa}{4ns}\|\bm{\beta}^{lasso}-\bm{\beta}^{true}\|^2_1\le \lambda(\|\bm{\beta}^{true}\|_1-\|\bm{\beta}^{lasso}\|_1)\notag\\
	\Rightarrow&-\|\nabla\mathcal{L}(\bm{\beta}^{true})\|_{\infty}\|(\bm{\beta}^{lasso}-\bm{\beta}^{true})\|_1+\frac{|\mathcal{A}|\kappa}{4ns}\|\bm{\beta}^{lasso}_S-\bm{\beta}^{true}_S\|^2_1\le \lambda\|\bm{\beta}^{true}-\bm{\beta}^{lasso}\|_1
	\end{align}
	Under event $\mathcal{E}_0$, we have:
	\begin{align}
	&-\frac{1}{2}\lambda\|(\bm{\beta}^{lasso}-\bm{\beta}^{true})\|_1+\frac{|\mathcal{A}|\kappa}{4ns}\|\bm{\beta}^{lasso}_{\mathcal{S}}-\bm{\beta}^{true}_S\|^2_1\le \lambda\|\bm{\beta}^{true}-\bm{\beta}^{lasso}\|_1\notag\\
	\Rightarrow&\frac{|\mathcal{A}|\kappa}{4ns}\|\bm{\beta}^{lasso}_S-\bm{\beta}^{true}_S\|^2_1\le \frac{3}{2}\lambda\|\bm{\beta}^{true}-\bm{\beta}^{lasso}\|_1\notag\\
	\Rightarrow&\frac{|\mathcal{A}|\kappa}{4ns}\|\bm{\beta}^{lasso}_S-\bm{\beta}^{true}_S\|^2_1\le 6\lambda\|\bm{\beta}^{true}_{\mathcal{S}}-\bm{\beta}^{lasso}_{\mathcal{S}}\|_1\label{eq:lemma:2.10.8}\\
	\Rightarrow&\|\bm{\beta}^{lasso}_S-\bm{\beta}^{true}_S\|_1\le \frac{24ns}{|\mathcal{A}|\kappa}\lambda\notag\\
	\Rightarrow&\|\bm{\beta}^{lasso}-\bm{\beta}^{true}\|_1\le \frac{96ns}{|\mathcal{A}|\kappa}\lambda\label{eq:lemma:2.10_9},
	\end{align}
	where \eqref{eq:lemma:2.10.8} and \eqref{eq:lemma:2.10_9} use $\|\bm{\beta}^{lasso}_{S^c}-\bm{\beta}^{true}_{S^c}\|_1\le3\|\bm{\beta}^{lasso}_S-\bm{\beta}^{true}_S\|_1$ in \eqref{eq:lemma:2.10_5}.

	Now, we assess the probability of event $\mathcal{E}_0$. The $i$-th element of $\nabla\mathcal{L}(\bm{\beta}^{ture})$ is $\frac{1}{n}\sum_{i=1}^nx_{ji}f^{'}(r_i|\bm{x}_j^{T}\bm{\beta}^{true})$. Denote $X_{ji}= x_{ji}f^{'}(r_i|\bm{x}_j^{T}\bm{\beta}^{true})$ for $j=1,2,...n$. Under assumptions {\bf A.1} and {\bf A.5} 
	, $X_{ji}$ are $x_{\max}\sigma$-subguassian random variables with mean $0$. We can use Hoeffding inequality to build the following probability bound.
	\begin{align}
	&\mathbbm{P}(|\frac{1}{n}\sum_{i=1}^nx_{ji}f^{'}(r_i|\bm{x}_j^{T}\bm{\beta}^{true})
	|\ge  t)\le \exp\left(-\frac{nt^2}{2\sigma^2x^{2}_{\max}}\right)\notag\\
\Rightarrow	&\mathbbm{P}\left(\max_{j}|\frac{1}{n}\sum_{i=1}^nx_{ji}f^{'}(r_i|\bm{x}_j^{T}\bm{\beta}^{true})
	|\le t\right)\ge 1-\sum_{j=1}^p\mathbbm{P}\left(|\frac{1}{n}\sum_{i=1}^nx_{ji}f^{'}(r_i|\bm{x}_j^{T}\bm{\beta}^{true})
	|\ge t\right)\notag\\
	&\ge 1- d\exp\left(-\frac{nt^2}{2\sigma^2x^{2}_{\max}}\right)\label{eq:lemma:2.10_111}
	\end{align}
	Set $t=\frac{1}{2}\lambda$, and we will have event $\mathcal{E}_0$ defined in \eqref{eq:event_0} holds with at least probability $1-\exp(-\frac{n\lambda^2}{8x_{\max}^2}+\log d)$. 
	The desirable result follows by \eqref{eq:lemma:2.10_6} and \eqref{eq:lemma:2.10_111}.

\end{proof}

\vspace{15mm}
		\begin{lemma}\label{lemma:2.13}
			Let $t_0=2C_0|\mathcal{K}|$, $C_0=\max\{10,16/p^*\}$, and $T\ge\max\{(t_0+1)^2/e^2-1,e\}$. Under assumptions {\bf A. 3} and {\bf A. 4}, the following statements hold true:
			\begin{enumerate}
				\item $\mathbbm{P}\left\{n<\frac{1}{2}C_0(T+1)\textrm{ or } n> 6C_0\log(T+1)\right\}\le \frac{2}{T+1}$
				\item $\mathbbm{P}\left\{|\mathcal{A}|<\frac{1}{4}p^*C_0\log(T+1)\right\}\le \frac{1}{T+1}$
				\item $\mathbbm{P}\left\{|\mathcal{A}|/n< \frac{1}{24}p^*\right\}\le\frac{3}{T+1}$
			\end{enumerate}
		\end{lemma}

\begin{proof} {Proof of ~\ref{lemma:2.13}}
	From Proposition \ref{MW_Sampling}, 
	we have
	\begin{align}
	&\mathbbm{P}\left(C_0(1+\log(T+1)-\log(t_0))\le n\le 3C_0(1+\log(T)-\log(t_0))\right)\ge 1-\frac{2}{T+1}\label{eq:lemma:2.13_0}.
	\end{align}
	In assumption {\bf A. 4}, we only assume that for $\bm{x}\in U_k,\ k\in\mathcal{K}$, RE condition will be held. And under Assumption {\bf A. 3}, we have $	\mathbbm{P}(\bm{x}\in U_k)\ge p^*$.
	Thus, among all $n$ samples, the expected number of samples belong to $U_k$ will be lower bounded by:
	\begin{align}
	\mathbbm{E}[\mathbbm{1}(\bm{x}\in U_k)]\ge p^*C_0(1+\log(T+1)-\log(t_0+1)).\label{eq:lemma:2.13_1}
	\end{align}
	Since $T>(t_0+1)^2/e^2-1$ implies $\frac{1}{2}\log(T+1)>\log(t_0+1)-1$. \eqref{eq:lemma:2.13_1} can be simplified into the following inequality:
	\begin{align}
	\mathbbm{E}[\sum_{i=1}^n\mathbbm{1}(x_i\in U_k)]&\ge p^*C_0(1+\log(T+1)-\log(t_0+1))\ge \frac{1}{2}p^*C_0\log(T+1).\label{eq:lemma:2.13_2}
	\end{align}
	We apply the Chernoff inequality on $\sum_{i=1}^n\mathbbm{1}(x_i\in U)$:
	\begin{align}
	\mathbbm{P}\left(\sum_{i=1}^n\mathbbm{1}(\bm{x}_i\in U_k)<\frac{1}{2}\mathbbm{E}[\sum_{i=1}^n\mathbbm{1}(\bm{x}_i\in U_k)]\right)\le\exp\left(-\frac{1}{8}\mathbbm{E}[\sum_{i=1}^n\mathbbm{1}(\bm{x}_i\in U_k)]\right)\notag\\
	\Rightarrow\mathbbm{P}\left(\sum_{i=1}^n\mathbbm{1}(\bm{x}_i\in U_k)<\frac{1}{4}p^*C_0\log(T+1)\right)\le \exp\left(-\frac{1}{16}p^*C_0\log(T+1)\right),\label{eq:lemma:2.13_3}
	\end{align}
	where \eqref{eq:lemma:2.13_3} uses \eqref{eq:lemma:2.13_2}.\\
	As we have $T\ge e$, the following result holds:
	\begin{align}
	3C_0(1+\log(T)-\log(t_0))&\le 3C_0(\log(T)+\log(T)-0)\notag\\
	&\le 6C_0\log(T)\notag\\
	&<6C_0\log(T+1)\label{eq:lemma:2.13_5}
	\end{align}
	From $T\ge (t_0+1)^2/e^2+1\Rightarrow \frac{1}{2}\log(T+1)-\log(t_0+1)\ge -1$, we have:
	\begin{align}
	C_0(1+\log(T+1)-\log(t_0+1))&= C_0(1+\frac{1}{2}\log(T+1)+\frac{1}{2}\log(T+1)-\log(t_0+1))\notag\\
	&\ge C_0(1+\frac{1}{2}\log(T+1)-1)\notag\\
	&=\frac{1}{2}C_0\log(T+1)\label{eq:lemma:2.13_6}.
	\end{align}
	Further, from \eqref{eq:lemma:2.13_5} and \eqref{eq:lemma:2.13_6}, we have
	\begin{align}
	&\left\{n<\frac{1}{2}C_0\log(T+1)\textrm{ or } n>6C_0\log(T+1)\right\}\notag\\
	=&\left\{n<\frac{1}{2}C_0\log(T+1)\right\}\cup\left\{n>6C_0\log(T+1)\right\}\notag\\
	\subseteq&\left\{n<C_0(1+\log(T+1)-\log(t_0+1))\right\}\cup\left\{ n>3C_0(1+\log(T)-\log(t_0))\right\}\notag\\
	=&\left\{n<C_0(1+\log(T+1)-\log(t_0+1))\textrm{ or } n>3C_0(1+\log(T)-\log(t_0))\right\}\label{eq:lemma:2.13_7}.
	\end{align}
	The following inequality can be obtained by combining \eqref{eq:lemma:2.13_7} and \eqref{eq:lemma:2.13_0}:
	\begin{align}
	&\mathbbm{P}\left\{n<\frac{1}{2}C_0\log(T+1)\textrm{ or } n>6C_0\log(T+1)\right\}\notag\\
	\le& \left\{n<C_0(1+\log(T+1)-\log(t_0+1))\textrm{ or } n>3C_0(1+\log(T)-\log(t_0))\right\}\notag\\
	\le& \frac{2}{T+1}\label{eq:lemma:2.13_8}
	\end{align}

	The second part of Lemma \ref{lemma:2.13} can be proved by \eqref{eq:lemma:2.13_3} with $C_0\ge 16/p^*$:
	\begin{align}
	\mathbbm{P}\left(|\mathcal{A}| = \sum_{i=1}^n\mathbbm{1}(\bm{x}_i\in U_k)<\frac{1}{4}p^*C_0\log(T+1)\right)&\le \exp\left(-\frac{1}{16}p^*C_0\log(T+1)\right)\notag\\
	&\le \exp\left(-\log(T+1)\right)=\frac{1}{T+1}.\label{eq:lemma:2.13_9}
	\end{align}

	Finally, we will show part 3. Notice that the follow result hold:
	\begin{align}
	\left\{|\mathcal{A}|/n\ge  \frac{1}{24}p^*\right\}&\supseteq\left\{|\mathcal{A}|\ge \frac{1}{4}C_0p^*\log(T+1)\right\}\cap\left\{n\le 6C_0\log(T+1)\right\}\notag\\
	&= \left(\left\{|\mathcal{A}|< \frac{1}{4}C_0p^*\log(T+1)\right\}\cup\left\{n> 6C_0\log(T+1)\right\}\right)^c.
	\end{align}
	Hence we can obtain
	\begin{align}
	\mathbbm{P}\left\{|\mathcal{A}|/n\ge  \frac{1}{24}p^*\right\}&\ge \mathbbm{P}\left\{\left(\left\{|\mathcal{A}|< \frac{1}{4}C_0p^*\log(T+1)\right\}\cup\left\{n> 6C_0\log(T+1)\right\}\right)^c\right\}\notag\\
	&=1-\mathbbm{P}\left\{\left\{|\mathcal{A}|< \frac{1}{4}C_0p^*\log(T+1)\right\}\cup\left\{n> 6C_0\log(T+1)\right\}\right\}\notag\\
	&=1-\mathbbm{P}\left\{|\mathcal{A}|< \frac{1}{4}C_0p^*\log(T+1)\right\}-\mathbbm{P}\left\{n> 6C_0\log(T+1)\right\}.\label{eq:lemma:2.13_10}
	\end{align}
	Combining \eqref{eq:lemma:2.13_7}, \eqref{eq:lemma:2.13_8}, and \eqref{eq:lemma:2.13_10}, we the following result:
	\begin{align}
	\mathbbm{P}\left\{|\mathcal{A}|/n<  \frac{1}{24}p^*\right\}&\le \mathbbm{P}\left\{|\mathcal{A}|< \frac{1}{4}C_0p^*\log(T+1)\right\}+\mathbbm{P}\left\{n> 6C_0\log(T+1)\right\}\notag\\
	&=      \frac{2}{T+1}+ \exp\left(-\frac{1}{16}p^*C_0\log(T+1)\right) = \frac{3}{T+1}.
	\end{align}

\end{proof}


\vspace{15mm}
\begin{lemma}\label{corollary:2.14}
	Let $t_0 = 2C_0|\mathcal{K}|$, $T\ge\max\{(t_0+1)^2/e^2-1,e\}$, $\lambda = C_5\sqrt{1+\frac{\log d}{\log(T+1)}}$, and  $a>\frac{2304s}{p^*\kappa}$. If assumptions {\bf A.1},{\bf A.3},{\bf A.4} and {\bf A.5} hold, we have:
	\begin{align}
	\mathbbm{P}\left(\|\bm{\beta}^{oracle}-\bm{\beta}^{true}\|_1\le  \min\left\{\frac{1}{\sigma_2x_{\max}},\frac{h}{4e\sigma_2R_{\max}x_{\max}}\right\}\right)\ge 1-\frac{15}{T+1},\label{eq:corollary:2.14_0}
	\end{align}
	where  $$       C_0 = \max\left\{10,\frac{16}{p^*},\frac{16}{p^*C_1}, \frac{x_{\max}^2}{3C_5^2}, \frac{x_{\max}^2}{3C_5^2\left((\frac{1}{4}-\frac{576s}{p^*\kappa a})\min\left\{1,\frac{\mu_0p^*}{192sx_{\max}^2}\right\}\right)^2}, \frac{64s\sigma_2x_{\max}^2\log s}{p^*\mu_0},\frac{\sigma^2x_{\max}^2\log s}{3t^2}\right\},$$
	$t\le \min\left\{\frac{\mu_0p^*\sqrt{\tilde{C}_2\lambda}}{48},\frac{p^*\mu_0}{48\sigma_2\sqrt{s}x_{\max}},\frac{hp^*\mu_0}{192e\sigma_2\sqrt{s}R_{\max}x_{\max}}\right\}$, $\tilde{C}_2 = \frac{\mu_0p^*}{2\sigma_2sx_{\max}^3(\mu_0p^*+48sx_{\max}^2)}$ and $C_5 = \frac{\bm{\beta}_{\min}p^*\kappa}{(2304s+ap^*\kappa)\sqrt{1+\log d}} $
\end{lemma}

\begin{proof} {Proof of ~ Lemma\ \ref{corollary:2.14}}
	Using Lemma \ref{lemma:2.13}, $t_0=2C_0|\mathcal{K}|$, $T\ge\max\{(t_0+1)^2/e^2-1,e\}$, and $C_0\ge \{10,16/p^*\}$, we have:
	\begin{align}
	\mathbbm{P}\left\{n\ge \frac{1}{2}C_0\log(T+1)\right\}&\le 1-\frac{2}{T+1}\label{eq:corollary:2.14_1}\\
	\mathbbm{P}\left\{|\mathcal{A}|\ge \frac{1}{4}p^*C_0\log(T+1)\right\}&\ge 1-\frac{1}{T+1}\label{eq:corollary:2.14_2}\\
	\mathbbm{P}\left\{\frac{|\mathcal{A}|}{n}\ge \frac{1}{24}p^*\right\}&\ge \frac{3}{T+1}\label{eq:corollary:2.14_3}.
	\end{align}
	Thus, with probability $1-\frac{3}{T+1}$, we have:
	\begin{align}
	\beta_{\min}&=(\frac{2304s}{p^*\kappa}+a)C_5\sqrt{1+\log d}\ge (\frac{2304s}{p^*\kappa}+a)\lambda \ge (\frac{96ns}{\kappa|\mathcal{A}|}+a)\lambda\notag\\
	a&>\frac{2304s}{p^*\kappa}\ge\frac{96ns}{\kappa|\mathcal{A}|}\notag\\
	\tilde{C}_2&= \frac{\mu_0p^*}{2\sigma_2sx_{\max}^3(\mu_0p^*+48sx_{\max}^2)}\le \frac{\mu_0|\mathcal{A}|}{2\sigma_2sx_{\max}^3(\mu_0|\mathcal{A}|+n2sx_{\max}^2)}=C_2
	\end{align}
	If we require $t\le\frac{\mu_0|\mathcal{A}|\sqrt{\tilde{C_2}\lambda}}{2n}\le\frac{\mu_0|\mathcal{A}|\sqrt{C_2\lambda}}{2n}$, then from the first part in Lemma \ref{lemma:2.11} 
	, we can obtain the following inequality.
	\begin{align}
	\mathbbm{P}\left(\|\bm{\beta}^{oracle}-\bm{\beta}^{true}\|_2\ge \frac{2nt}{|\mathcal{A}|\mu_0}\right)\le \delta_1+\delta_2+\delta_3+\delta_4+\delta_5,
	\end{align}
	where
	\begin{align}
	\delta_1 &= \exp\left(-C_1|\mathcal{A}|\right)\le  \exp\left(-\frac{1}{16}p^*C_0C_1\log(T+1)\right)\\
	\delta_2 &= \exp\left(-\frac{n\lambda^2}{2x_{\max}^2}+\log d\right)\le\exp\left(-\frac{6C_0\log(T+1)\lambda^2}{2x_{\max}^2}+\log d\right) \\
	\delta_3 &= \exp\left(-\frac{n\lambda^2\left((\frac{1}{4}-\frac{24ns}{|\mathcal{A}|\kappa a})\min\left\{1,\frac{\mu_0|\mathcal{A}|}{8snx_{\max}^2}\right\}\right)^2}{2x_{\max}^2}+\log d\right)\notag\\
	&\le \exp\left(-\frac{6C_0\log(T+1)\lambda^2\left((\frac{1}{4}-\frac{576s}{p^*\kappa a})\min\left\{1,\frac{\mu_0p^*}{192sx_{\max}^2}\right\}\right)^2}{2x_{\max}^2}+\log d\right)\\
	\delta_4 &= 2s\exp\left(-\frac{|\mathcal{A}|\mu_0}{4sLx_{\max}^2}\right)\le 2\exp\left(-\frac{\frac{1}{16}p^*C_0\log(T+1)\mu_0}{4sLx_{\max}^2}+\log s\right)\\
	\delta_5 &= \exp\left(-\frac{nt^2}{2\sigma^2x_{\max}^2}\right)\le \exp\left(-\frac{6C_0\log(T+1)t^2}{2\sigma^2x_{\max}^2}+\log s\right).
	\end{align}
	If we require $      C_0 = \max\left\{\frac{16}{p^*C_1}, \frac{x_{\max}^2}{3C_5^2}, \frac{x_{\max}^2}{3C_5^2\left((\frac{1}{4}-\frac{576s}{p^*\kappa a})\min\left\{1,\frac{\mu_0p^*}{192sx_{\max}^2}\right\}\right)^2}, \frac{64sLx_{\max}^2\log s}{p^*\mu_0},\frac{\sigma^2x_{\max}^2\log s}{3t^2}\right\}$ and $\lambda = C_5\sqrt{1+\log d/\log(T+1)}$, then we can verify the following inequalities hold:
	\begin{align}
	\delta_1&\le \exp\left(-\log(T+1)\right)=\frac{1}{T+1}\label{eq:corollary:2.14_4}\\
	\delta_2&\le \exp\left(-\log(T+1)(1+\log d/\log(T+1))+\log d\right)\le \frac{1}{T+1}\label{eq:corollary:2.14_5}\\
	\delta_3&\le \exp\left(-\log(T+1)(1+\log d/\log(T+1))+\log d\right)\le \frac{1}{T+1}\label{eq:corollary:2.14_6}\\
	\delta_4&\le 2\exp\left(-\log s(\log(T+1)-\log e)\right)\le 2\exp\left(-\log(\frac{T+1}{e})\right)\le \frac{6}{T+1}\label{eq:corollary:2.14_7}\\
	\delta_5&\le \exp\left(-\log s(\log(T+1)-\log e)\right)\le \exp\left(-\log(\frac{T+1}{e})\right)\le \frac{3}{T+1}\label{eq:corollary:2.14_8}.
	\end{align}
	Hence, we have
	\begin{align}
	\mathbbm{P}\left(\|\bm{\beta}^{MCP}-\bm{\beta}^{true}\|_2\le \frac{2nt}{|\mathcal{A}|\mu_0}\right)&\ge 1-\frac{15}{T+1}\notag\\
	\Rightarrow \mathbbm{P}\left(\|\bm{\beta}^{MCP}-\bm{\beta}^{true}\|_1\le \frac{2nt\sqrt{s}}{|\mathcal{A}|\mu_0}\right)&\ge 1-\frac{15}{T+1},\label{eq:corollary:2.14_8.5}
	\end{align}
	where \eqref{eq:corollary:2.14_8.5} uses $\bm{\beta}^{MCP}$ being the oracle solution with $\bm{\beta}_{S^c}^{MCP} =\bm{\beta}_{S^c}^{true}= \bm{0}$.
	Moreover, combine $t\le \min\left\{\frac{p^*\mu_0}{48\sigma_2\sqrt{s}x_{\max}},\frac{hp^*\mu_0}{192e\sigma_2\sqrt{s}R_{\max}x_{\max}}\right\}$, \eqref{eq:corollary:2.14_3} and we have the following results with probability $1-\frac{3}{T+1}$:
	\begin{align}
	\frac{2nt\sqrt{s}}{|\mathcal{A}|\mu_0}\le \frac{2nhp^*\mu_0\sqrt{s}}{192e\sigma_2\sqrt{s}R_{\max}x_{\max}|\mathcal{A}|\mu_0}=\frac{h}{4e\sigma_2R_{\max}x_{\max}}\cdot\frac{n}{|\mathcal{A}|}\cdot\frac{p^*}{24}\le \frac{h}{4e\sigma_2R_{\max}x_{\max}}\label{eq:corollary:2.14_9}\\
	\frac{2nt\sqrt{s}}{|\mathcal{A}|\mu_0}\le \frac{p^*\mu_0\sqrt{s}}{48\sigma_2\sqrt{s}x_{\max}|\mathcal{A}|\mu_0}=\frac{1}{\sigma_2x_{\max}}\cdot\frac{n}{|\mathcal{A}|}\cdot\frac{p^*}{24}\le \frac{1}{\sigma_2x_{\max}}\label{eq:corollary:2.14_10}
	\end{align}
	Desirable result follows immediately by
	combining \eqref{eq:corollary:2.14_4}-\eqref{eq:corollary:2.14_8}, \eqref{eq:corollary:2.14_9}, \eqref{eq:corollary:2.14_10}, and \eqref{eq:corollary:2.14_8.5}.
\end{proof}

\vspace{15mm}

	\begin{lemma}\label{lemma:2.16}
	Under assumptions {\bf A.3} and {\bf A.5}, 
	for any $\bm{x}\in U_k, i\in\mathcal{K}$, the following two statements hold.	\begin{enumerate}
	\item		$\left|\mathbbm{E}(R_i|x,\bm{\beta}^{true}_i)-\mathbbm{E}(R_i|x,\bm{\beta}^{MCP}_i)\right|\le R_{\max}e^{\sigma_2x_{\max}\|\bm{\beta}^{MCP}_i-\bm{\beta}^{true}_i\|_1}\sigma_2x_{\max}\|\bm{\beta}^{MCP}_i-\bm{\beta}^{true}_i\|_1$
		\item Moreover, if $\|\bm{\beta}^{MCP}_k-\bm{\beta}^{true}_k\|_1\le \min\left\{\frac{1}{\sigma_2x_{\max}},\frac{h}{4e\sigma_2R_{\max}x_{\max}}\right\},\ k\in\mathcal{K}$, we have 
		$\mathbbm{E}(R_i|\bm{x},\beta^{MCP}_k)\ge\max_{j\ne i}\mathbbm{E}(R_j|\bm{x},\beta^{MCP}_k)+ \frac{h}{2}.$
	\end{enumerate}

\end{lemma}

\begin{proof} {Proof of Lemma\ ~\ref{lemma:2.16}}
	To show the first part, we first expand the left-hand-side of the first statement as follows:
	\begin{align}
	&\left|\mathbbm{E}(R_i|\bm{x},\bm{\beta}^{true}_i)-\mathbbm{E}(R_i|\bm{x},\bm{\beta}^{MCP}_i)\right|\notag\\
	=&\left|\int_{-\infty}^{+\infty}r_k dF(r_k|\bm{x}^T\bm{\beta}_i^{true})-\int_{-\infty}^{+\infty}r_k dF(r_k|\bm{x}^T\bm{\beta}_i^{MCP})\right|\notag\\
	=&\left|\int_{-\infty}^{+\infty}r_k e^{-f(r_k|\bm{x}^T\bm{\beta}_i^{true})} dr_k-\int_{-\infty}^{+\infty}r_k e^{-f(r_k|\bm{x}^T\bm{\beta}_i^{MCP})} dr_k\right|\label{eq:whole_sample_2}\\
	=&\left|\int_{-\infty}^{+\infty}r_k \left(e^{-f(r_k|\bm{x}^T\bm{\beta}_i^{true})}- e^{-f(r_k|\bm{x}^T\bm{\beta}_i^{MCP})}\right) dr_k\right|.\notag\\
	=&\left|\int_{-\infty}^{+\infty}-r_k \left.\left(e^{-f(r_k|\bm{x}^T\bm{\beta}_i)}\right)^{'}\right|_{\bm{\beta}=\bm{\beta}_i^{true}+\bm{\delta}}\bm{x}^T(\bm{\beta}^{MCP}_i-\bm{\beta}^{true}_i) dr_k\right|,\label{eq:whole_sample_3}
	\end{align}
	where \eqref{eq:whole_sample_2} uses $f$ being the negative log density function and $\bm{\delta}$ is between $\bm{0}$ and $\bm{\beta}_i^{MCP}-\bm{\beta}_i^{true}$. Since term $\bm{x}^T(\bm{\beta}^{MCP}-\bm{\beta}^{true})$ is independent on $r_k$, we can pull it out:
	\begin{align}
	&\left|\int_{-\infty}^{+\infty}-r_k \left.\left(e^{-f(r_k|\bm{x}^T\bm{\beta}_i)}\right)^{'}\right|_{\bm{\beta}=\bm{\beta}_i^{true}+\delta}\bm{x}^T(\bm{\beta}^{MCP}-\bm{\beta}^{true}) dr_k\right|\notag\\
	=&\left|\bm{x}^T(\bm{\beta}^{MCP}-\bm{\beta}^{true})\int_{-\infty}^{+\infty}-r_k \left.\left(e^{-f(r_k|\bm{x}^T\bm{\beta}_i)}\right)^{'}\right|_{\bm{\beta}=\bm{\beta}_i^{true}+\bm{\delta}} dr_k\right|\notag\\
	\le &\left|\int_{-\infty}^{+\infty}r_k e^{-f(r_k|\bm{x}^T(\bm{\beta}_i^{true}+\bm{\delta}))} f^{'}(r_k|\bm{x}^T(\bm{\beta}_i^{true}+\bm{\delta}))dr_k\right|x_{\max}\|\bm{\beta}^{MCP}_i-\bm{\beta}^{true}_i\|_1.\label{eq:whole_sample_3.1}
	\end{align}
	As we assume $|f^{'}(\cdot|\cdot)|$ is bounded by $\sigma_2$ in assumption {\bf A.5}, \eqref{eq:whole_sample_3.1} is upper bounded by:
	\begin{align}
	&\left|\int_{-\infty}^{+\infty}r_k e^{-f(r_k|\bm{x}^T(\bm{\beta}_i^{true}+\bm{\delta}))} f^{'}(r_k|\bm{x}^T(\bm{\beta}_i^{true}+\bm{\delta}))dr_k\right|x_{\max}\|\bm{\beta}^{MCP}_i-\bm{\beta}^{true}_i\|_1\notag\\
	\le &\left|\int_{-\infty}^{+\infty}r_k e^{-f(r_k|\bm{x}^T(\bm{\beta}_i^{true}+\bm{\delta}))} dr_k\right|\sigma_2x_{\max}\|\bm{\beta}^{MCP}_i-\bm{\beta}^{true}_i\|_1\label{eq:whole_sample_3.2}.
	\end{align}
	We then expand term $f(r_k|\bm{x}^T(\bm{\beta}_i^{true}+\bm{\delta}))$ in \eqref{eq:whole_sample_3.2}, and there exists a $\bm{\xi}$ between $\bm{0}$ and $\bm{\beta}^{true}+\bm{\delta}$ such that:
	\begin{align}
	&\left|\int_{-\infty}^{+\infty}r_k e^{-f(r_k|\bm{x}^T(\bm{\beta}_i^{true}+\bm{\delta}))} dr_k\right|\sigma_2x_{\max}\|\bm{\beta}_i^{MCP}-\bm{\beta}_i^{true}\|_1\notag\\
	= &\left|\int_{-\infty}^{+\infty}r_k e^{-f(r_k|\bm{x}^T\bm{\beta}_i^{true})-f^{'}(r_k|\bm{x}^T\bm{\xi})\bm{x}^T\bm{\delta}} dr_k\right|\sigma_2x_{\max}\|\bm{\beta}^{MCP}_i-\bm{\beta}_j^{true}\|_1\notag\\
	\le &\left|\int_{-\infty}^{+\infty}r_k e^{-f(r_k|\bm{x}^T\bm{\beta}_i^{true})+|f^{'}(r_k|\bm{x}^T\bm{\xi})|\|\bm{x}\|_{\infty}\|\bm{\delta}\|_1} dr_k\right|\sigma_2x_{\max}\|\bm{\beta}^{MCP}_i-\bm{\beta}^{true}_i\|_1\notag\\
	\le &\left|\int_{-\infty}^{+\infty}r_k e^{-f(r_k|\bm{x}^T\bm{\beta}_i^{true})} dr_k\right|e^{\sigma_2x_{\max}\|\bm{\beta}^{MCP}_i-\bm{\beta}^{true}_i\|_1}\sigma_2x_{\max}\|\bm{\beta}^{MCP}_i-\bm{\beta}^{true}_i\|_1\label{eq:whole_sample_3.3}\\
	= &|\mathbbm{E}(R_k|x,\bm{\beta}^{true}_k)|e^{\sigma_2x_{\max}\|\bm{\beta}_i^{MCP}-\bm{\beta}_i^{true}\|_1}\sigma_2x_{\max}\|\bm{\beta}_i^{MCP}-\bm{\beta}_i^{true}\|_1\label{eq:whole_sample_3.4}
	\end{align}
	where \eqref{eq:whole_sample_3.3} uses that $\bm{\delta}$ is between $\bm{0}$ and $\bm{\beta}^{MCP}_i-\bm{\beta}_i^{true}$, which implies $\|\bm{\delta}\|_1\le \|\bm{\beta}^{MCP}_i-\bm{\beta}_i^{true}\|_1$, and \eqref{eq:whole_sample_3.4} comes from the definition of $\mathbbm{E}(R_k|\bm{x},\bm{\beta}^{true}_k)$. Combining $|r_k|\le R_{\max}$, \eqref{eq:whole_sample_3.4}, and \eqref{eq:whole_sample_3}, we have:
	\begin{align}
	\left|\mathbbm{E}(R_i|x,\bm{\beta}^{true}_i)-\mathbbm{E}(R_i|x,\bm{\beta}^{MCP}_i)\right|\le R_{\max}e^{\sigma_2x_{\max}\|\bm{\beta}^{MCP}_i-\bm{\beta}^{true}_i\|_1}\sigma_2x_{\max}\|\bm{\beta}^{MCP}_i-\bm{\beta}^{true}_i\|_1.\label{eq:whole_sample_3.11}
	\end{align}
	To show the second part, note that the assumption $\|\bm{\beta}^{MCP}_k-\bm{\beta}^{true}_k\|_1\le\frac{1}{\sigma_2x_{\max}},\ k\in\mathcal{K}$ implies the following inequality:
	\begin{align}
	\|\bm{\beta}^{MCP}_i-\bm{\beta}^{true}_i\|_1\le \frac{1}{\sigma_2x_{\max}}&\Rightarrow e^{\sigma_2x_{\max}\|\bm{\beta}^{MCP}_i-\bm{\beta}_i^{true}\|_1}\le e\label{eq:whole_sample_3.12}
	\end{align}
	Combining \eqref{eq:whole_sample_3.12} and \eqref{eq:whole_sample_3.11}, we obtain
	\begin{align}
	\left|\mathbbm{E}(r_i|x,\bm{\beta}^{true}_i)-\mathbbm{E}(r_i|x,\bm{\beta}^{MCP}_i)\right|\le &R_{\max}e^{\sigma_2x_{\max}\|\bm{\beta}^{MCP}_i-\bm{\beta}_i^{true}\|_1}\sigma_2x_{\max}\|\bm{\beta}_i^{MCP}-\bm{\beta}_i^{true}\|_1\notag\\
	\le &R_{\max}e\sigma_2x_{\max}\|\bm{\beta}^{MCP}_i-\bm{\beta}_i^{true}\|_1\label{eq:whole_sample_5}	
	\end{align}
	Under assumption {\bf A.3}, 
	for any $x\in U_k$, the following inequalities hold:
	\begin{align}
	\mathbbm{E}(R_i|\bm{x},\bm{\beta}^{true}_i)&\ge \max_{j\ne i} \mathbbm{E}(R_j|\bm{x},\bm{\beta}^{true}_j)+h\notag\\
	\Rightarrow\mathbbm{E}(r_i|\bm{x},\bm{\beta}^{true}_i)-\mathbbm{E}(r_i|\bm{x},\bm{\beta}^{MCP}_i))&\ge\max_{j\ne i}\left[\mathbbm{E}(r_j|x,\bm{\beta}^{true}_j)-\mathbbm{E}(r_j|\bm{x},\bm{\beta}^{MCP}_j)\right]\notag\\
	+&\max_{j\ne i}\mathbbm{E}(r_j|\bm{x},\bm{\beta}^{MCP}_j)-\mathbbm{E}(r_i|\bm{x},\bm{\beta}^{MCP}_i)+h\notag\\
	\Rightarrow \mathbbm{E}(r_i|\bm{x},\bm{\beta}^{MCP}_i)-\max_{j\ne i}\mathbbm{E}(r_j|\bm{x},\bm{\beta}^{MCP}_j)&\ge-\left|\mathbbm{E}(r_i|\bm{x},\bm{\beta}^{MCP}_i)-\mathbbm{E}(r_i|x,\bm{\beta}^{true}_i))\right|\notag\\
	-&\max_{j\ne i}\left|\mathbbm{E}(r_j|\bm{x},\bm{\beta}^{true}_j)-\mathbbm{E}(r_j|\bm{x},\bm{\beta}^{MCP}_j)\right|+h.\label{eq:whole_sample_2.4}
	\end{align}
As we assume $\|\bm{\beta}^{MCP}_k-\bm{\beta}^{true}_k\|_1\le \frac{h}{4e\sigma_2R_{\max}x_{\max}},\ k\in\mathcal{K}$, we have
\begin{align}
	\|\bm{\beta}^{MCP}_i-\bm{\beta}_i^{true}\|_1\le  \frac{h}{4e\sigma_2R_{\max}x_{\max}}&\Rightarrow\|R_{\max}e\sigma_2x_{\max}(\bm{\beta}^{MCP}_i-\bm{\beta}_i^{true})\|_1\le  \frac{h}{4}\label{eq:whole_sample_3.13}
\end{align}
	Combining \eqref{eq:whole_sample_5},\eqref{eq:whole_sample_3.13} and \eqref{eq:whole_sample_2.4}, we will have
	\begin{align}
	&\mathbbm{E}(r_i|\bm{x},\bm{\beta}^{MCP}_i)-\max_{j\ne i}\mathbbm{E}(r_j|\bm{x},\bm{\beta}^{MCP}_j)\ge-\frac{h}{4}-\frac{h}{4}+h\notag\\
	\Rightarrow&\mathbbm{E}(r_i|\bm{x},\bm{\beta}^{MCP}_i)\ge \max_{j\ne i}\mathbbm{E}(r_j|\bm{x},\bm{\beta}^{MCP}_j)+\frac{h}{2}.\label{eq:whole_sample_6}
	\end{align}
\end{proof}

\vspace{15mm}

\begin{lemma}\label{lemma:2.17}
Denote events $\mathcal{E}_3,\mathcal{E}_4$, and $\mathcal{E}_5$ as follows: \begin{align}
	\mathcal{E}_3&=\left\{\|\nabla_{\mathcal{S}^c}\mathcal{L}(\bm{\beta}^{true})\|_{\infty}\le \left(1-\frac{96ns}{|\mathcal{A}|\kappa a}\right)\frac{\lambda}{4}\label{eq:event_3}\right\},\\
	\mathcal{E}_4&=\left\{\|\nabla_{\mathcal{S}}\mathcal{L}(\bm{\beta}^{true})\|_{\infty}\le \left(1-\frac{96ns}{|\mathcal{A}|\kappa a}\right)\frac{\mu_0|\mathcal{A}|\lambda}{8snx_{\max}^2}\right\}\label{eq:event_4},\\
	\mathcal{E}_5&=\left\{\|\bm{\beta}^{oracle}-\bm{\beta}^{true}\|_2\le\sqrt{C_2\lambda} \right\}\label{eq:event_5},
	\end{align}
	where $C_2 \doteq \frac{\mu_0|\mathcal{A}|}{2\sigma_2sx_{\max}^3(\mu_0|\mathcal{A}|+2snx_{\max}^2)}$.
Under assumption {\bf A.1} and {\bf A.5}, events $\mathcal{E}_3,\mathcal{E}_4$ and $\mathcal{E}_5$ implies $\mathcal{E}_2$ defined in \eqref{eq:event_2}.
\end{lemma}

\begin{proof} {Proof of Lemma\ ~\ref{lemma:2.17}} 
We first expend $\nabla \mathcal{L}(\bm{\beta}^{oracle})$ at $\bm{\beta}^{true}$:
	\begin{align}
	\nabla \mathcal{L}(\bm{\beta}^{oracle}) &= \nabla \mathcal{L}(\bm{\beta}^{true})+\nabla^2\mathcal{L}(\bm{\xi})(\bm{\beta}^{oracle}-\bm{\beta}^{true})\label{eq:lemma:2.12_9.5}\\
	&= \nabla \mathcal{L}(\bm{\beta}^{true}) +\nabla^2\mathcal{L}(\bm{\beta}^{true})(\bm{\beta}^{oracle}-\bm{\beta}^{true})+(\nabla^2\mathcal{L}(\bm{\xi})-\nabla^2\mathcal{L}(\bm{\beta}^{true}))(\bm{\beta}^{oracle}-\bm{\beta}^{true}),\label{eq:lemma:2.12_10}
	\end{align}
	where $\bm{\xi}= \tau\bm{\beta}^{true}+(1-\tau)\bm{\beta}^{oracle}$, $\tau\in [0,1]$. The last term in \eqref{eq:lemma:2.12_10} can be further expanded as follows:
	\begin{align}
	&(\nabla^2\mathcal{L}(\bm{\xi})-\nabla^2\mathcal{L}(\bm{\beta}^{true}))(\bm{\beta}^{oracle}-\bm{\beta}^{true})\notag\\
	=&\frac{1}{n}\sum_{j=1}^n\left[f^{''}(r_j|\bm{x}_j^T\bm{\xi})-f^{''}(r_j|\bm{x}_j^T\bm{\beta}^{true})\right]\bm{x}_j\bm{x}_j^T(\bm{\beta}^{oracle}-\bm{\beta}^{true})\notag\\
	=&\frac{1}{n}\sum_{j=1}^n\left[
	-f^{'''}(r_j|\bm{x}_j^T\eta)\bm{x}_j^T(\bm{\xi}-\bm{\beta}^{true})\right]\bm{x}_j\bm{x}_j^T(\bm{\beta}^{oracle}-\bm{\beta}^{true}),\label{eq:lemma:2.12_11}
	\end{align}
	where \eqref{eq:lemma:2.12_11} comes from the mean value theorem and the fact that $\eta$ is on the line of of $\bm{\xi}$ and $\bm{\beta}^{true}$.
	Hence, assumption {\bf A.5} and \eqref{eq:lemma:2.12_11} imply
	\begin{align}
	&\|(\nabla^2\mathcal{L}(\bm{\xi})-\nabla^2\mathcal{L}(\bm{\beta}^{true}))(\bm{\beta}^{oracle}-\bm{\beta}^{true})\|_{\infty}\notag\\
	&=\left\| \frac{1}{n}\sum_{j=1}^n\left[
	-f^{'''}(r_j|\bm{x}_j^T\eta)\bm{x}_j^T(\bm{\xi}-\bm{\beta}^{true})\right]\bm{x}_j\bm{x}_j^T(\bm{\beta}^{oracle}-\bm{\beta}^{true})\right\|_{\infty}\notag\\
	&\le \left\| \frac{1}{n}\sum_{j=1}^n
	\sigma_3x_{\max}(\bm{\xi}-\bm{\beta}^{true})\bm{x}_j\bm{x}_j^T(\bm{\beta}^{oracle}-\bm{\beta}^{true})\right\|_{\infty}\notag\\
	&\le \left\| \frac{1}{n}\sum_{j=1}^n
	\sigma_3x_{\max}(\bm{\beta}^{oracle}-\bm{\beta}^{true})^T\bm{x}_j\bm{x}_j^T(\bm{\beta}^{oracle}-\bm{\beta}^{true})\right\|_{\infty}\notag\\
	&\le \sigma_3x_{\max}\lambda_{\max}(\frac{1}{n}X_{\mathcal{S}}X_{\mathcal{S}}^T)\|\bm{\beta}^{oracle}-\bm{\beta}^{true}\|^2\notag\\
	&\le \sigma_3sx^3_{\max}\|\bm{\beta}^{oracle}-\bm{\beta}^{true}\|^2\label{eq:lemma:2.12_12}.
	\end{align}
Combining \eqref{eq:lemma:2.12_10}, \eqref{eq:lemma:2.12_12}, and the fact $\bm{\beta}^{oracle}_{\mathcal{S}^c}=\bm{\beta}^{true}_{\mathcal{S}^c} = 0$, we have
	\begin{align}
	&\|\nabla_{\mathcal{S}^c}\mathcal{L}(\bm{\beta}^{oracle})\|_{\infty}\le \|\nabla_{\mathcal{S}^c}\mathcal{L}(\bm{\beta}^{true})\|_{\infty}+\|\nabla^2_{\mathcal{S}^c,\mathcal{S}}\mathcal{L}(\bm{\beta}^{true})(\bm{\beta}^{oracle}_{\mathcal{S}}-\bm{\beta}^{true}_{\mathcal{S}})\|_{\infty}+\sigma_3sx^3_{\max}\|\bm{\beta}^{oracle}-\bm{\beta}^{true}\|^2\label{eq:lemma:2.12_14}.
	\end{align}

In addition, from $\nabla_{\mathcal{S}} \mathcal{L}(\bm{\beta}^{oracle}) = 0$ and \eqref{eq:lemma:2.12_10}, we have
	\begin{align}
	(\bm{\beta}^{oracle}_{\mathcal{S}}-\bm{\beta}^{true}_{\mathcal{S}}) = -(\nabla_{\mathcal{S},\mathcal{S}}^2\mathcal{L}(\bm{\beta}^{true}))^{-1}( \nabla_{\mathcal{S}}\mathcal{L}(\bm{\beta}^{true})+(\nabla^2_{\mathcal{S},\mathcal{S}}\mathcal{L}(\bm{\xi})-\nabla^2_{\mathcal{S},\mathcal{S}}\mathcal{L}(\bm{\beta}^{true}))(\bm{\beta}^{oracle}_{\mathcal{S}}-\bm{\beta}^{true}_{\mathcal{S}}))\label{eq:lemma:2.12_13}.
	\end{align}
	Under events $\mathcal{E}_3$, $\mathcal{E}_4$, and \eqref{eq:lemma:2.12_13}, 
	the inequality \eqref{eq:lemma:2.12_14} can be upper bounded as follows:
	\begin{align}
	&\|\nabla_{\mathcal{S}^c}\mathcal{L}(\bm{\beta}^{oracle})\|_{\infty}\le \left(1-\frac{96ns}{|\mathcal{A}|\kappa a}\right)\frac{\lambda}{4}+\sigma_3x_{\max}\lambda_{\max}(\frac{1}{n}X_{\mathcal{S}}X_{\mathcal{S}}^T)\|\bm{\beta}^{oracle}-\bm{\beta}^{true}\|^2\notag\\
	&\quad+\|\nabla^2_{\mathcal{S}^c,\mathcal{S}}\mathcal{L}(\bm{\beta}^{true})(\nabla_{\mathcal{S},\mathcal{S}}^2\mathcal{L}(\bm{\beta}^{true}))^{-1}( \nabla_{\mathcal{S}}\mathcal{L}(\bm{\beta}^{true})+(\nabla^2_{\mathcal{S},\mathcal{S}}\mathcal{L}(\bm{\xi})-\nabla^2_{\mathcal{S},\mathcal{S}}\mathcal{L}(\bm{\beta}^{true}))(\bm{\beta}^{oracle}_{\mathcal{S}}-\bm{\beta}^{true}_{\mathcal{S}}))\|_{\infty}\notag\\
	&\le \left(1-\frac{96ns}{|\mathcal{A}|\kappa a}\right)\frac{\lambda}{4}+\sigma_3sx^3_{\max}\|\bm{\beta}^{oracle}-\bm{\beta}^{true}\|^2\notag\\
	&\quad\quad+\left\|\nabla^2_{\mathcal{S}^c,\mathcal{S}}\mathcal{L}(\bm{\beta}^{true})(\nabla_{\mathcal{S},\mathcal{S}}^2\mathcal{L}(\bm{\beta}^{true}))^{-1}\right\|\left(\|\nabla_{\mathcal{S}}\mathcal{L}(\bm{\beta}^{true})\|_{\infty}+\sigma_3sx^3_{\max}\|\bm{\beta}^{oracle}-\bm{\beta}^{true}\|^2\right)\notag\\
	&\le \left(1-\frac{96ns}{|\mathcal{A}|\kappa a}\right)\frac{\lambda}{4}+\sigma_3sx^3_{\max}\|\bm{\beta}^{oracle}-\bm{\beta}^{true}\|^2\notag\\
	&\quad\quad+\left\|\nabla^2_{\mathcal{S}^c,\mathcal{S}}\mathcal{L}(\bm{\beta}^{true})(\nabla_{\mathcal{S},\mathcal{S}}^2\mathcal{L}(\bm{\beta}^{true}))^{-1}\right\|\left(\left(1-\frac{96ns}{|\mathcal{A}|\kappa a}\right)\frac{\mu_0|\mathcal{A}|\lambda}{8snx_{\max}^2}+\sigma_3sx^3_{\max}\|\bm{\beta}^{oracle}-\bm{\beta}^{true}\|^2\right).\label{eq:lemma:2.12_15}
	\end{align}
	Note that the maximum value of $\left\|\nabla^2_{\mathcal{S}^c,\mathcal{S}}\mathcal{L}(\bm{\beta}^{true})(\nabla_{\mathcal{S},\mathcal{S}}^2\mathcal{L}(\bm{\beta}^{true}))^{-1}\right\|$ can be bounded:
	\begin{align}
	\left\|\nabla^2_{\mathcal{S}^c,\mathcal{S}}\mathcal{L}(\bm{\beta}^{true})(\nabla_{\mathcal{S},\mathcal{S}}^2\mathcal{L}(\bm{\beta}^{true}))^{-1}\right\| \le  \max_{\|v\|=1}\left\|\nabla^2_{\mathcal{S}^c,\mathcal{S}}\mathcal{L}(\bm{\beta}^{true})(\nabla_{\mathcal{S},\mathcal{S}}^2\mathcal{L}(\bm{\beta}^{true}))^{-1}v\right\|.\label{eq:lemma:2.12_16}
	\end{align}
	From \eqref{eq:lemma:2.12_16} and Lemma \ref{lemma:2.15}
	, the following inequality holds with probability $1-2s\exp\left(-\frac{|\mathcal{A}|\mu_0}{4s\sigma_2x_{\max}^2}\right)$.
	\begin{align}
	\max_{\|v\|=1}\left\|\nabla^2_{\mathcal{S}^c,\mathcal{S}}\mathcal{L}(\bm{\beta}^{true})(\nabla_{\mathcal{S},\mathcal{S}}^2\mathcal{L}(\bm{\beta}^{true}))^{-1}v\right\|&\le \frac{2n}{\mu_0|\mathcal{A}|}\max_{\|v\|=1}\left\|\nabla^2_{\mathcal{S}^c,\mathcal{S}}\mathcal{L}(\bm{\beta}^{true})v\right\|\notag\\
	&\le \frac{2n}{\mu_0|\mathcal{A}|}\cdot sx_{\max}^2=\frac{2snx_{\max}^2}{\mu_0|\mathcal{A}|}.\label{eq:lemma:2.12_17}
	\end{align}
	Thus, \eqref{eq:lemma:2.12_15} can be simplified to:
	\begin{align}
	\|\nabla_{\mathcal{S}^c}\mathcal{L}(\bm{\beta}^{oracle})\|_{\infty} &\le \left(1-\frac{96ns}{|\mathcal{A}|\kappa a}\right)\frac{\lambda}{4}+\sigma_3sx^3_{\max}\|\bm{\beta}^{oracle}-\bm{\beta}^{true}\|^2\notag\\
	&\quad\quad+\frac{2snx_{\max}^2}{\mu_0|\mathcal{A}|}\left(\left(1-\frac{96ns}{|\mathcal{A}|\kappa a}\right)\frac{\mu_0|\mathcal{A}|\lambda}{8snx_{\max}^2}+\sigma_3sx^3_{\max}\|\bm{\beta}^{oracle}-\bm{\beta}^{true}\|^2\right)\notag\\
	&=\left(1-\frac{96ns}{|\mathcal{A}|\kappa a}\right)\frac{\lambda}{2}+\frac{\sigma_3sx^3_{\max}(\mu_0|\mathcal{A}|+2snx_{\max}^2}{\mu_0|\mathcal{A}|}\|\bm{\beta}^{oracle}-\bm{\beta}^{true}\|_2^2\label{eq:lemma:2.12_18}.
	\end{align}
	Further, conditioning on event $\mathcal{E}_5$ defined in \eqref{eq:event_5}, we have:
	\begin{align}
	\|\nabla_{\mathcal{S}^c}\mathcal{L}(\bm{\beta}^{oracle})\|_{\infty}&\le \left(1-\frac{96ns}{|\mathcal{A}|\kappa a}\right)\frac{\lambda}{2}+\frac{\sigma_3sx^3_{\max}(\mu_0|\mathcal{A}|+2snx_{\max}^2}{\mu_0|\mathcal{A}|}\left(\sqrt{C_2\lambda}\right)^2\notag\\
	&\le \left(1-\frac{ 96ns}{|\mathcal{A}|\kappa a}\right)\lambda,\label{eq:lemma:2.12_19}
	\end{align}
	where \eqref{eq:lemma:2.12_19} uses $C_2 = \frac{\mu_0|\mathcal{A}|\lambda}{2\sigma_3sx_{\max}^3(\mu_0|\mathcal{A}|+2snx_{\max}^2)}$. The inequality \eqref{eq:lemma:2.12_19} directly implies event $\mathcal{E}_2$.

\end{proof}

{\color{cyan}

	\medskip

\end{document}